\documentclass{article}

\newcommand{\couplingset}{\Pi}

\newcommand{\E}{\mathbb{E}}

\newcommand{\xbf}{\mathbf{x}}

\newcommand{\ybf}{\mathbf{y}}

\def\msgw{\text{max-$SGW_p$}}

\usepackage{etoolbox}
\makeatletter
\usepackage{ragged2e}

\usepackage{arxiv}

\usepackage[utf8]{inputenc} % allow utf-8 input
\usepackage[T1]{fontenc}    % use 8-bit T1 fonts
\usepackage{hyperref}       % hyperlinks
\usepackage{url}            % simple URL typesetting
\usepackage{booktabs}       % professional-quality tables
\usepackage{amsfonts}       % blackboard math symbols
\usepackage{nicefrac}       % compact symbols for 1/2, etc.
\usepackage{microtype}      % microtypography
\usepackage{subfigure}
\usepackage{placeins}
\usepackage{tikz}
\usetikzlibrary{shapes.geometric, arrows}% 
\usepackage{amsmath, tikz-cd}
\usepackage{amssymb}
\usepackage{graphicx}
\usepackage{xcolor}         % colors
\usepackage{mathrsfs}
\usepackage{commath}
\usepackage{accents}
\usepackage{natbib}
\usepackage{mathtools}
\usepackage[bb=dsserif]{mathalpha}
\usepackage{bm}
\usepackage[shortlabels]{enumitem}
\usepackage{xparse}
\usepackage{dsfont}
\usepackage{bbm}
\usepackage{caption}
\usepackage{subcaption}
\usepackage{float}
\usepackage{threeparttable}
\usepackage{booktabs}
\usepackage{multirow}
\usepackage{pifont}
\usetikzlibrary{calc}
\usepackage{algorithm, algorithmic}
\usepackage{wrapfig}
\usepackage{fontawesome}

\usepackage{amsthm}
\usepackage{subcaption}
\DeclareMathOperator*{\argmin}{argmin}

\usepackage[capitalize,noabbrev]{cleveref}

\theoremstyle{plain}
\newtheorem{theorem}{Theorem}[section]
\newtheorem{proposition}[theorem]{Proposition}
\newtheorem{lemma}[theorem]{Lemma}

\theoremstyle{definition}
\newtheorem{definition}[theorem]{Definition}

\theoremstyle{remark}
\newtheorem{remark}[theorem]{Remark}

\usepackage[textsize=tiny]{todonotes}

\title{Relation-Aware Slicing in Cross-Domain Alignment}

\author{
 Dhruv Sarkar \\
  Indian Institute of Technology, Kharagpur\\
   \And
 Aprameyo Chakrabartty \\
  Indian Institute of Technology, Kharagpur \\
  \And
 Anish Chakrabarty \\
  Theoretical Statistics and Mathematics Unit\\
  Indian Statistical Institute, Kolkata\\
  \And
 Swagatam Das \\
  Electronics and Communication Sciences Unit\\
  Indian Statistical Institute, Kolkata\\
}

\begin{document}
\maketitle

\begin{abstract}
The Sliced Gromov-Wasserstein (SGW) distance, aiming to relieve the computational cost of solving a non-convex quadratic program that is the Gromov-Wasserstein distance, utilizes projecting directions sampled uniformly from unit hyperspheres. This slicing mechanism incurs unnecessary computational costs due to uninformative directions, which also affects the representative power of the distance. However, finding a more appropriate distribution over the projecting directions (\textit{slicing distribution}) is often an optimization problem in itself that comes with its own computational cost. In addition, with more intricate distributions, the sampling itself may be expensive. As a remedy, we propose an optimization-free slicing distribution that provides fast sampling for the Monte Carlo approximation. We do so by introducing the Relation-Aware Projecting Direction (RAPD), effectively capturing the pairwise association of each of two pairs of random vectors, each following their ambient law. This enables us to derive the Relation-Aware Slicing Distribution (RASD), a location-scale law corresponding to sampled RAPDs. Finally, we introduce the RASGW distance and its variants, e.g., IWRASGW (Importance Weighted RASGW), which overcome the shortcomings experienced by SGW. We theoretically analyze its properties and substantiate its empirical prowess using extensive experiments on various alignment tasks.
\end{abstract}

\section{Introduction}
\label{submission}

Optimal Transport (OT) theory serves as a cornerstone of contemporary machine learning. A notable outcome of OT theory is the Wasserstein distance (WD), widely used in various learning applications \citep{khamis2024scalable}. Metrizing weak convergence in a compact metric space, WD offers a strong measure of discrepancy between probability distributions of the same dimensionality. However, it is the associated computational complexity that depreciates its appeal. Given $n$ replicates from a $d$-variate distribution, calculating the empirical WD incurs a space complexity of $\mathcal{O}(n^2)$. Moreover, the associated rate of convergence cannot escape the curse of dimensionality ($\mathcal{O}(n^{-\frac{1}{d}})$) if the underlying distributions have arbitrary smoothness. Among the several remedies proposed \citep{cuturi2013sinkhorn, altschuler2017near, guo2020fast}, it is the Sliced Wasserstein Distance (SWD) \citep{bonneel2015sliced} that achieves linear space complexity without sacrificing statistical or topological properties \citep{nadjahi2020statistical}. The underlying trick involves linearly projecting observations onto $\mathbb{R}$, which renders the remaining transportation simply sorting the data points and matching. The projected distributions also incur a sample complexity of $\mathcal{O}(n^{-\frac{1}{2}})$ under WD, making SWD impervious to the dimensionality curse. 

The correspondence problem becomes challenging if the two underlying distributions originate from distinct metric spaces. This is due to the notion of mass transportation failing in the absence of a common medium of comparison between the spaces. Invoking a $L^{p}$-relaxation to the well-known Gromov-Hausdorff metric, the Gromov-Wasserstein distance (GW) \citep{memoli2011gromov} overcomes this issue. Instead of calculating the average cost of transportation, it compares the geometry of the underlying spaces. Precisely, it is the expected departure from isometry between them. While theoretically rich, computing the GW distance runs into roadblocks more demanding than WD. Typically, it requires solving a quadratic assignment problem (QAP) that is in general
NP-hard to compute. To circumvent this, one may approximately calculate a linear lower bound as proposed by \citet{chowdhury2019gromov}. So far, the most reliable solution to the problem remains Entropic regularization \citep{scetbon2022linear, rioux2024entropic}. It follows earlier estimation techniques using regularized versions of entropy, which can be solved more efficiently by repeatedly employing Sinkhorn projections \citep{solomon2016entropic,peyre2016gromov}. However, among the many methods inspired by similar techniques in OT, the most attractive in terms of ease of implementation and theoretical elegance is still \textit{slicing}. The Sliced GW (SGW) \citep{vayer2019sliced} again uses Radon transforms to project the measures to univariate distributions, followed by computing the GW-distance between them. Taking the expectation over all slices, it emerges as a proxy for GW.

Regardless, the slicing has its drawbacks. Firstly, as also observed in WD, the complexity of the projections may increase exponentially with the data dimension (e.g., LIPO for Max-slicing \citep{nietert2022statistical}). Moreover, the outer expectation over the slices typically becomes intractable \citep{vayer2019sliced} and has to be replaced by a Monte-Carlo (MC) approximate. The approximation is done by sampling uniformly from the unit hypersphere ($\mathbb{S}^{d-1}$). However, this can be highly computationally inefficient since randomly sampled directions may include highly \textit{uninformative} ones. As GW encapsulates the similarity between pairwise distances from the underlying spaces, the geometric structure must remain intact relative to one another after projections. This is imperative to a proxy distance reflecting the same properties as GW. Clearly, there exist projecting directions that perturb pairwise distances (see Section \ref{subsec:intuition}) and, as a result, distort the alignment of cross-domain observations. We only focus on the directions $\Theta \subset \mathbb{S}^{d-1}$ that maintain the geometric association between the two spaces.  

We prioritize the concentration of such informative \textit{relation-aware} projecting directions while designing the slicing distribution. They are defined as the directions that capture the intra-pair distance of each pair sampled from the respective input measures. More specifically, we take the unit bisecting vector of the displacement vectors of each pair. Based on what we call the \textit{Relation-aware Slicing Distribution} (RASD), we introduce a novel and efficient slicing technique into cross-domain alignment metrics such as GW.
We show that sampling from the same is also quite simple as it involves no additional optimization and only relies on input samples from both input distributions, which are assumed to be easy to sample from in most machine learning (ML) tasks.

\vspace{-8pt}
\section{Related Works}
To better appreciate our propositions, let us first go through the ever-expanding list of slicing techniques addressing diverse objectives, such as minimizing slicing complexity or projection variability. Notably, GSW \citep{kolouri2019generalized} and ASW \citep{chen2020augmented} extend the approach to non-linear projections to represent complex associations between distributions. In search of the most efficient slice in terms of information, \citet{deshpande2019max} and \citet{nguyen2021distributional} propose Max-SW and DSW distances, respectively. Most of such methods involve precise optimizations for finding a class of suitable directions. For example, the HSW \citep{nguyen2023hierarchical} undergoes a selection of \textit{bottleneck} slices that linearly combine to simulate projections. HHSW \citep{nguyen2024hierarchical} follows a similar combining mechanism, however, non-linearly. At the end of parameter-free techniques, EBSW \citep{nguyen2023energy} relies on slicing distributions whose density is proportional to an energy
function of the projected univariate WD. 

We emphasize that all such methods originate to cater to the OT problem, out of which only the vanilla slicing \citep{bonneel2015sliced} (similarly, Max-SW) is adopted for cross-domain alignment to date. The first obstacle in doing so is the involvement of distributions from disparate spaces, which immediately renders a projection suitable for one measure useless for the other. Secondly, the alignment problem being fundamentally different from transportation makes the notion of \textit{informativeness} of a projecting direction incompatible in the case of GW. Keeping in mind the cross-domain problem, where preserving the geometry is key, we redefine what it means to be `informative' based on relation-awareness. The work that relates to our approach the best is RPSW \citep{nguyen2024slicedwassersteinrandompathprojecting}, which facilitates Monte-Carlo sampling by introducing normalized differences between replicates from the two distributions while choosing slices. 

\textbf{Contributions.} In summary, our contributions are as follows.
\begin{itemize}[leftmargin=*, topsep=0pt]
    \item We introduce a novel optimization-free slicing mechanism based on what we call the \textit{relation aware} projecting directions (RAPD). This involves randomly perturbing the bisecting unit vector of the displacement between sample pairs from different input measures. Based on RAPD, we define the relation-aware slicing distribution (RASD), from which drawing random projecting sample directions is both easy and efficient.
    \item Based on RA slices, we develop two new variants of Gromov-Wasserstein distances that preserve metric topological and statistical properties of GW besides overcoming computational drawbacks. The first is termed \textit{relation aware sliced Gromov-Wasserstein} (RASGW), which replaces uniform sampling in SGW by RASD. The second variant, \textit{importance-weighted relation aware sliced GW} (IWRASGW), is defined as the expectation of a weighted ensemble of several randomly projected distances using RAPDs. We explore the theoretical attributes of both RASGW and IWRASGW, including topological, statistical, and computational properties. More specifically, we assess the metric properties, their relation to other SGW variants, incurred sample complexity, and computational complexities when utilizing Monte Carlo techniques.
    \item Comparing the performance of RASGW and IWRASGW with other slicing mechanisms poses new challenges. This is primarily due to the absence of existing benchmarks in cross-domain alignment tasks. In a first, we adapt several OT projections, such as SW, Max-SW, DSW, RPSW, and EBSW in the GW framework. We empirically show that in generative tasks involving unalike spaces, RASGW and IWRASGW surpass the competition on diverse datasets. In the process, we modify Gromov-Wasserstein GANs \citep{bunne2019learning} and Gromov-Wasserstein Autoencoders \citep{nakagawa2023gromovwasserstein} using our metrics, which reveal that our metric is indeed superior in capturing the relational aspects between datasets at minimal computational overhead.
\end{itemize}

We defer the proofs of key results and additional experimental details to the Appendix. All codes for our experiments, along with execution instructions, are placed in this \href{https://anonymous.4open.science/r/Relation-Aware-Slicing-in-Cross-Domain-Alignment-887A}{repository}\footnote{\faFileCodeO~\scriptsize{\url{https://anonymous.4open.science/r/Relation-Aware-Slicing-in-Cross-Domain-Alignment-887A}}}.

\vspace{-4pt}
\section{Background}
\label{sec:preliminaries}
In this section, we review the preliminary concepts needed to develop our framework, namely the Gromov-Wasserstein distance, the sliced Gromov-Wasserstein distance, and also elucidate the idea of relational discrepancy. 

\textbf{Notations.}  We denote by $\mathcal{P}(\mathcal{X})$ the set of all probability measures supported on a space $\mathcal{X}$. For $p\geq 1$, $\mathcal{P}_p(\mathcal{X})$ denotes the set of all probability that have finite $p$-moments. Given $d \geq 2$, we denote $\mathbb{S}^{d-1}:=\{\theta \in \mathbb{R}^{d}\mid  ||\theta||_2^2 =1\}$ and $\mathcal{U}(\mathbb{S}^{d-1})$ as the unit hyper-sphere and the uniform law supported on it. For any two sequences $a_{n}$ and $b_{n}$, the notation $a_{n} = \mathcal{O}(b_{n})$ means that $a_{n} \leq C b_{n}$ for all $n \geq 1$, where $C$ is some universal constant.

\textbf{Gromov-Wasserstein Distance.}
Let $(\mathcal{X}, c_{X}, \mu)$ and $(\mathcal{Y}, c_{Y}, \nu)$ be two metric measure (mm) spaces, where $c_X$ and $c_Y$ are the endowed metrics. For simplicity, we consider without loss of generality that both $\mathcal{X}$ and $\mathcal{Y}$ are Polish. We sparingly refer to the mm spaces as $X$ and $Y$, respectively. Let $\Pi(\mu, \nu) = \{\pi \in \mathcal{P}(\mathcal{X} \times \mathcal{Y}) : \pi(\cdot \times \mathcal{Y}) = \mu, \pi(\mathcal{X} \times \cdot) = \nu\}$ be the set of all couplings between $\mu$ and $\nu$. Then the $p$-GW distance \citep{memoli2011gromov} between the two mm spaces is defined as follows.
\vspace{2pt}
\begin{align}
\label{gw}
\text{GW}_{p}^{p}(\mu,\nu) =  \underset{\pi \in \couplingset}{\inf} \E_{\pi \otimes \pi}[c_p(\xbf,\xbf',\ybf,\ybf')], 
\end{align}
where $c_p(\xbf,\xbf',\ybf,\ybf') = |c_{X}(\xbf,\xbf')-c_{Y}(\ybf,\ybf')|^{p}$ denotes the realized distortion given samples $(\xbf,\ybf),(\xbf',\ybf') \sim \mu \otimes \nu$.

\textbf{Sliced Gromov-Wasserstein Distance.} It is essentially defined as the average of several GW distances between $\mu$ and $\nu$ once they are projected onto the space of univariate measures \citep{vayer2019sliced}. In the population regime, $\text{SGW}_p^p(\mu,\nu)  =  \mathbb{E}_{ \theta \sim \mathcal{U}(\mathbb{S}^{d-1})} [\text{GW}_p^p (\theta \sharp \mu,\theta \sharp \nu)]$, where $\theta \sharp \mu$ is the push-forward of the measure $\mu$ through the function $f_{\theta}:\mathbb{R}^{d} \to \mathbb{R}$. Specifically, one employs the linear projection $f_{\theta}(x) = \theta^\top x$. However, the expectation in the definition of the SGW distance is intractable to compute. Therefore, the Monte Carlo scheme is employed to approximate the value: $\widehat{\text{SGW}_{p}^p}(\mu,\nu;M) = \frac{1}{M} \sum_{l=1}^M\text{GW}_p^p (\theta_l \sharp \mu,\theta_l \sharp \nu)$, where $\theta_{1},\ldots,\theta_{M} \overset{i.i.d}{\sim}\mathcal{U}(\mathbb{S}^{d-1})$ and are referred to as \textit{projecting directions}. The pushfoward measures $\theta_1\sharp \mu,\ldots, \theta_M \sharp \mu$ are called projections of $\mu$ (similarly for $\nu$). The number of Monte Carlo samples $M$ is often referred to as the number of projections. When $\mu$ and $\nu$ are discrete measures that have at most $n$ supports (sample problem), the best achievable time complexity and memory complexity of the SGW are $\mathcal{O}(Mn\log n+Mdn)$ and $\mathcal{O}(Md+Mn)$ respectively \citep{vayer2019sliced}. We elaborate on the same in the Appendix \ref{sec:algorithm}. Evidently, to get a fair estimate, $M$ should be large enough compared to $d$. Employing instead the accumulation operation $\max$, one may define the max-sliced Gromov Wasserstein (Max-SGW) distance between $\mu, \nu \in \mathcal{P}_p(\mathbb{R}^d)$, given as $\text{Max-SGW}_p(\mu,\nu)  = \max_{\theta \in \mathbb{S}^{d-1}} \text{GW}_p (\theta \sharp \mu,\theta \sharp \nu)$. Since most slices from either space turn out non-informative, obtaining a suitable coupling becomes a challenge in Max-SGW.

\textbf{On Relational Awareness.} In cross-domain alignment, pairwise distances between observations is key to defining any measure of discrepancy. Preserving this information is crucial for ensuring that the intrinsic structure of one population ($\mu$) is maintained during translation, matching or transform sampling \citep{hur2024reversible} from the other ($\nu$). Given samples $\{(\xbf_{i},\ybf_{j})\}^{m,n} \sim \mu \otimes \nu$, we refer to the distribution of pairwise distances $c_{X}(\xbf_{i}, \xbf_{i'})$, $i,i' = 1, \cdots, m$ (similarly for $c_{X}(\ybf_{j}, \ybf_{j'})$) as \textit{relational information}. Collectively, they characterize the relational structure (local geometry) of the ambient space. The extent of departure from perfect alignment (i.e. isometry) $c_{p} = |c_{X}(\xbf_{i}, \xbf_{i'})-c_{Y}(\ybf_{j}, \ybf_{j'})|^{p}$ embodies the association between such structures. Unlike the case $\mathcal{X} = \mathbb{R}^d$ and $\mathcal{Y} = \mathbb{R}^{d'}$, it is quite difficult to characterize the concentration of sample distortions between arbitrary spaces, e.g., networks. In general, given an optimal coupling, we refer to the same as \textit{relational discrepancy}.
 
In this context, we identify a crucial property a projection ($f_{\theta}$) must possess during slicing. Ideally, it should preserve the relational information of one space with respect to the other. As such, any perturbation due to projection must reflect symmetrically on pairwise distances from both spaces. We call this the \textit{relational awareness} of a slice $\theta$. Additionally, in Appendix \ref{more_preliminaries} we provide further background on OT and GW. Appendix \ref{sec:additional_baselines} is dedicated to establishing some baselines extending existing slicing methods in OT to GW, besides SGW and Max-SGW.
\vspace{-8pt}
\section{Relation-Aware Sliced Gromov Wasserstein}
\label{sec:RASGW}
We first introduce RAPD as candidate directions having relational awareness, which in turn invoke RASD in Section~\ref{subsec:rpd}. We then discuss the Relation-Aware Sliced Gromov Wasserstein variants in Section~\ref{subsec:RASGW}.

\begin{wrapfigure}[20]{r}{0.55\textwidth}
    \centering
    \vspace{-1.2\intextsep}
    \includegraphics[width=\linewidth]{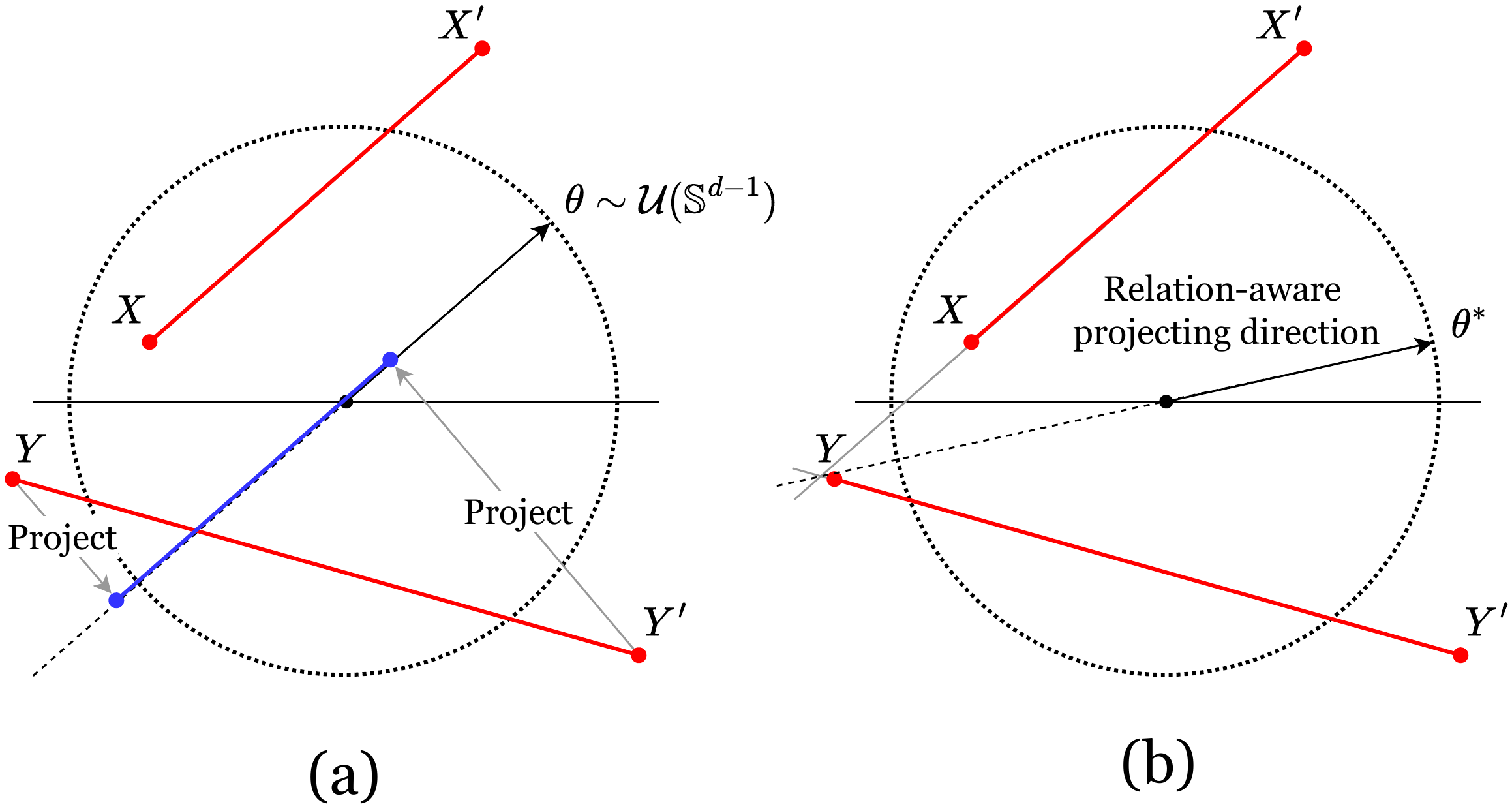}
    \caption{(a) For any $\theta \sim \mathcal{U}(\mathbb{S}^{d-1})$, drawn uniformly randomly, it may so happen that $X-X' // \:\theta$, where $X,X' \sim \mu$. While this preserves the distance $||X-X'||$ post-projection, slicing $Y-Y'$ along the same direction (marked in \textcolor{blue}{blue}) distorts the relational information. (b) In the absence of specific scaling for each pair $(X,X')$ (or $(Y,Y')$), $\theta^{*}$ is the only direction that has the same effect on both pairwise distances during slicing.}
    \label{fig:RAPD}
\end{wrapfigure}

% \subsection{Motivation}
\label{subsec:intuition}
\textbf{Motivation.} Note that, while a dimensionality reduction of a $d$-variate distribution inevitably leads to a loss of information, the relational information can still be recovered. However, given two participating distributions $\mu, \nu$, slicing directions also must ensure that relational discrepancy is kept intact to the greatest extent. On the other hand, for $\theta \sim \mathcal{U}(\mathbb{S}^{d-1})$, we have $\mathbb{P}\big(\theta^{T}(X-X') \leq \epsilon\big) \geq 1- e^{-d\epsilon^{2}}$, given any $X,X' \sim \mu$ \citep{kolouri2019generalized}. As such, slicing directions chosen to preserve relational information in one space may always distort the same corresponding to the other domain. The degradation of relational discrepancy, as a result, occurs exponentially with $d$ if slices are chosen uniformly at random.

% \begin{figure}
%     \centering
%     \includegraphics[width=0.9\linewidth]{RAPD.png}
%     \caption{Relation Aware Projecting Direction}
%     \label{fig:RAPD}
% \end{figure}

% \begin{tikzpicture}
%     % Draw the circle
%     \draw[dashed] (0,0) circle (2cm);
    
%     % Draw the axes lines
%     \draw[red] (-2.5,0) -- (2.5,0) node[right] {$X'$};
%     \draw[blue] (0,-2.5) -- (0,2.5) node[above] {$Y'$};
    
%     % Draw the red vectors
%     \draw[red, very thick] (-1.5,1) -- (1.5,-1) node[midway, sloped, above] {$x$};
    
%     % Draw the notation for projection
%     \draw[blue, very thick] (0,0) -- (0.7,0.7) node[midway, left] {Project};
    
%     % Draw small angle notation
%     \draw[thin, arc] (0.5,0) arc (0:45:0.5);
%     \node at (0.7,0.3) {$\theta \sim U(S^{d-1})$};
% \end{tikzpicture}

Searching for remedy, we identify only a handful of effective directions that become relation-aware. To illustrate our construction, let us consider a simple example. Assume both spaces to be $\mathbb{R}^d$, endowed with the Euclidean norm. As depicted in Figure \ref{fig:RAPD}(a), a set of observations $(X,Y), (X',Y') \sim \mu \otimes \nu$ may experience unequal perturbation in intra-pair distances after projection based on an arbitrary $\theta$. Those are rendered uninformative. While rescaling is a remedy for such unwanted amplifications, it immediately makes the slicing mechanism dependent on parameters. A parameter-free approach demands a $\theta$ that leaves a symmetric impression on $X-X'$ and $Y-Y'$.

To that end, let us consider a bisecting direction $\theta^{*}$ at an angle $\alpha_1$ with $\overrightarrow{XX'}$ and $\alpha_2$ with $\overrightarrow{YY'}$. Thus,
\begin{align*}
    c_{p}^{\frac{1}{p}} = \abs{||\theta^\top (X-X')|| - ||\theta^\top (Y-Y')||} = \abs{||X-X'||\cos{\alpha_1} - ||Y-Y'||\cos{\alpha_2}}.
\end{align*}
As such, we observe distorted relational discrepancies given an arbitrary choice of $(\alpha_1,\alpha_2)$, which may lead to a suboptimal coupling $\pi$ different from the intended. Only by selecting $\alpha_1 = \alpha_2 = \alpha$ do we keep the underlying optimization unscathed. Hence, the corresponding $\theta^\star$ is a candidate for relation-aware slices. The only impact on the proxy sliced GW value is a multiplicative factor of the form $\times \cos{\alpha}$. Notably, our construction results in a unique distortion $|c_X-c_Y|/\sqrt{2}$ if $\overrightarrow{XX'} \perp \overrightarrow{YY'}$, which has $0$ probability of appearing in a pool of random slices (e.g. in SGW). While Max slicing can realize an optimal direction theoretically in SW, its maximization objective makes little sense in the alignment problem where relation preservation is more important and the direction which maximizes the value of the metric need not be relation preserving. %there is no guarantee that the same exists in the case of distinct spaces. Moreover, similar to SGW, it may experience infinite waiting time to draw a desired $\theta^{*}$ if uniform random sampling is followed. 

\vspace{-6pt}
\subsection{Relation-Aware Projecting Direction}
\label{subsec:rpd}
Let us now statistically characterize the directions we intuitively constructed.
\begin{definition}[Intra-Relational Path]
\label{def:rp}
Given $p \geq 1$ and dimension $d \geq 1$, the Intra-relational path (IRP) between $X, X' \sim \mu \in \mathcal{P}(\mathbb{R}^d)$ is defined as the transformation $Z_{X,X'} =X-X'$.
% \begin{align*}
%     Z_{X,X'} =X-X'.
% \end{align*}
\end{definition}
% From Definition~\ref{def:rpd}, the relation aware $Z_{X,Y}$ belongs to the set of distribution over the unit-hypersphere $\mathcal{P}(\mathbb{S}^{d-1})$ for any $\mu$ and $\nu$ in $\mathcal{P}(\mathbb{R}^d)$.
The distribution of $Z_{X,X'}$ can be written as $\sigma_{\mu}:=\mu * (-)\sharp \mu$, where $*$ denotes the convolution operator, and $(-)\sharp \mu$ denotes the pushforward measure of $\mu$ through the function $f(x)=-x$ \citep{nguyen2024slicedwassersteinrandompathprojecting}. While the density of $\sigma_{\mu}$ is intractable in general, it is simple to sample from. Specifically, one needs to draw $X, X' \sim \mu$, followed by setting $Z_{X,X'} =X-X'$. We may readily define a corresponding projecting direction by normalizing it. That is, $\bar{Z}_{X,X'}= \frac{Z_{X,X'}}{\|Z_{X,X'}\|_2}$. This is the trivial slice that preserves relational information in $\mathcal{X}$ exactly.
\begin{remark}
    \label{remark:rpd} Note that there lies a chance of $\bar{Z}_{X,X'}$ being undefined in a sample problem. It can be primarily solved by following a without-replacement sampling scheme while drawing the dataset. Given an existing dataset, sacrificing repeated samples leads to information loss. \citet{nguyen2024slicedwassersteinrandompathprojecting} suggest considering a small additive constant $c$, which makes $Z_{X,X'} + c$ a valid intra-relational path.
\end{remark}
Since we require a direction that takes into account the relational structure of both distributions, we consider
\begin{align*}
    Z_{X,X',Y,Y'}= \frac{\bar{Z}_{X,X'} + \bar{Z}_{Y,Y'}}{\|\bar{Z}_{X,X'} + \bar{Z}_{Y,Y'}\|_2}, \; \textrm{and} \;
    Z'_{X,X',Y,Y'} = \frac{\bar{Z}_{X,X'} - \bar{Z}_{Y,Y'}}{\|\bar{Z}_{X,X'} - \bar{Z}_{Y,Y'}\|_2}
\end{align*}
\vspace{2pt}
Taking both of the above directions into account allows us to ensure that at least one of them is the bisecting vector of the acute angle, which further implies $\alpha \leq \pi/4$. Drawing such a direction $\theta^{*}$ with high probability requires defining a distribution that has a high central tendency towards it, with decaying tails. Thus, we add a continuous random perturbation, drawn from a suitable location-scale distribution around the normalized random path. This additionally helps cover the entire support $\mathbb{S}^{d-1}$ in representative samples. Based on this idea, a RAPD is defined as follows.
\iffalse
We now define a new measure $\sigma_{\mu,\nu} = (\sqrt{2})\sharp \sigma_{\mu}*\sigma_{\nu}$.However, the density of the IRP, i.e., $\sigma_{\mu}$ does not always have full support on $\mathbb{R}^d$, e.g., in discrete cases ($\mu$ is discrete). Therefore, $\frac{Z_{X,X'}}{\|Z_{X,X'}\|_2}$ does not fully support $\mathbb{S}^{d-1}$, which cannot guarantee theoretical property and harms practical performance. As a solution, we can add a random perturbation around the normalized random path. As long as the random perturbation has a continuous density, the marginal density of the final projecting direction is continuous. Now, we define the relation-aware projecting direction as follows:
\fi

\begin{definition}[Relation-Aware Projecting Direction]
\label{def:rpd}
Given a location-scale distribution $\sigma_\kappa$ on $\mathbb{S}^{d-1}$, exactly characterized by $\kappa >0$, the relation aware projecting direction (RAPD) between $(X,Y), (X',Y') \sim \mu \otimes \nu$ such that $\mu, \nu \in \mathcal{P}(\mathbb{R}^d)$ is defined as $\theta|X,X',Y,Y',\kappa \sim D_{\kappa}(X,X',Y,Y')$, where $D_{\kappa} =\frac{1}{2}\sigma_\kappa(\cdot;Z_{X,X',Y,Y'}) + \frac{1}{2}\sigma_\kappa(\cdot;Z'_{X,X',Y,Y'})$.
\end{definition}
Observe that, simulation from $D_{\kappa}$ essentially means tossing a fair coin to select either $\sigma_\kappa(\cdot;Z)$ or $\sigma_\kappa(\cdot;Z')$. Typically, we choose $\sigma_\kappa$ as the vMF distribution, whose density follows the form $\text{vMF}(\theta;\epsilon,\kappa)\propto \exp(\kappa\epsilon^\top \theta)$~\cite{jupp1979maximum} and PS, given as $\text{PS}(\theta;\epsilon,\kappa)  \propto (1+\epsilon^\top \theta)^\kappa$~\cite{de2020power}. In Definition~\ref{def:rpd}, the reason for choosing the location-scale family is to guarantee that the projecting direction concentrates around the prescribed direction. For further details justifying the choice, we refer the reader to Appendix \ref{more_preliminaries}.

\textbf{Relation-Aware Slicing Distribution.} From Definition~\ref{def:rpd}, we can obtain the slicing distribution of the RAPD by marginalizing out $X,Y$. In particular, we define the RASD as follows.
\begin{definition}
\label{def:rpsd}
Given $0<\kappa<\infty$, the relation-aware slicing distribution based on $\mu, \nu \in \mathcal{P}(\mathbb{R}^d)$, $d \geq 1$ is defined as 
\begin{align*}
    \sigma_{\textrm{RA}}(\theta;\mu,\nu,\sigma_\kappa) = \int D_{\kappa}(X,X',Y,Y') d(\mu \otimes \nu)^{\otimes 2}.
\end{align*}
\end{definition}
We can sample from $\sigma_{\textrm{RA}}(\theta;\mu,\nu,\sigma_\kappa)$ easily by sampling $X, X' \sim \mu$ and $Y, Y' \sim \nu$ independently, then $\theta\sim D_{\kappa}(\xbf,\xbf',\ybf,\ybf')$, where $\xbf$ is the realization of $X$.
\iffalse
\begin{remark}
    \label{remark:rpsd} We can rewrite $\sigma_{\textrm{RA}}(\theta;\mu,\nu,\sigma_\kappa) = \int D(x,x',y,y') d \pi\otimes\pi(x,y,x',y')$ where $\pi = \mu \otimes \nu$ is the independent coupling between $\mu$ and $\nu$ (efficiently for computation). As a natural extension, we can use other coupling between $\mu$ and $\nu$. However, a more complicated coupling could cost more computation for doing sampling while the benefit of using such coupling is not trivial. 
\end{remark}
\fi

\subsection{Relation-Aware Sliced Gromov-Wasserstein}
\label{subsec:RASGW}
We now propose two sliced GW variants that rely on RASDs.
\begin{definition}[Relation-Aware SGW]
    \label{def:RASGW}
    Given $0<\kappa<\infty$, $p \geq 1$, and $d \geq 1$, the relation-aware sliced Gromov-Wasserstein (RASGW) distance between two probability measures $\mu \in \mathcal{P}_p(\mathbb{R}^d)$ and $\nu \in \mathcal{P}_p(\mathbb{R}^d)$ is defined as
    \begin{align*}
        \text{RASGW}^p_p(\mu,\nu;\sigma_\kappa) &=\mathbb{E}_{\theta \sim \sigma_{\textrm{RA}}(\theta;\mu,\nu,\sigma_\kappa)}[\text{GW}_p^p(\theta \sharp \mu, \theta \sharp \nu)] \\ &=\mathbb{E}_{X,X' \sim \mu,Y,Y'\sim \nu} \mathbb{E}_{\theta \sim D_\kappa(\theta; X,X',Y,Y')}[\text{GW}_p^p(\theta \sharp \mu, \theta \sharp \nu)], 
    \end{align*}
    where $\sigma_{\textrm{RA}}$ is as in Definition~\ref{def:rpsd}.
\end{definition}
\label{submission_1}

% \begin{remark}
%     \label{remark:RASGW} The equivalent definition of RASGW to Definition~\ref{def:RASGW} is: 
%     $
%         \text{RASGW}^p_p(\mu,\nu;\sigma_\kappa) =
%     $
% \end{remark}

To diversify the slices according to multiple views of the dataset, we can generalize the weights of RAPDs, namely $\theta_1,\cdots,\theta_L \sim \sigma_{\textrm{RA}}(\theta;\mu,\nu,\sigma_\kappa)$ by their corresponding projected distances $\text{GW}_p^p(\theta_1 \sharp \mu, \theta_1 \sharp \nu)),\cdots, \text{GW}_p^p(\theta_L \sharp \mu, \theta_L \sharp \nu))$. This takes inspiration from the importance sampling estimation of EBSGW as in Equation~\eqref{eq:ebsgw} (see Appendix \ref{sec:additional_baselines}). The GW variant so constructed is called the Importance-Weighted RASGW (IWRASGW) distance and is given as follows.
\begin{definition}
    \label{def:iwRASGW}
    Let $f:[0,\infty) \to (0,\infty)$ be an increasing function. Then, given $0<\kappa<\infty$ and $L \geq 1$, the importance weighted relation aware projection sliced Gromov-Wasserstein (IWRASGW) between two probability measures $\mu \in \mathcal{P}_p(\mathbb{R}^d)$ and $\nu \in \mathcal{P}_p(\mathbb{R}^d)$ is defined as
    \begin{align*}
        \text{IWRASGW}^p_p(\mu,\nu;\sigma_\kappa,L,f) =\mathbb{E}_{\theta_1,\ldots, \theta_L\sim \sigma_{\textrm{RA}}(\theta;\mu,\nu,\sigma_\kappa)}\bigg[\sum_{l=1}^L \text{GW}_p^p(\theta_l \sharp \mu, \theta_l \sharp \nu) \frac{w_{l}(f;\mu,\nu)}{\sum_{j=1}^L w_{j}(f;\mu,\nu)} \bigg], 
    \end{align*}
    where $\sigma_{\textrm{RA}}$ is defined as in Definition~\ref{def:rpsd}, and $w_{l}(f;\mu,\nu) \coloneqq f(\text{GW}_p^p(\theta_l \sharp \mu, \theta_l \sharp \nu))$, $d\geq 1$.
\end{definition}
The intuition behind utilizing $L$ appropriately weighted random projecting directions is to prioritize the relational discrepancy corresponding to a pair of observations from $(\mu \otimes \nu)^{\otimes 2}$ that has greater contribution towards finding an optimal coupling.

\textit{\textbf{Metric Properties.}} Since GW metrizes the equivalence class of isomorphic mm spaces \citep{memoli2011gromov}, any proposed sliced variant must ideally emulate the same. Towards the same, we first investigate the metricity of RASGW and IWRASGW.
\begin{theorem} Given an increasing function $f:[0,\infty) \to (0,\infty)$ and $0<\kappa<\infty$, $\text{RASGW}_{p}(\cdot,\cdot;\sigma_\kappa)$ and $\text{IWRASGW}_{p}(\cdot,\cdot;\sigma_\kappa,L)$ form semi-metrics on $\mathcal{P}(\mathcal{X})$ for any $p \geq 1$, $L \geq 1$. Namely, RASGW and IWRASGW satisfy non-negativity, symmetry, and identity of isometric isomorphism. The RASGW  satisfies the `quasi'-triangle inequality given as, 
\begin{align*}
    \text{RASGW}_{p}(\mu_1,\mu_2;\sigma_\kappa) \leq \text{RASGW}_{p}(\mu_1,\mu_3;\sigma_\kappa,\mu_1,\mu_2) + \text{RASGW}_{p}(\mu_3,\mu_2;\sigma_\kappa,\mu_1,\mu_2),
\end{align*} where $\text{RASGW}_{p}^p(\mu_1,\mu_3;\sigma_\kappa,\mu_1,\mu_2) = \mathbb{E}_{\theta \sim \sigma_{\textrm{RA}}(\theta;\mu_1,\mu_2,\sigma_\kappa)}[\text{GW}_p^p(\theta \sharp \mu_1, \theta \sharp \mu_3)]$ and a similar definition holds for $\text{RASGW}_{p}^p(\mu_3,\mu_2;\sigma_\kappa,\mu_1,\mu_2)$, for all $\mu_{i} \in \mathcal{P}(\mathcal{X})$, $i=1,2,3$ such that $\mathcal{X} \subset \mathbb{R}^d$ Polish.
\label{theo:metricity}
\end{theorem}
The proof of Theorem~\ref{theo:metricity} is given in Appendix~\ref{subsec:proof:theo:metricity}. Observe that both RASGW and IWRASGW can be immediately extended to serve distinct dimensional measures by introducing a \textit{padding} operation first. For example, if $\nu \in \mathcal{P}(\mathcal{Y})$, given $\mathcal{Y} \subset \mathbb{R}^{d'}$, $d'<d$; we rather calculate the RASGW distance between $\mu \in \mathcal{P}(\mathcal{X})$ and $\Delta\sharp\nu$, where $\Delta: \mathcal{Y} \rightarrow\mathcal{X}$ is a padding operator. \citet{vayer2019sliced}, in the context of SGW, use the `uplifting' padding given as $\Delta(y) \coloneqq (y_{1}, \cdots, y_{d'}, \underbrace{0, \cdots, 0}_{d-d'})$. The properties in Theorem \ref{theo:metricity} still hold under such adjustments.

% It is worth noting that the triangle inequality for RASGW and  is challenging to prove due to the dependency of the two input measures with the slicing distribution. Therefore, it is unknown if they satisfy the triangle inequality.
\iffalse
\begin{remark}
    \label{remark:symmetry} Although the $\sigma_{\textrm{RA}}(\theta;\mu,\nu,\sigma_\kappa)$ is not symmetric with respect to $\mu$ and $\nu$, RASGW and IWRASGW are still symmetric since $\textrm{GW}_p^p(\theta \sharp \mu, \theta \sharp \nu)$ is symmetric with respect to $\theta$ i.e., $\textrm{GW}_p^p(\theta \sharp \mu, \theta \sharp \nu)=\textrm{GW}_p^p(-\theta \sharp \mu, -\theta \sharp \nu)$ and we have $\sigma_\kappa(\theta;P_{\mathbb{S}^{d-1}}(x-y)) = \sigma_\kappa(-\theta;P_{\mathbb{S}^{d-1}}(y-x))$. We refer to the proof of Theorem~\ref{theo:metricity} for more details.
\end{remark}
\fi

% We now discuss the connection between RASGW and IWRASGW and their connection to other SGW variants.
\begin{proposition}[Metric Hierarchy]
    \label{prop:connection}
    For any $p \geq 1$, $L \geq 1$, increasing function $f:[0,\infty) \to (0,\infty)$, and $0<\kappa<\infty$, we have the following relationships: \\ $\text{(i) } \text{RASGW}_{p}(\mu, \nu;\sigma_\kappa) \leq \text{IWRASGW}_{p}(\mu, \nu;\sigma_\kappa,L,f)  \leq \text{Max-SGW}_p(\mu,\nu),$\\
    $\text{(ii) } \text{RASGW}_{p}(\mu, \nu;\sigma_\kappa) \xrightarrow{\kappa \to 0} \text{SGW}_p(\mu,\nu),$   \\
    $\text{(iii) } \text{IWRASGW}_{p}(\mu, \nu;\sigma_\kappa,L,f) \xrightarrow{L \to \infty} \text{EBSGW}_p(\mu,\nu;f).$
    
\end{proposition}
Proposition~\ref{prop:connection} positions RASGW and IWRASGW within the SGW hierarchy, showing they approach slicing (SGW) and the energy-based slicing (EBSGW) respectively, theoretically justifying their adaptiveness. The proof of Proposition~\ref{prop:connection} is given in Appendix~\ref{subsec:proof:prop:connection}. 

% \begin{theorem}
%     \label{theo:weak}
% For any $p \geq 1$, the convergence of probability measures under the relation aware sliced Wasserstein distance $\text{RASGW}_{p}(\cdot,\cdot)$ implies weak convergence of probability measures and vice versa.
% \end{theorem}

% The proof of Theorem~\ref{theo:weak} is given in Appendix~\ref{subsec:proof:theo:weak}.
\textit{\textbf{Statistical Properties.}} Given that sliced GW variants innately benefit tasks such as shape matching or surface correspondence \citep{solomon2016entropic}, it is crucial to study the one-sided sample complexity of RASGW and IWRASGW.

\begin{proposition}
    \label{prop:sample_complexity}
    Let $X_{1}, X_{2}, \ldots, X_{n} \overset{i.i.d.}{\sim} \mu$, such that $\mu$ is fully supported on $\mathcal{X} \subset \mathbb{R}^{d}$. We denote the empirical measure by $\mu_{n} = \frac{1}{n} \sum_{i = 1}^{n} \delta_{X_{i}}$. Then, for any $p > 1$, $L\geq 1$ and  $0<\kappa<\infty$, there exists a universal constant $C > 0$ such that
\begin{align*}
    \mathbb{E} [\text{RASGW}_{p}(\mu_{n},\mu;\sigma_\kappa)] \leq \mathbb{E} [\text{IWRASGW}_{p} (\mu_{n},\mu;\sigma_\kappa,L,f)] \leq C \sqrt{\frac{(d+1)\log (n+1)}{n}},
\end{align*}
where the outer expectation is taken with respect to $X_{1}, X_{2}, \ldots, X_{n}$.
\end{proposition}

Proposition~\ref{prop:sample_complexity} proves sublinear sample complexity ($\mathcal{O}(n^{-1/2})$), at par with the conventional SGW~\citep{nadjahi2019asymptotic}. This positions RASGW-- computationally superior already-- as a better choice for a suitable loss in cross-domain generative and shape-matching tasks, compared to SGW and Max-SGW. The proof of Proposition~\ref{prop:sample_complexity} is given in Appendix~\ref{subsec:proof:prop:sample_complexity}. 
% From the proposition, we can see that the approximation rate of using an empirical probability measure to a population measure of the proposed RASGW and IWRASGW is at the order of $n^{-1/2}$. 
\iffalse
Due to the missing proof of triangle inequality of RASGW and IWRASGW, it is non-trivial to extend from the one-sided sample complexity to the two-sided sample complexity as in .
\fi

\textit{\textbf{Computational Properties.}} The upcoming section is dedicated to computational details regarding our distances.

\textbf{Monte Carlo Estimation.} In an empirical setup, we resort to Monte Carlo simulations to estimate RASGW. This is also crucial due to the optimization that underlies the outer expectation. In particular, we sample $\theta_1,\ldots,\theta_M \sim \sigma_{\textrm{RA}}(\theta;\mu,\nu,\kappa)$ (as described in Section~\ref{subsec:rpd}) to construct the Monte Carlo estimate of RASGW:
\vspace{-4pt}
\begin{align}
    \label{eq:MC_RASGW}
    \widehat{\text{RASGW}^{p}_{p}}(\mu,\nu;\sigma_\kappa,M)  =\frac{1}{M} \sum_{l=1}^M \text{GW}_p^p(\theta_l \sharp \mu,\theta_l \sharp \nu).
\end{align}We refer to the reader to Algorithm~\ref{alg:RASGW}-~\ref{alg:IWRASGW} in Appendix~\ref{sec:algorithm} for a detailed discussion on the computation of RASGW and IWRASGW. The immediate concern related to such an approximation is its realized error. To that end, given that the sampling of RA slices incurs finite variability, we ensure that the approximation error is of order $\mathcal{O}(M^{-1/2})$.
\begin{proposition}
    \label{proposition:MCerror}
    For any $p\geq 1$, $0<\kappa<\infty$ , dimension $d \geq 1$, and $\mu,\nu \in \mathcal{P}_p(\mathbb{R}^d)$, we have:
    \begin{align*}
        \abs{ \mathbb{E}[\widehat{\text{RASGW}^{p}_{p}}(\mu,\nu;\sigma_\kappa)] - \text{RASGW}_{p}^p (\mu,\nu;\sigma_\kappa)} \leq \frac{1}{\sqrt{M}} \textrm{Var}_{\theta \sim \sigma_{\textrm{RA}}(\theta;\mu,\nu,\sigma_\kappa)}\left[ \text{GW}_p^p \left(\theta \sharp \mu, \theta \sharp \nu \right)\right]^{\frac{1}{2}}.
    \end{align*}
\end{proposition}
The proof of Proposition~\ref{proposition:MCerror} is given in Appendix~\ref{subsec:proof:proposition:MCerror}. Since the stochasticity of $\sigma_{\textrm{RA}}$ is completely user-defined, it is easy to control the $\textrm{Var}(\cdot)$ and the error bound as a result. As such, we may always observe a tighter bound than SGW, which shares a similar rate in terms of $M$.

Similarly, for IWRASGW,
we draw $H$ MC samples each corresponding to the $L$ random projecting directions, namely $\theta_{11},\ldots,\theta_{HL} \sim \sigma_{\textrm{RA}}(\theta;\mu,\nu,\kappa)$. This forms the MC estimate of IWRASGW as follows:
\begin{align}
    \label{eq:MC_IWRASGW}
    \widehat{\text{IWRASGW}_p^p}(\mu,\nu;\sigma_\kappa,L,H) =\frac{1}{H}\sum_{h=1}^H\nonumber \bigg[\sum_{l=1}^L \text{GW}_p^p(\theta_{hl} \sharp \mu, \theta_{hl} \sharp \nu) \frac{w_{hl}(f;\mu,\nu)}{\sum_{j=1}^L w_{hj}(f;\mu,\nu)} \bigg].
\end{align}
\vspace{2pt}
To simplify the parameter selections, we often set $H=1$ in the experiments. This leaves $L$ as the only tunable parameter for the number of projections. In contrast to RASGW, the error rate of IWRASGW with respect to $H$ is non-trivial due to the importance weights.

% \textbf{Unbiasedness.} Since $\text{RASGW}^p_p(\mu,\nu;\sigma_\kappa) $ and $\text{IWRASGW}^p_p(\mu,\nu;\sigma_\kappa,L,f) $ can be approximated directly by Monte Carlo samples, their corresponding estimators $\widehat{\text{RASGW}}_{p}^p(\mu,\nu;\sigma_\kappa)$ and $\widehat{\text{IWRASGW}}_p^p(\mu,\nu;\sigma_\kappa,L,H) $ are unbiased estimates. 

\textbf{Computational Complexities.} When $\mu$
 and $\nu$ are discrete measures that have at most $n$ supports, sampling from them only cost $\mathcal{O}(n)$ in terms of time and space. Hence, sampling $M$ random paths between $\mu$ and $\nu$ cost $\mathcal{O}(Mdn)$ in time and memory. Moreover, sampling from the vMF and the PS distribution costs $\mathcal{O}(Md)$ (see Algorithm 1 in~\citet{de2020power}) in time and memory. Adding the complexities for computing one-dimensional Gromov-Wasserstein distance, the time complexity and space complexity for RASGW turn out $\mathcal{O}(Mn\log n+Mdn)$ and $\mathcal{O}(Md+Mn)$ respectively. The involvement of the data dimensionalities in the overall complexity stems solely from the projections, making RASGW scale linearly with the same. This complements the complexities achieved by SGW, given that it results in an optimal plan. For a detailed discussion, we refer the reader to Appendix \ref{sec:algorithm}. In the case of IWRASGW, the complexities remain the same. 

\textbf{Gradient Estimation.} In applications (see Section \ref{sec:experiments}), where RASGW and IWRASGW are used as a risk to estimate some parameters of interest, i.e., $\mu_\phi$ with $\phi \in \Phi$, we might want to estimate the gradient of RASGW and IWRASGW with respect to $\phi$. For RASGW, we have
\begin{align*}
    \nabla_\phi \text{RASGW}^p_p(\mu_\phi,\nu;\sigma_\kappa) =\nabla_\phi \mathbb{E}_{X,X' \sim \mu_\phi; Y,Y'\sim \nu} \mathbb{E}_{\theta \sim\sigma_\kappa(\theta;D_{\kappa}(X,X',Y,Y') }[\text{GW}_p^p(\theta \sharp \mu_\phi, \theta \sharp \nu)].
\end{align*}
Since vMF and PS are reparameterizable~\citep{de2020power}, we may do the same for $\sigma_{\textrm{RA}}(\theta;\mu,\nu,\sigma_\kappa)$ given $\mu_\phi$ is reparameterizable (e.g., $\mu_\phi:=f_\phi \sharp \varepsilon$ for $\varepsilon$ being a fixed noise, or as discussed in~\citet{kingma2013auto}). A common example of this may be a generative model, such as GAN, where $f_\phi$ is a transformation induced by a neural network. With parameterized sampling, we can sample $\theta_{1,\phi},\ldots,\theta_{M,\phi} \sim \sigma_{\textrm{RA}}(\theta;\mu_\phi,\nu,\sigma_\kappa)$, then form an unbiased gradient estimator as follows:
\vspace{-4pt}
\begin{align*}
    \nabla_\phi \text{RASGW}^p_p(\mu_\phi,\nu;\sigma_\kappa)  \approx \frac{1}{M} \sum_{l=1}^M \nabla_\phi \text{GW}_p^p(\theta_{l,\phi}\sharp \mu_\phi,\theta_{l,\phi}\sharp \nu).
\end{align*}
For IWRASGW, we sample $\theta_{11,\phi},\ldots,\theta_{HL,\phi} \sim \sigma_{\textrm{RA}}(\theta;\mu_\phi,\nu,\kappa)$ to form a similar gradient estimator, the expression of which can be found in Appendix \ref{sec:add_exps}. In this context, we mention that the requirement of gradient computation renders fast GW approximation methods, such as \citet{scetbon2022linear}, unsuitable to be adapted as a loss in cross-domain generative models (see Appendix \ref{sec:additional_baselines} for details). 
% \begin{align*}
%     \nabla_\phi\text{IWRASGW}_p^p(\mu_\phi,\nu;\sigma_\kappa,L,H) =\frac{1}{H}\sum_{h=1}^H\bigg[\nabla_\phi\sum_{l=1}^L \text{GW}_p^p(\theta_{hl,\phi} \sharp \mu_\phi, \theta_{hl,\phi} \sharp \nu) \frac{w_{hl,\phi}(f;\mu_\phi,\nu)}{\sum_{j=1}^L w_{hj,\phi}(f;\mu_\phi,\nu)} \bigg].
% \end{align*}
\iffalse
\begin{remark}
\label{remark:gradient}
    We found in later experiments that removing the dependent of $\phi$ in the slicing distribution i.e., using a dependent copy $\sigma_{\textrm{RA}}(\theta;\mu_{\phi'},\nu,\kappa)$ with $\phi'=\phi$ in value, leads to a more stable estimator in practice. It is worth noting that using a copy of both $\mu$ and $\nu$ (i.e., $\mu'$ and $\nu'$) in RASGW can even lead to a variant that satisfies the triangle inequality since the slicing distribution becomes a fixed distribution in this situation. We refer to this gradient estimator as the simplified gradient estimator.
\end{remark} 
\fi
\vspace{-6pt}
\section{Simulation Results}
\label{sec:experiments}
% \subsection{Gromov-Wasserstein GANs}
\textit{\textbf{Gromov-Wasserstein GAN.}} In contrast to the vanilla GAN \citep{goodfellow2014generative}, \citet{bunne2019learning} proposed a variant that operates between spaces of different dimensions by capturing their geometric relations using GW loss instead. Given a lower-dimensional input noise $Z \sim \nu$ (typically a Gaussian), the underlying optimization over a class of generators $\mathcal{G}$ becomes: $G^* = \argmin_{G \in \mathcal{G}} \, \text{GW}^2_2 \left(\mu, G\sharp\nu \right)$, where $\mu$ is the target distribution. Observe that here, the transformation $G$ executes an uplifting operation. \citet{vayer2019sliced} adopts this framework by replacing GW with SGW as the loss function for training $G$. Our first experiment involves checking the efficacy of RASGW and IWRASGW in such cross-domain generations. We consider both cases, i.e., $\mu$ (target) and $\nu$ (source), of similar and dissimilar dimensionalities. 

% \begin{wraptable}[11]{r}{0.45\textwidth}
%     \centering
%     \vspace{-1.0\intextsep}
%     \scriptsize
%     \caption{\footnotesize{Gromov Wasserstein-2 distance and computational times at iteration 10000 for 2D$\rightarrow$3D and 3D$\rightarrow$2D generations.}}
%     \scalebox{0.8}{
%     \begin{tabular}{lcc|cccccc}
%     \toprule
%      Method & \multicolumn{2}{c|}{Step 10000 (3D to 2D)} & \multicolumn{2}{c}{Step 10000 (2D to 3D)} \\
%      \cmidrule{2-5}
%      & $\text{GW}_2$($\downarrow$) & $\text{Time (s)}$($\downarrow$) & $\text{GW}_2$($\downarrow$) & $\text{Time (s)}$($\downarrow$) \\
%      \midrule
%     SGW & 111.32 & \underline{42.82} & 2.59 & \textbf{41.98} \\
%     Max-SGW & 379.07 & 180.87 & 2.50 & 186.66 \\
%     DSGW & 115.54 & 220.90 & 1.56 & 222.39 \\
%     EBSGW & 82.20 & \textbf{41.84} & 1.82 & \underline{42.33} \\
%     RPSGW & 49.16 & 43.20 & 1.60 & 52.51 \\
%     IWRPSGW & 41.08 & 47.97 & 1.47 & 52.04 \\
%     RASGW & \underline{36.58} & 50.08 & \underline {1.26} & 79.92 \\
%     IWRASGW & \textbf{34.63} & 46.04 & \textbf{1.03} & 78.31 \\
%     \bottomrule
%     \end{tabular}
%     }
%     \label{tab:step10000_results}
% \end{wraptable}

\begin{wraptable}[12]{r}{0.6\textwidth}
    \centering
    % \vspace{-1.0\intextsep}
    \scriptsize
    \caption{\footnotesize{GW-2 distance and computational times for 2D$\rightarrow$3D and 3D$\rightarrow$2D generations in 4-point experiment based on 10 repetitions.}}
    %\scalebox{1}{
    \begin{tabular}{l|cc|cccccc}
    \toprule
     Method & \multicolumn{2}{c|}{Step 10000 (3D to 2D)} & \multicolumn{2}{c}{Step 10000 (2D to 3D)} \\
     \cmidrule{2-5}
     & $\text{GW}_2$($\downarrow$) & $\text{Time(s)}$($\downarrow$) & $\text{GW}_2$($\downarrow$) & $\text{Time(s)}$($\downarrow$) \\
     \midrule
    SGW & 122.54$\pm$11.94 & \underline{42.30$\pm$0.50} & 2.23$\pm$0.62 & \underline{41.90$\pm$0.43} \\
    Max-SGW & 385.28$\pm$8.98 & 213.45$\pm$2.23 & 2.25$\pm$0.25 & 210.98$\pm$2.76 \\
    DSGW & 102.95$\pm$3.59 & 247.68$\pm$5.65 & 1.64$\pm$0.40 & 242.39$\pm$4.87 \\
    EBSGW & 85.54$\pm$5.51 & \textbf{40.09$\pm$0.88}  & 1.79$\pm$0.70 &\textbf{41.39$\pm$0.63} \\
    RPSGW & 39.73$\pm$3.54 & 52.90$\pm$0.63 & 1.34$\pm$0.35 & 55.65$\pm$0.89 \\
    IWRPSGW & 37.44$\pm$3.23 & 53.12$\pm$1.09 & 1.31$\pm$0.45 & 57.84$\pm$1.21 \\
    RASGW & \textbf{18.80$\pm$2.46 } & 51.89$\pm$4.09 & \underline{1.28$\pm$0.54} & 59.03$\pm$3.54 \\
    IWRASGW &\underline{ 22.45$\pm$3.18} & 54.45$\pm$3.41 &\textbf{1.15$\pm$0.32} & 56.54$\pm$5.49 \\
    \bottomrule
    \end{tabular}
    %}
    \label{tab:step10000_results}
\end{wraptable}

\textbf{I. Learning Across Identical Spaces.} First, we consider both mm spaces to be similar, as in the case of a typical GAN. Specifically, we assess the model's ability to recover 2D Gaussian mixtures, a common task for mode recovery \citep{che2016mode, metz2016unrolled}. For this synthetic experiment, we follow the setup of \citet{bunne2019learning}, where both the generator and adversary are feed-forward networks with ReLU activations. The immediate consequence of plugging in RASGW as the loss is a significant improvement in convergence speed without compromising generation quality (see Section \ref{2d_2d}).

% \begin{figure}[ht]
%     \centering
%     \includegraphics[width=0.95\columnwidth]{2dto2d.png}
%     \caption{Generation of 2D distribution from 2D distribution}
%     \label{fig:2dto2d}
% \end{figure}

% \begin{figure}[ht]
%     \centering
%     \includegraphics[width=0.95\columnwidth]{GW_distance_improved_.pdf}
%     \caption{GW distances for 4-point 3D$\rightarrow$2D generations.}
%     \label{fig:gw_iter}
    
%     \vspace{0.5cm} % Adjust space between the images

%     \includegraphics[width=0.95\columnwidth]{Time_improved_.pdf}
%     \caption{Computation time for 4-point 3D$\rightarrow$2D generations.}
%     \label{fig:time_iter}
% \end{figure}

\textbf{II. Learning Across Dimensionalities.} We consider a mixture of bivariate Gaussians as the source distribution for a truly cross-domain experiment. The target is constructed as a tri-variate Gaussian mixture, each component being drawn around a vertex of a hypercube in $\mathbb{R}^3$ as its center. Here, we also use ReLU-activated generators. For a detailed discussion, we refer the reader to Section \ref{sec:incomparable}. 

\textbf{Results.} Based on our setting (see Section \ref{sec:architecture_gwgan} of Appendix), following \citet{bunne2019learning}, RASGW and IWRASGW bear superior outcomes. For the 2D source mixture distribution of 4 Gaussians (4-point experiment), Table \ref{tab:step10000_results} shows that the average GW-2 distance between the target and the generated 3D distributions is minimally realized using our models. The corresponding computational times are also comparable to SGW and other optimization-free variants, such as EBSGW, RPSGW and IWRPSGW. Figure \ref{fig:2_3} (see Section \ref{4pt}) compares the same for the generation of the base law. The efficient bidirectional generation hints at achieving lower relational discrepancy. Figure \ref{fig:2dto3d} and Figure \ref{fig:3dto2d} (see Appendix) illustrate the propagation of errors across iterations. Furthermore, RASGW and IWRASGW demonstrate significantly lower computational costs than DSGW and Max-SGW.

\textit{\textbf{Gromov-Wasserstein Autoencoder.}} We select autoencoders for the second generative experiment, which inherently considers unalike spaces in the form of the input space and the latent representation. Specifically, we modify the GW Autoencoder (GWAE) \citep{nakagawa2023gromovwasserstein} to reconstruct input image datasets (CIFAR-10 \citep{krizhevsky2009learning} and Omniglot \citep{liu2015faceattributes}). Instead of likelihood maximization, GWAE models learn representations by matching the geometric structure between the latent and data spaces. The training objective of GWAE between the metric measure spaces~$(\mathcal{X},d_{X},\mu)$ (input) and~$(\mathcal{Z},d_{Z},\pi_{\theta})$ (latent) is given as $\min_{\theta} \textrm{GW}_p^p(\mu, \pi_{\theta})$, where $\pi_{\theta}$ is the trainable latent prior. \citet{nakagawa2023gromovwasserstein} recast the GW objective into a main GW component~$\mathcal{L}_{gw}$ with three regularizations: a reconstruction loss~$\mathcal{L}_w$, a joint dual loss $\mathcal{L}_d$, and an entropy regularization~$\mathcal{L}_h$. We replace the objective using RASGW and IWRASGW and perform extensive experiments against competitors SGW and the other sliced variants for image reconstruction.

\begin{figure}[ht]
\centering
% \vspace{-1.0\intextsep}
\begin{minipage}{.48\textwidth}
  \centering
  \includegraphics[width=\textwidth]{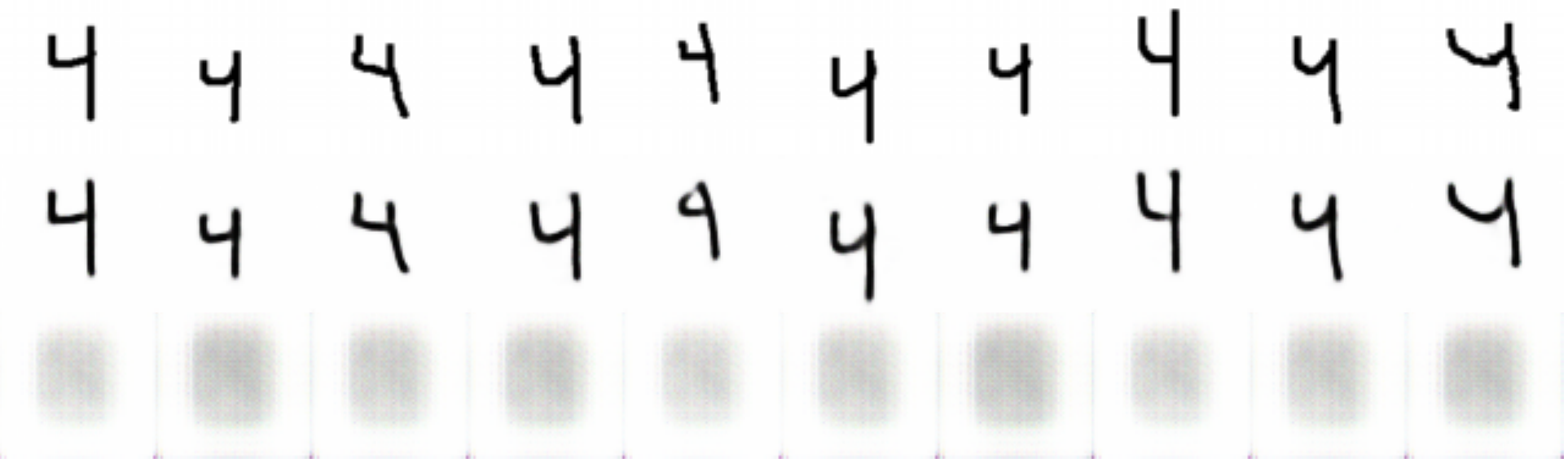}
  \label{fig:one-column-1}
\end{minipage}%
\hspace{6pt}
\begin{minipage}{.48\textwidth}
  \centering
  \includegraphics[width=\textwidth]{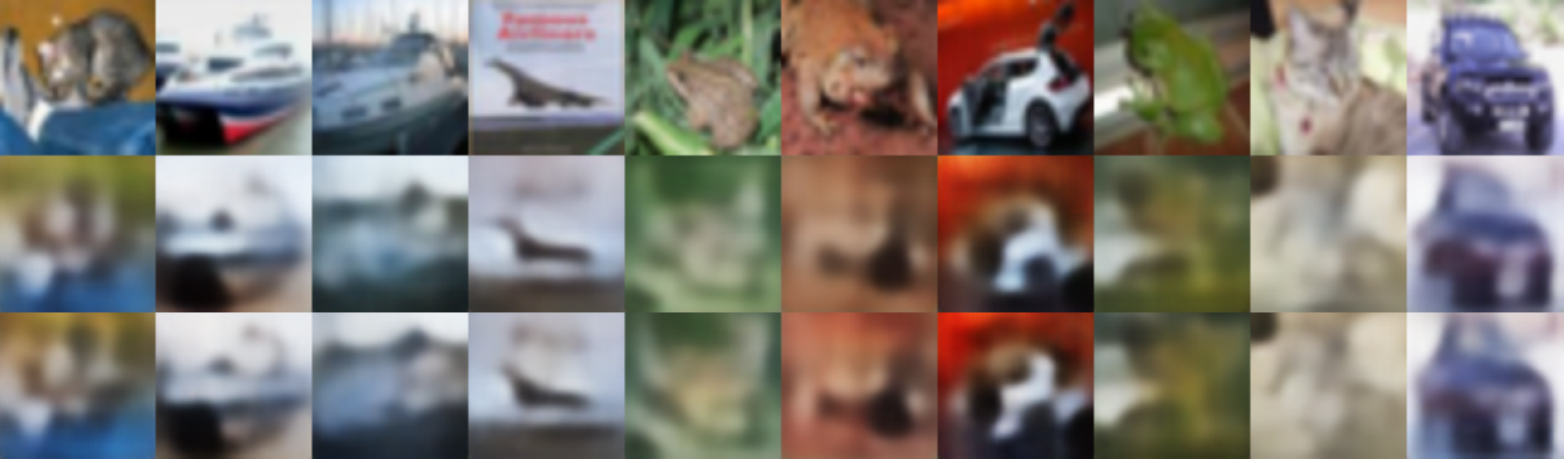}
  \label{fig:one-column-2}
\end{minipage}
\vspace{-8pt}
\caption{\textbf{(Left)} Target (\textit{top}), reconstructed Omniglot images (\textit{middle}) and outputs (\textit{bottom}) at initialization using IWRASGW. \textbf{(Right)} Input (\textit{top}), reconstructed images using RASGW at $40$th epoch (\textit{middle}) and reconstructed images using SGW at $48$th epoch (\textit{bottom}) for CIFAR-10.}
\label{fig:recon}
\end{figure}

% \begin{wrapfigure}{r}{0.6\textwidth}
%     \centering
%     \includegraphics[width=\linewidth]{omni.pdf}
%     \caption{(\textit{top}) Target, (\textit{middle}) reconstructed Omniglot images and (\textit{bottom}) outputs at initialization using IWRASGW.}
%     \label{fig:one-column-1}
% \end{wrapfigure}

% \begin{wrapfigure}{r}{0.6\textwidth}
%     \centering
%     \includegraphics[width=\linewidth]{cifar.pdf}
%     \caption{(\textit{top}) Input, (\textit{middle}) reconstructed images using RASGW at $40$th epoch and (\textit{bottom}) reconstructed images using SGW at $48$th epoch for CIFAR-10.}
%     \label{fig:one-column-2}
% \end{wrapfigure}

% Hence the GWAE objective differs from these other approaches as it aims to directly match the latent and data distributions based on their distance structure. Given an~$N$-sized set of data points~$\{\mathbb{x}_i\}_{i=1}^N$ supported on a data space $\mathcal{X}$, representation learning aims to build a latent space~$\mathcal{Z}$ and obtain mappings between both the spaces. A model is constructed with a trainable latent prior $\pi_{\theta}(z)$ to approach the data distribution $p_{data}(x)$ in
% terms of distance structure.

% Hence, setting $\rho=1$ they estimate and minimize the GW metric using the variational autoencoding scheme, which captures the latent factors of complex data in quite a stable manner.
% They   in CIFAR-10\cite{krizhevsky2009learning} and Omniglot\cite{liu2015faceattributes} datasets.
% \newline

\begin{wraptable}[11]{r}{0.5\textwidth}
    \centering
    \vspace{-1.5\intextsep}
    \caption{Results on CIFAR-10 (C) and Omniglot (O) datasets. FID scores ($\downarrow$), PSNR values (in dB, $\uparrow$), total training time (seconds, $\downarrow$), and epochs ($\downarrow$) are reported for each method. Columns C and O refer to CIFAR-10 and Omniglot, respectively.}
    \scalebox{0.6}{
    \begin{tabular}{lcccccccc}
    \toprule
    Method & \multicolumn{2}{c}{FID ($\downarrow$)} & \multicolumn{2}{c}{PSNR ($\uparrow$)} & \multicolumn{2}{c}{Total ($\downarrow$)} & \multicolumn{2}{c}{Epochs($\downarrow$)} \\
           & C & O & C & O & C & O & C & O \\
    \midrule
    SGW & 72.87  & \underline{20.95} & 22.65 & \textbf{22.13} & \underline{22459} & 16932 & 48 & 70 \\
    Max-SGW & 88.97 & 22.27 & 22.38 & 21.94 & 51345 & 42975 & \textbf{27} & 63 \\
    DSGW & 83.45 & 21.66 & 22.41 & \underline{22.10} & 59876 & 47585 & \underline{30} & 67 \\
    EBSGW & 78.50 & 38.99 & 22.48 & 20.53 & 28765 & \underline{15533} & 52 & \underline{55} \\
    RPSGW & 80.80 & 29.88 & 22.55 & 21.33 & 24356 & \textbf{12321} & 43 & \textbf{45} \\
    RASGW & \textbf{70.46} & 21.48 & \underline{22.70} & 20.56 & 23613 & 16731 & 40 & 62 \\
    IWRASGW & \underline{71.99} & \textbf{20.25} & \textbf{22.73} & 21.87 & \textbf{19879} & 18972 & 35 & 76 \\
    \bottomrule
    \end{tabular} \label{tab:final_results_combined_gwae}
    }
\end{wraptable}
\vspace{-4pt}
\textbf{Results.} We follow the architecture of \citet{nakagawa2023gromovwasserstein} as discussed in Section \ref{sec:architecture_gwae} of the Appendix. For all the experiments, we followed early stopping with a patience of 10 epochs, ensuring that training terminated if the validation performance did not improve for 10 consecutive epochs. In CIFAR-10, based on FID scores \citep{heusel2017gans}, RASGW and IWRASGW perform the best and the second best, respectively, whereas for Omniglot, IWRASGW performs the best with RASGW following closely in second. On PSNR scores, IWRASGW performs best, followed by SGW in CIFAR-10 (Table \ref{tab:final_results_combined_gwae}). The corresponding runtime attests to the efficacy of our models. For a discussion on the evaluation metrics, refer to Section \ref{sec:evaluation_metrics} in the Appendix. The qualitative results on reconstruction and plots of change in PSNR and loss values are placed in Section \ref{sec:results_cifar} (CIFAR-10) and \ref{sec:results_omniglot} (Omniglot).

\vspace{-8pt}
\section{Conclusion and Future Works}
\label{sec:conclusion}
We introduce a novel approach for developing informative projecting directions for SGW. Based on this, we derive the concept of relation-aware slicing distributions along with two optimization-free SGW variants, namely RASGW and IWRASGW, that preserve topological (metric) properties and improve computational cost. It improves over uniform and Max slicing based on informative projections that reflect in cross-domain generative tasks over distinct spaces. We demonstrate the competitive performance of RASGW and IWRASGW within the contexts of GW-GANs and GWAE. Future research might involve integrating the Wasserstein distance into our framework to develop a fused metric, better serving attributed network spaces.

% \section{Impact Statement}
% This paper presents work whose goal is to advance the field of Machine Learning. There are many potential societal consequences of our work, none of which we feel must be specifically highlighted here.

\bibliography{example_paper}

\begin{thebibliography}{50}
\providecommand{\natexlab}[1]{#1}
\providecommand{\url}[1]{\texttt{#1}}
\expandafter\ifx\csname urlstyle\endcsname\relax
  \providecommand{\doi}[1]{doi: #1}\else
  \providecommand{\doi}{doi: \begingroup \urlstyle{rm}\Url}\fi

\bibitem[Altschuler et~al.(2017)Altschuler, Niles-Weed, and Rigollet]{altschuler2017near}
Jason Altschuler, Jonathan Niles-Weed, and Philippe Rigollet.
\newblock Near-linear time approximation algorithms for optimal transport via {S}inkhorn iteration.
\newblock In \emph{Advances in Neural Information Processing Systems}, pages 1964--1974, 2017.

\bibitem[Beinert et~al.(2023)Beinert, Heiss, and Steidl]{beinert2023assignment}
Robert Beinert, Cosmas Heiss, and Gabriele Steidl.
\newblock On assignment problems related to gromov--wasserstein distances on the real line.
\newblock \emph{SIAM Journal on Imaging Sciences}, 16\penalty0 (2):\penalty0 1028--1032, 2023.

\bibitem[Bonneel et~al.(2015)Bonneel, Rabin, Peyr{\'e}, and Pfister]{bonneel2015sliced}
Nicolas Bonneel, Julien Rabin, Gabriel Peyr{\'e}, and Hanspeter Pfister.
\newblock Sliced and {R}adon {W}asserstein barycenters of measures.
\newblock \emph{Journal of Mathematical Imaging and Vision}, 1\penalty0 (51):\penalty0 22--45, 2015.

\bibitem[Bonnotte(2013)]{bonnotte2013unidimensional}
Nicolas Bonnotte.
\newblock \emph{Unidimensional and evolution methods for optimal transportation}.
\newblock PhD thesis, Paris 11, 2013.

\bibitem[Bunne et~al.(2019)Bunne, Alvarez-Melis, Krause, and Jegelka]{bunne2019learning}
Charlotte Bunne, David Alvarez-Melis, Andreas Krause, and Stefanie Jegelka.
\newblock Learning generative models across incomparable spaces.
\newblock In \emph{International Conference on Machine Learning}, pages 851--861, 2019.

\bibitem[Cao et~al.(2018)Cao, Ma, Long, and Wang]{cao2018partial}
Zhangjie Cao, Lijia Ma, Mingsheng Long, and Jianmin Wang.
\newblock Partial adversarial domain adaptation.
\newblock In \emph{Proceedings of the European Conference on Computer Vision (ECCV)}, pages 135--150, 2018.

\bibitem[Che et~al.(2016)Che, Li, Jacob, Bengio, and Li]{che2016mode}
Tong Che, Yanran Li, Athul~Paul Jacob, Yoshua Bengio, and Wenjie Li.
\newblock Mode regularized generative adversarial networks.
\newblock \emph{arXiv preprint arXiv:1612.02136}, 2016.

\bibitem[Chen et~al.(2020)Chen, Yang, and Li]{chen2020augmented}
Xiongjie Chen, Yongxin Yang, and Yunpeng Li.
\newblock Augmented sliced wasserstein distances.
\newblock \emph{arXiv preprint arXiv:2006.08812}, 2020.

\bibitem[Chowdhury and M{\'e}moli(2019)]{chowdhury2019gromov}
Samir Chowdhury and Facundo M{\'e}moli.
\newblock The gromov--wasserstein distance between networks and stable network invariants.
\newblock \emph{Information and Inference: A Journal of the IMA}, 8\penalty0 (4):\penalty0 757--787, 2019.

\bibitem[Cuturi(2013)]{cuturi2013sinkhorn}
Marco Cuturi.
\newblock Sinkhorn distances: Lightspeed computation of optimal transport.
\newblock In \emph{Advances in Neural Information Processing Systems}, pages 2292--2300, 2013.

\bibitem[De~Cao and Aziz(2020)]{de2020power}
Nicola De~Cao and Wilker Aziz.
\newblock The power spherical distribution.
\newblock \emph{arXiv preprint arXiv:2006.04437}, 2020.

\bibitem[Deng et~al.(2009)Deng, Dong, Socher, Li, Li, and Fei-Fei]{deng2009imagenet}
Jia Deng, Wei Dong, Richard Socher, Li-Jia Li, Kai Li, and Li~Fei-Fei.
\newblock Imagenet: A large-scale hierarchical image database.
\newblock In \emph{2009 IEEE conference on computer vision and pattern recognition}, pages 248--255. Ieee, 2009.

\bibitem[Deshpande et~al.(2018)Deshpande, Zhang, and Schwing]{deshpande2018generative}
Ishan Deshpande, Ziyu Zhang, and Alexander~G Schwing.
\newblock Generative modeling using the sliced {W}asserstein distance.
\newblock In \emph{Proceedings of the IEEE Conference on Computer Vision and Pattern Recognition}, pages 3483--3491, 2018.

\bibitem[Deshpande et~al.(2019)Deshpande, Hu, Sun, Pyrros, Siddiqui, Koyejo, Zhao, Forsyth, and Schwing]{deshpande2019max}
Ishan Deshpande, Yuan-Ting Hu, Ruoyu Sun, Ayis Pyrros, Nasir Siddiqui, Sanmi Koyejo, Zhizhen Zhao, David Forsyth, and Alexander~G Schwing.
\newblock Max-sliced {W}asserstein distance and its use for {GAN}s.
\newblock In \emph{Proceedings of the IEEE Conference on Computer Vision and Pattern Recognition}, pages 10648--10656, 2019.

\bibitem[Devroye et~al.(2013)Devroye, Gy{\"o}rfi, and Lugosi]{devroye2013probabilistic}
Luc Devroye, L{\'a}szl{\'o} Gy{\"o}rfi, and G{\'a}bor Lugosi.
\newblock \emph{A probabilistic theory of pattern recognition}, volume~31.
\newblock Springer Science \& Business Media, 2013.

\bibitem[Figurnov et~al.(2018)Figurnov, Mohamed, and Mnih]{figurnov2018implicit}
Mikhail Figurnov, Shakir Mohamed, and Andriy Mnih.
\newblock Implicit reparameterization gradients.
\newblock \emph{Advances in neural information processing systems}, 31, 2018.

\bibitem[Goodfellow et~al.(2014)Goodfellow, Pouget-Abadie, Mirza, Xu, Warde-Farley, Ozair, Courville, and Bengio]{goodfellow2014generative}
Ian Goodfellow, Jean Pouget-Abadie, Mehdi Mirza, Bing Xu, David Warde-Farley, Sherjil Ozair, Aaron Courville, and Yoshua Bengio.
\newblock Generative adversarial nets.
\newblock In \emph{Advances in Neural Information Processing Systems}, pages 2672--2680, 2014.

\bibitem[Guo et~al.(2020)Guo, Ho, and Jordan]{guo2020fast}
Wenshuo Guo, Nhat Ho, and Michael Jordan.
\newblock Fast algorithms for computational optimal transport and wasserstein barycenter.
\newblock In \emph{International Conference on Artificial Intelligence and Statistics}, pages 2088--2097. PMLR, 2020.

\bibitem[Helgason(2011)]{helgason2011radon}
Sigurdur Helgason.
\newblock The radon transform on r n.
\newblock In \emph{Integral Geometry and Radon Transforms}, pages 1--62. Springer, 2011.

\bibitem[Hendrycks and Gimpel(2023)]{hendrycks2023gaussianerrorlinearunits}
Dan Hendrycks and Kevin Gimpel.
\newblock Gaussian error linear units (gelus), 2023.
\newblock URL \url{https://arxiv.org/abs/1606.08415}.

\bibitem[Heusel et~al.(2017)Heusel, Ramsauer, Unterthiner, Nessler, and Hochreiter]{heusel2017gans}
Martin Heusel, Hubert Ramsauer, Thomas Unterthiner, Bernhard Nessler, and Sepp Hochreiter.
\newblock Gans trained by a two time-scale update rule converge to a local nash equilibrium.
\newblock \emph{Advances in neural information processing systems}, 30, 2017.

\bibitem[Hur et~al.(2024)Hur, Guo, and Liang]{hur2024reversible}
YoonHaeng Hur, Wenxuan Guo, and Tengyuan Liang.
\newblock Reversible gromov--monge sampler for simulation-based inference.
\newblock \emph{SIAM Journal on Mathematics of Data Science}, 6\penalty0 (2):\penalty0 283--310, 2024.

\bibitem[Jupp and Mardia(1979)]{jupp1979maximum}
Peter~E Jupp and Kanti~V Mardia.
\newblock Maximum likelihood estimators for the matrix von {M}ises-{F}isher and bingham distributions.
\newblock \emph{The Annals of Statistics}, 7\penalty0 (3):\penalty0 599--606, 1979.

\bibitem[Khamis et~al.(2024)Khamis, Tsuchida, Tarek, Rolland, and Petersson]{khamis2024scalable}
Abdelwahed Khamis, Russell Tsuchida, Mohamed Tarek, Vivien Rolland, and Lars Petersson.
\newblock Scalable optimal transport methods in machine learning: A contemporary survey.
\newblock \emph{IEEE Transactions on Pattern Analysis and Machine Intelligence}, 2024.

\bibitem[Kingma and Welling(2013)]{kingma2013auto}
Diederik~P Kingma and Max Welling.
\newblock Auto-encoding variational bayes.
\newblock \emph{arXiv preprint arXiv:1312.6114}, 2013.

\bibitem[Kolouri et~al.(2016)Kolouri, Zou, and Rohde]{kolouri2016sliced}
Soheil Kolouri, Yang Zou, and Gustavo~K Rohde.
\newblock Sliced {W}asserstein kernels for probability distributions.
\newblock In \emph{Proceedings of the IEEE Conference on Computer Vision and Pattern Recognition}, pages 5258--5267, 2016.

\bibitem[Kolouri et~al.(2019)Kolouri, Nadjahi, Simsekli, Badeau, and Rohde]{kolouri2019generalized}
Soheil Kolouri, Kimia Nadjahi, Umut Simsekli, Roland Badeau, and Gustavo Rohde.
\newblock Generalized sliced wasserstein distances.
\newblock \emph{Advances in neural information processing systems}, 32, 2019.

\bibitem[Krizhevsky et~al.(2009)Krizhevsky, Hinton, et~al.]{krizhevsky2009learning}
Alex Krizhevsky, Geoffrey Hinton, et~al.
\newblock Learning multiple layers of features from tiny images.
\newblock \emph{Master's thesis, Department of Computer Science, University of Toronto}, 2009.

\bibitem[Lake et~al.(2015)Lake, Salakhutdinov, and Tenenbaum]{lake2015human}
Brenden~M Lake, Ruslan Salakhutdinov, and Joshua~B Tenenbaum.
\newblock Human-level concept learning through probabilistic program induction.
\newblock \emph{Science}, 350\penalty0 (6266):\penalty0 1332--1338, 2015.

\bibitem[Liu et~al.(2015)Liu, Luo, Wang, and Tang]{liu2015faceattributes}
Ziwei Liu, Ping Luo, Xiaogang Wang, and Xiaoou Tang.
\newblock Deep learning face attributes in the wild.
\newblock In \emph{Proceedings of International Conference on Computer Vision (ICCV)}, December 2015.

\bibitem[M{\'e}moli(2011)]{memoli2011gromov}
Facundo M{\'e}moli.
\newblock Gromov--wasserstein distances and the metric approach to object matching.
\newblock \emph{Foundations of computational mathematics}, 11:\penalty0 417--487, 2011.

\bibitem[Metz et~al.(2016)Metz, Poole, Pfau, and Sohl-Dickstein]{metz2016unrolled}
Luke Metz, Ben Poole, David Pfau, and Jascha Sohl-Dickstein.
\newblock Unrolled generative adversarial networks.
\newblock \emph{arXiv preprint arXiv:1611.02163}, 2016.

\bibitem[Nadjahi et~al.(2019)Nadjahi, Durmus, Simsekli, and Badeau]{nadjahi2019asymptotic}
Kimia Nadjahi, Alain Durmus, Umut Simsekli, and Roland Badeau.
\newblock Asymptotic guarantees for learning generative models with the sliced-{W}asserstein distance.
\newblock In \emph{Advances in Neural Information Processing Systems}, pages 250--260, 2019.

\bibitem[Nadjahi et~al.(2020)Nadjahi, Durmus, Chizat, Kolouri, Shahrampour, and Simsekli]{nadjahi2020statistical}
Kimia Nadjahi, Alain Durmus, L{\'e}na{\"\i}c Chizat, Soheil Kolouri, Shahin Shahrampour, and Umut Simsekli.
\newblock Statistical and topological properties of sliced probability divergences.
\newblock \emph{Advances in Neural Information Processing Systems}, 33:\penalty0 20802--20812, 2020.

\bibitem[Nakagawa et~al.(2023)Nakagawa, Togo, Ogawa, and Haseyama]{nakagawa2023gromovwasserstein}
Nao Nakagawa, Ren Togo, Takahiro Ogawa, and Miki Haseyama.
\newblock Gromov-wasserstein autoencoders.
\newblock In \emph{The Eleventh International Conference on Learning Representations}, 2023.

\bibitem[Nguyen and Ho(2023)]{nguyen2023energy}
Khai Nguyen and Nhat Ho.
\newblock Energy-based sliced {W}asserstein distance.
\newblock \emph{Advances in Neural Information Processing Systems}, 2023.

\bibitem[Nguyen and Ho(2024)]{nguyen2024hierarchical}
Khai Nguyen and Nhat Ho.
\newblock Hierarchical hybrid sliced wasserstein: A scalable metric for heterogeneous joint distributions.
\newblock In \emph{The Thirty-eighth Annual Conference on Neural Information Processing Systems}, 2024.
\newblock URL \url{https://openreview.net/forum?id=XwrMd1njqq}.

\bibitem[Nguyen et~al.(2021)Nguyen, Ho, Pham, and Bui]{nguyen2021distributional}
Khai Nguyen, Nhat Ho, Tung Pham, and Hung Bui.
\newblock Distributional sliced-{W}asserstein and applications to generative modeling.
\newblock In \emph{International Conference on Learning Representations}, 2021.

\bibitem[Nguyen et~al.(2023)Nguyen, Ren, Nguyen, Rout, Nguyen, and Ho]{nguyen2023hierarchical}
Khai Nguyen, Tongzheng Ren, Huy Nguyen, Litu Rout, Tan~Minh Nguyen, and Nhat Ho.
\newblock Hierarchical sliced wasserstein distance.
\newblock In \emph{The Eleventh International Conference on Learning Representations}, 2023.

\bibitem[Nguyen et~al.(2024)Nguyen, Zhang, Le, and Ho]{nguyen2024slicedwassersteinrandompathprojecting}
Khai Nguyen, Shujian Zhang, Tam Le, and Nhat Ho.
\newblock Sliced wasserstein with random-path projecting directions.
\newblock In \emph{Proceedings of the 41st International Conference on Machine Learning}, ICML'24, 2024.

\bibitem[Nietert et~al.(2022)Nietert, Sadhu, Goldfeld, and Kato]{nietert2022statistical}
Sloan Nietert, Ritwik Sadhu, Ziv Goldfeld, and Kengo Kato.
\newblock Statistical, robustness, and computational guarantees for sliced {W}asserstein distances.
\newblock \emph{Advances in Neural Information Processing Systems}, 2022.

\bibitem[Peyr{\'e} et~al.(2016)Peyr{\'e}, Cuturi, and Solomon]{peyre2016gromov}
Gabriel Peyr{\'e}, Marco Cuturi, and Justin Solomon.
\newblock Gromov-wasserstein averaging of kernel and distance matrices.
\newblock In \emph{International conference on machine learning}, pages 2664--2672. PMLR, 2016.

\bibitem[Rioux et~al.(2024)Rioux, Goldfeld, and Kato]{rioux2024entropic}
Gabriel Rioux, Ziv Goldfeld, and Kengo Kato.
\newblock Entropic gromov-wasserstein distances: Stability and algorithms.
\newblock \emph{Journal of Machine Learning Research}, 25\penalty0 (363):\penalty0 1--52, 2024.

\bibitem[Scetbon et~al.(2022)Scetbon, Peyr{\'e}, and Cuturi]{scetbon2022linear}
Meyer Scetbon, Gabriel Peyr{\'e}, and Marco Cuturi.
\newblock Linear-time gromov wasserstein distances using low rank couplings and costs.
\newblock In \emph{International Conference on Machine Learning}, pages 19347--19365. PMLR, 2022.

\bibitem[Solomon et~al.(2016)Solomon, Peyr{\'e}, Kim, and Sra]{solomon2016entropic}
Justin Solomon, Gabriel Peyr{\'e}, Vladimir~G Kim, and Suvrit Sra.
\newblock Entropic metric alignment for correspondence problems.
\newblock \emph{ACM Transactions on Graphics (ToG)}, 35\penalty0 (4):\penalty0 1--13, 2016.

\bibitem[Sra(2016)]{Suvrit_directional}
Suvrit Sra.
\newblock Directional statistics in machine learning: a brief review.
\newblock \emph{arXiv preprint arXiv:1605.00316}, 2016.

\bibitem[Temme(2011)]{temme2011special}
Nico~M Temme.
\newblock \emph{Special functions: An introduction to the classical functions of mathematical physics}.
\newblock John Wiley \& Sons, 2011.

\bibitem[Vayer et~al.(2019)Vayer, Flamary, Tavenard, Chapel, and Courty]{vayer2019sliced}
Titouan Vayer, R{\'e}mi Flamary, Romain Tavenard, Laetitia Chapel, and Nicolas Courty.
\newblock Sliced gromov-wasserstein.
\newblock \emph{arXiv preprint arXiv:1905.10124}, 2019.

\bibitem[Wainwright(2019)]{wainwrighthigh}
Martin~J Wainwright.
\newblock \emph{High-dimensional statistics: A non-asymptotic viewpoint}.
\newblock Cambridge University Press, 2019.

\bibitem[Zhang et~al.(2024)Zhang, Goldfeld, Mroueh, and Sriperumbudur]{zhang2024gromov}
Zhengxin Zhang, Ziv Goldfeld, Youssef Mroueh, and Bharath~K Sriperumbudur.
\newblock Gromov--wasserstein distances: Entropic regularization, duality and sample complexity.
\newblock \emph{The Annals of Statistics}, 52\penalty0 (4):\penalty0 1616--1645, 2024.

\end{thebibliography}
\bibliographystyle{plainnat}

%%%%%%%%%%%%%%%%%%%%%%%%%%%%%%%%%%%%%%%%%%%%%%%%%%%%%%%%%%%%%%%%%%%%%%%%%%%%%%%
%%%%%%%%%%%%%%%%%%%%%%%%%%%%%%%%%%%%%%%%%%%%%%%%%%%%%%%%%%%%%%%%%%%%%%%%%%%%%%%
% APPENDIX
%%%%%%%%%%%%%%%%%%%%%%%%%%%%%%%%%%%%%%%%%%%%%%%%%%%%%%%%%%%%%%%%%%%%%%%%%%%%%%%
%%%%%%%%%%%%%%%%%%%%%%%%%%%%%%%%%%%%%%%%%%%%%%%%%%%%%%%%%%%%%%%%%%%%%%%%%%%%%%%
\newpage
\appendix

% \begin{center}
% {\bf{\Large{Supplement to ``Sliced Gromov Wasserstein with Relation Aware Projections"}}}
% \end{center}
\section*{\Large{Appendix}} \label{appendix}
First, we present skipped proofs in the main text in Appendix~\ref{sec:proof}. We then discuss some related works in~\ref{sec:additional_baselines}. We provide the algorithms in Appendix \ref{sec:algorithm}. After that, we provide some additional experimental details in Appendix~\ref{sec:add_exps}. Finally, we provide details on the computational infrastructure in Appendix \ref{sec:comp_infra}. For codes for our experiments, we refer the reader to the repository \url{https://anonymous.4open.science/r/Relation-Aware-Slicing-in-Cross-Domain-Alignment-887A/README.md}.
\section{Technical Proofs}
\label{sec:proof}

\subsection{Proof of Theorem~\ref{theo:metricity}}
\label{subsec:proof:theo:metricity}

\textbf{Non-negativity.} Since the Gromov-Wasserstein distance is non-negative, we have $\textrm{GW}_p(\theta \sharp \mu,\theta \sharp \nu)\geq 0$ for all $\theta \in \mathbb{S}^{d-1}$. Therefore, $\mathbb{E}_{\theta \sim \sigma_{\textrm{RA}}(\theta;\mu,\nu,\sigma_\kappa)}[\textrm{GW}_p(\theta \sharp \mu,\theta \sharp \nu)] \geq 0$ which leads to $\text{RASGW}_p(\mu,\nu;\sigma_\kappa)\geq 0$. Similarly, we have $\mathbb{E}_{\theta_1,\ldots, \theta_L\sim \sigma_{\textrm{RA}}(\theta;\mu,\nu,\sigma_\kappa)} \left[\sum_{l=1}^L \textrm{GW}_p^p(\theta_l \sharp \mu, \theta_l \sharp \nu) \frac{f(\textrm{GW}_p^p(\theta_l \sharp \mu, \theta_l \sharp \nu))}{\sum_{j=1}^L f(\textrm{GW}_p^p(\theta_j \sharp \mu, \theta_j \sharp \nu))} \right] \geq 0 $ which implies $\text{IWRASGW}_p(\mu,\nu;\sigma_\kappa,L)\geq 0$

\textbf{Symmetry.} From the definition of RASGW and $f_{D_{\kappa}(x,x',y,y')}(\theta)$ as the density function of $D_\kappa(x,x',y,y')$ we have
\begin{align*}
    &\textrm{GW}_p^p(\theta \sharp \mu, \theta \sharp \nu) f_{D_{\kappa}(x,x',y,y')} (\theta) \\
    =& \frac{1}{2}\textrm{GW}_p^p(\theta \sharp \mu, \theta \sharp \nu) f_{\sigma_{\kappa}(Z_{x,x',y,y'})}(\theta) + \frac{1}{2}\textrm{GW}_p^p(\theta \sharp \mu, \theta \sharp \nu) f_{\sigma_{\kappa}(Z'_{x,x',y,y'})} (\theta) \nonumber \\
    =& \frac{1}{2}\textrm{GW}_p^p(\theta \sharp \mu, \theta \sharp \nu) f_{\sigma_{\kappa}(Z_{y,y',x,x'})}(\theta) + \frac{1}{2}\textrm{GW}_p^p(-\theta \sharp \mu, -\theta \sharp \nu) f_{\sigma_{\kappa}(Z'_{y,y',x,x'})} (-\theta).
\end{align*}
Now,
\begin{align*}
     &\int_{\mathbb{S}^{d-1}} \textrm{GW}_p^p(\theta \sharp \mu, \theta \sharp \nu) f_{D_{\kappa}(x,x',y,y')} (\theta) d\theta \nonumber \\
     =& \frac{1}{2} \int_{\mathbb{S}^{d-1}} \textrm{GW}_p^p(\theta \sharp \mu, \theta \sharp \nu) f_{\sigma_{\kappa}(Z_{y,y',x,x'})}(\theta) d\theta + \frac{1}{2} \int_{\mathbb{S}^{d-1}} \textrm{GW}_p^p(-\theta \sharp \mu, -\theta \sharp \nu) f_{\sigma_{\kappa}(Z'_{y,y',x,x'})} (-\theta) d\theta 
     \nonumber \\
     =& \frac{1}{2} \int_{\mathbb{S}^{d-1}} \textrm{GW}_p^p(\theta \sharp \mu, \theta \sharp \nu) f_{\sigma_{\kappa}(Z_{y,y',x,x'})}(\theta) d\theta + \frac{1}{2} \int_{\mathbb{S}^{d-1}} \textrm{GW}_p^p(\theta \sharp \mu, \theta \sharp \nu) f_{\sigma_{\kappa}(Z'_{y,y',x,x'})} (\theta) d\theta
     \nonumber \\
     =& \int_{\mathbb{S}^{d-1}} \textrm{GW}_p^p(\theta \sharp \nu, \theta \sharp \mu) f_{D_{\kappa}(y,y',x,x')} (\theta) d\theta,
\end{align*}
where we use the reflection property of $\mathbb{S}^{d-1}$,  i.e.,
\begin{align*}
    \textrm{GW}_p^p(\theta \sharp \mu,\theta \sharp \nu)&= \inf_{\pi \in \Pi(\mu,\nu)} \int |\theta^\top(x - y)|^p d\pi(x,y) \\
    &=\inf_{\pi \in \Pi(\mu,\nu)} \int |-\theta^\top(x - y)|^p d\pi(x,y)  = \textrm{GW}_p^p(-\theta \sharp \mu,-\theta \sharp \nu),
\end{align*}
and $$f_{\sigma_\kappa(\theta;Z_{x,x',y,y'})}(\theta) = f_{\sigma_\kappa(\theta;Z_{y,y',x,x'})}(\theta)$$ $$f_{\sigma_\kappa(\theta;Z'_{x,x',y,y'})}(\theta) = f_{\sigma_\kappa(\theta;Z'_{y,y',x,x'})}(-\theta)$$ which holds for both the vMF density $f_{\sigma_\kappa(\theta;\epsilon)}(\theta) \propto \exp\left(\kappa \epsilon^\top \theta\right)$ and the PS density $f_{\sigma_\kappa(\theta;\epsilon)}(\theta) \propto \left(1+\epsilon^\top \theta\right)^\kappa $.
After integrating over the measures, we get, 
\[
\text{RASGW}^p_p(\mu,\nu;D_\kappa) = \text{RASGW}^p_p(\nu,\mu;D_\kappa)
\]
\iffalse
\begin{align*}
    &\text{RASGW}^p_p(\mu,\nu;\sigma_\kappa) \nonumber \\ &=\mathbb{E}_{X, X' \sim \mu,Y,Y' \sim \nu} \mathbb{E}_{\theta \sim D_{\kappa}(X,X',Y,Y')}[\textrm{GW}_p^p(\theta \sharp \mu, \theta \sharp \nu)] \\
    &=\int_{\mathbb{R}^d \times \mathbb{R}^d} \int_{\mathbb{R}^d \times \mathbb{R}^d}  \int_{\mathbb{S}^{d-1}} \textrm{GW}_p^p(\theta \sharp \mu, \theta \sharp \nu) f_{D_{\kappa}(x,x',y,y')} (\theta)  d\theta d \mu\otimes\mu(x,x') d \nu\otimes\nu(y,y') \\
    &=\int_{\mathbb{R}^d} \int_{\mathbb{R}^d}  \int_{\mathbb{S}^{d-1}} \textrm{GW}_p^p(-\theta \sharp \mu, -\theta \sharp \nu) f_{\sigma_\kappa(P_{\mathbb{S}^{d-1}}(x-y))} (-\theta)  d\theta d \mu\otimes\mu(x,x') d \nu\otimes\nu(y,y') \\
&=\int_{\mathbb{R}^d} \int_{\mathbb{R}^d}  \int_{\mathbb{S}^{d-1}} \textrm{GW}_p^p(\theta \sharp \mu, \theta \sharp \nu) f_{\sigma_\kappa(P_{\mathbb{S}^{d-1}}(y-x))} (\theta)  d\theta d \mu(x) d \nu(y)\\
&=\int_{\mathbb{R}^d} \int_{\mathbb{R}^d}  \int_{\mathbb{S}^{d-1}} \textrm{GW}_p^p(\theta \sharp \nu, \theta \sharp \mu) f_{\sigma_\kappa(P_{\mathbb{S}^{d-1}}(y-x))} (\theta)  d\theta  d \nu(y)d \mu(x)\\
    &=\mathbb{E}_{Y,Y'\sim \nu,X,X' \sim \mu} \mathbb{E}_{\theta \sim\sigma_\kappa(\theta;Z_{X,X',Y,Y'}) }[\textrm{GW}_p^p(\theta \sharp \nu, \theta \sharp \mu)]  = \text{RASGW}^p_p(\nu,\mu;\sigma_\kappa),
\end{align*}
\fi
 Similarly, we have $
    \text{IWRASGW}^p_p(\mu,\nu;\sigma_\kappa,L,f) =\text{IWRASGW}^p_p(\nu,\mu;\sigma_\kappa,L,f).$
    
\textbf{Existence of Isometric Isomorphism.}  We need to show that $\text{RASGW}_{p}(\mu,\nu;\sigma_k) = 0 $ if and only if $\exists$ an isometric isomorphism between $\mu$ and $\nu$. For the forward direction, since $\textrm{GW}_p(\theta\sharp \mu,\theta \sharp \nu)=0$ when $\mu$ and $\nu$ have an isometric isomorphism between them, we obtain directly $\text{RASGW}_{p}(\mu,\nu;\sigma_k) = 0 $. For the reverse direction, we use the same proof technique in~\cite{vayer2019sliced}. If $\text{RASGW}_{p}(\mu,\nu;\sigma_k) = 0$, we have $\int_{\mathbb{S}^{d-1}}\text{GW}_p\left(\theta {\sharp} \mu, \theta \sharp \nu\right) \mathrm{d} \sigma_{\textrm{RA}}(\theta;\mu,\nu,\sigma_k)=0$. Hence, we have $\textrm{GW}_p(\theta \sharp \mu,\theta \sharp \nu) =0 $ for $\sigma_{\textrm{RA}}(\theta;\mu,\nu,\sigma_k)$-almost surely $\theta \in \mathbb{S}^{d-1}$. Since $\sigma_{\textrm{RA}}(\theta;\mu,\nu,\sigma_k)$ is continuous due to the continuity of $\sigma_k$, we have  $\textrm{GW}_p(\theta \sharp \mu,\theta \sharp \nu) =0 $ for all $\theta \in \mathbb{S}^{d-1}$. Therefore, the measures are isometrically isomorphic. This follows from the proof of Theorem 10 in \citet{vayer2019sliced}.

For IWRASGW,  when $\mu$, $\nu$, are isometrically isomorphic we have $\textrm{GW}_p(\theta_l \sharp \mu,\theta_l \sharp \nu) = 0$ for any $\theta_1,\ldots,\theta_L \in \mathbb{S}^{d-1}$. Therefore, we have $\sum_{l=1}^L \textrm{GW}_p^p(\theta_l \sharp \mu, \theta_l \sharp \nu) \frac{f(\textrm{GW}_p^p(\theta_l \sharp \mu, \theta_l \sharp \nu))}{\sum_{j=1}^L \textrm{GW}_p^p(\theta_j \sharp \mu, \theta_j \sharp \nu)}=0$ for any $\theta_1,\ldots,\theta_L \in \mathbb{S}^{d-1}$ which implies 
\begin{align*}
    \text{IWRASGW}_p(\mu,\nu;\sigma_\kappa,L)= \mathbb{E} \left[\sum_{l=1}^L \textrm{GW}_p^p(\theta_l \sharp \mu, \theta_l \sharp \nu) \frac{f(\textrm{GW}_p^p(\theta_l \sharp \mu, \theta_l \sharp \nu))}{\sum_{j=1}^L \textrm{GW}_p^p(\theta_j \sharp \mu, \theta_j \sharp \nu)}\right]= 0.
\end{align*}
In the reverse direction, when  $\text{IWRASGW}_p(\mu,\nu;\sigma_\kappa,L)=0$, it means that 
we have $\sum_{l=1}^L \textrm{GW}_p^p(\theta_l \sharp \mu, \theta_l \sharp \nu) \frac{f(\textrm{GW}_p^p(\theta_l \sharp \mu, \theta_l \sharp \nu))}{\sum_{j=1}^L \textrm{GW}_p^p(\theta_j \sharp \mu, \theta_j \sharp \nu)}=0$ 
for any $\theta_1,\ldots,\theta_L \in \mathbb{S}^{d-1}$. Since $f(\textrm{GW}_p^p(\theta_l \sharp \mu, \theta_l \sharp \nu))>0$ for any $\theta_j$, it implies that $\textrm{GW}_p^p(\theta_l \sharp \mu, \theta_l \sharp \nu)=0$ for all $\theta_l \in \mathbb{S}^{d-1}$. With similar arguments to the proof of RASGW, we obtain $\mu$, $\nu$ are isometrically isomorphic, which completes the proof.

\textbf{Quasi-Triangle Inequality.}  Given three probability measures $\mu_1,\mu_2,\mu_3 \in \mathcal{P}_p(\mathbb{R}^d)$ we have:
\begin{align*}
    &\text{RASGW}_p(\mu_1,\mu_2;\sigma_\kappa) =\left(\mathbb{E}_{\theta \sim \sigma_{\textrm{RA}}(\theta;\mu_1,\mu_2,\sigma_\kappa)}[\textrm{GW}_p^p(\theta \sharp \mu_1, \theta \sharp \mu_2)]\right)^{\frac{1}{p}} \\
    \leq& \left(\mathbb{E}_{\theta \sim \sigma_{\textrm{RA}}(\theta;\mu_1,\mu_2,\sigma_\kappa)}[(\textrm{GW}_p(\theta \sharp \mu_1, \theta \sharp \mu_3)+ \textrm{GW}_p(\theta \sharp \mu_3, \theta \sharp \mu_2))^p]\right)^{\frac{1}{p}} \\
    \leq& \left(\mathbb{E}_{\theta \sim \sigma_{\textrm{RA}}(\theta;\mu_1,\mu_2,\sigma_\kappa)}[\textrm{GW}_p^p(\theta \sharp \mu_1, \theta \sharp \mu_3)]\right)^{\frac{1}{p}} +\left(\mathbb{E}_{\theta \sim \sigma_{\textrm{RA}}(\theta;\mu_1,\mu_2,\sigma_\kappa)}[\textrm{GW}_p^p(\theta \sharp \mu_3, \theta \sharp \mu_2)]\right)^{\frac{1}{p}} \\
    =& \text{RASGW}_{p}(\mu_1,\mu_3;\sigma_\kappa,\mu_1,\mu_2) + \text{RASGW}_{p}(\mu_3,\mu_2;\sigma_\kappa,\mu_1,\mu_2), 
\end{align*}
where the first inequality is due to the triangle inequality of Wasserstein distance and the second inequality is due to the Minkowski inequality. We complete the proof here.
\subsection{Proof of Proposition~\ref{prop:connection}}
\label{subsec:proof:prop:connection}

\textit{(i)} To prove that $\text{RASGW}_{p}(\mu,\nu;\sigma_\kappa)\leq \text{IWRASGW}_{p}(\mu,\nu;\sigma_\kappa, L)$, we introduce the following lemma which had been proved in~\citet{nguyen2023energy}. Here, we provide the proof for completeness.

\begin{lemma}
    \label{lemma:inequality} For any $L\geq 1$, $0\leq a_{1} \leq a_{2} \leq \ldots \leq a_{L}$ and $0< b_{1} \leq b_{2} \leq \ldots \leq b_{L}$, we have:
    \begin{align}
        \frac{1}{L} (\sum_{i = 1}^{L} a_{i}) (\sum_{i = 1}^{L} b_{i}) \leq \sum_{i = 1}^{L} a_{i} b_{i}. \label{eq:key_inequality_2}
    \end{align}
\end{lemma}
\begin{proof}
    We prove Lemma~\ref{lemma:inequality} via an induction argument. For $L=1$, it is clear that $a_ib_i = a_ib_i$. Now, we assume that the inequality holds for $L$ i.e., $\frac{1}{L} (\sum_{i = 1}^{L} a_{i}) (\sum_{i = 1}^{L} b_{i}) \leq \sum_{i = 1}^{L} a_{i} b_{i}$ or  $
    (\sum_{i = 1}^{L} a_{i}) (\sum_{i = 1}^{L} b_{i}) \leq L \sum_{i = 1}^{L} a_{i} b_{i}.$
Now, we want to show that the inequality holds for $L+1$ i.e., $ (\sum_{i = 1}^{L+1} a_{i}) (\sum_{i = 1}^{L} b_{i}) \leq (L+1) \sum_{i = 1}^{L+1} a_{i} b_{i}.$ First, we have:
\begin{align*}
    (\sum_{i = 1}^{L + 1} a_{i}) (\sum_{i = 1}^{L + 1} b_{i}) &= (\sum_{i = 1}^{L } a_{i}) (\sum_{i = 1}^{L} b_{i})+ (\sum_{i = 1}^{L} a_{i}) b_{L + 1} + (\sum_{i = 1}^{L} b_{i}) a_{L + 1} + a_{L + 1} b_{L + 1} \\
    &\leq L \sum_{i = 1}^{L} a_{i} b_{i} + (\sum_{i = 1}^{L} a_{i}) b_{L + 1} + (\sum_{i = 1}^{L} b_{i}) a_{L + 1} + a_{L + 1} b_{L + 1}.
\end{align*}
By rearrangement inequality, we have $a_{L + 1} b_{L + 1} + a_{i} b_{i} \geq a_{L + 1} b_{i} + b_{L + 1} a_{i}$ for all $1 \leq i \leq L$. By taking the sum of these inequalities over $i$ from $1$ to $L$, we obtain:
\begin{align*}
    (\sum_{i = 1}^{L} a_{i}) b_{L + 1} + (\sum_{i = 1}^{L} b_{i}) a_{L + 1} \leq \sum_{i = 1}^{L} a_{i} b_{i} + L a_{L + 1} b_{L + 1}. 
\end{align*}
Therefore, we have
\begin{align*}
    (\sum_{i = 1}^{L + 1} a_{i}) (\sum_{i = 1}^{L + 1} b_{i}) &\leq L \sum_{i = 1}^{L} a_{i} b_{i} + (\sum_{i = 1}^{L} a_{i}) b_{L + 1} + (\sum_{i = 1}^{L} b_{i}) a_{L + 1} + a_{L + 1} b_{L + 1} \\
    &\leq L \sum_{i = 1}^{L} a_{i} b_{i} + \sum_{i = 1}^{L} a_{i} b_{i} + L a_{L + 1} b_{L + 1} + a_{L + 1} b_{L + 1}  \\
    &=(L + 1) (\sum_{i = 1}^{L + 1} a_{i} b_{i}),
\end{align*}
which completes the proof.
\end{proof}
From Lemma~\ref{lemma:inequality}, with $a_i = \text{GW}_p^p(\theta_i\sharp \mu,\theta_j \sharp \nu)$ and $b_i = f(\text{GW}_p^p(\theta_i\sharp \mu,\theta_j \sharp \nu))$, we have:
\begin{align*}
    \frac{1}{L} \sum_{i=1}^l \text{GW}_p^p(\theta_i\sharp \mu,\theta_i \sharp \nu) \leq \sum_{i=1}^L \text{GW}_p^p(\theta_i\sharp \mu,\theta_i \sharp \nu) \frac{f(\text{GW}_p^p(\theta_i\sharp \mu,\theta_i \sharp \nu)) }{\sum_{j=1}^L f(\text{GW}_p^p(\theta_j\sharp \mu,\theta_j \sharp \nu))}.
\end{align*}
Taking the expectation with respect to $\theta_1,\ldots,\theta_L \overset{i.i.d}{\sim} \sigma_{\textrm{RA}}(\theta;\mu,\nu,\sigma_\kappa)$, we obtain $\text{RASGW}_{p}(\mu,\nu;\sigma_\kappa)\leq \text{IWRASGW}_{p}(\mu,\nu;\sigma_\kappa, L)$.

Now to show that $\text{IWRASGW}_{p}(\mu,\nu;\sigma_\kappa, L)\leq \text{Max-SGW}(\mu,\nu)$, we have  $\theta^\star =\text{argmax}_{\theta \in \mathbb{S}^{d-1}} \text{GW}_p(\theta \sharp \mu,\theta \sharp \nu)$ since  $\mathbb{S}^{d-1}$ is compact and the function $\theta \to \text{GW}_p(\theta \sharp \mu,\theta \sharp \nu)$ is continuous. From the definition of the IWRASGW, for any $L\geq 1, \sigma_\kappa \in \mathcal{P}(\mathbb{S}^{d-1})$ we have:
\begin{align*}
    &\text{IWRASGW}_{p}(\mu,\nu;\sigma_\kappa, L) \\ =&  \left(\mathbb{E}_{\theta_1,\ldots, \theta_L\sim \sigma_{\textrm{RA}}(\theta;\mu,\nu,\sigma_\kappa)}  \left[\sum_{l=1}^L \text{GW}_p^p(\theta_l \sharp \mu, \theta_l \sharp \nu) \frac{f(\text{GW}_p^p(\theta_l \sharp \mu, \theta_l \sharp \nu))}{\sum_{j=1}^L f(\text{GW}_p^p(\theta_j \sharp \mu, \theta_j \sharp \nu))} \right]\right)^{\frac{1}{p}} \\
    \leq& \left(\mathbb{E}_{\theta \sim \sigma_{\textrm{RA}}(\theta;\mu,\nu,\sigma_k)} \left[ \text{GW}_p^p \left(\theta^\star \sharp \mu, \theta^\star \sharp \nu \right)\right]\right)^{\frac{1}{p}} = \text{GW}_p^p \left(\theta^\star \sharp \mu, \theta^\star \sharp \nu \right) = \text{Max-SGW}_p(\mu,\nu).
\end{align*}

\textit{(ii)}  We first recall the density of the von Mises-Fisher distribution and the Power spherical distribution. In particular, we have $vMF(\theta;\epsilon,\kappa):= \frac{\kappa^{d/2-1}}{(2\pi)^{d/2} I_{d/2-1}(\kappa)} \exp(\kappa\epsilon^\top \theta)$ with $I_s$  denotes the modified Bessel function of the first kind, and the PS distribution $PS(\theta;\epsilon,\kappa)  =  \left( 2^{d+\kappa -1} \pi^{(d-1)/2} \frac{\Gamma( (d-1)/2+\kappa )}{\Gamma (d+\kappa -1)} \right)^{-1}(1+\epsilon^\top \theta)^\kappa $. We have:
\begin{align*}
    &\lim_{\kappa \to 0}\frac{\kappa^{d/2-1}}{(2\pi)^{d/2} I_{d/2-1}(\kappa)} \exp(\kappa\epsilon^\top \theta) \to C_1,\\
    &\lim_{\kappa \to 0}\left( 2^{d+\kappa -1} \pi^{(d-1)/2} \frac{\Gamma( (d-1)/2+\kappa )}{\Gamma (d+\kappa -1)} \right)^{-1}(1+\epsilon^\top \theta)^\kappa\to C_2
\end{align*}
for some constant $C_1$ and $C_2$ which do not depend on $\theta$, hence, the vMF distribution and the PS distribution converge to the uniform distribution when $\kappa \to 0$. Therefore, we have $vMF(\theta;\epsilon,\kappa) \to \mathcal{U}(\mathbb{S}^{d-1})$ and $PS(\theta;\epsilon,\kappa) \to \mathcal{U}(\mathbb{S}^{d-1})$.
Therefore, we have $\sigma_{\textrm{RA}}(\theta;\mu,\nu,\sigma_\kappa) \to \mathcal{U}(\mathbb{S}^{d-1})$.

Now, we need to show that $\text{GW}_p^p(\theta\sharp \mu,\theta\sharp \nu)$ is bounded and continuous with respective to $\theta$. For boundedness, it is sufficient to note that the GW distance is a bounded distance.

For the continuity, let $(\theta_t)_{t\geq 1}$ be a sequence on $\mathbb S^{d-1}$ which converges to $\theta\in\mathbb S^{d-1}$ i.e., $\|\theta_t - \theta \| \to 0$ as $t \to \infty$, and a arbitrary measure $\mu \in \mathcal{P}_p(\mathbb{R}^d)$. Then we have:
\begin{align*}
    \text{GW}_p (\theta \sharp \mu,\theta_t \sharp \mu) & = \left(\inf_{\pi\in\Pi(\mu,\mu)}\int_{\mathbb R^{d}} \vert | \theta_t^\top x - \theta_t^\top x' |- | \theta^\top y - \theta^\top y' |\vert^p d\pi(x,y)d\pi(x',y')\right)^{1/p}\\
     & \leq \left(\int_{\mathbb R^{d}} \vert | \theta_t^\top x - \theta_t^\top x' |- | \theta^\top x - \theta^\top x' |\vert^p d\pi(x,x)d\pi(x',x')\right)^{1/p} \\
     & \leq \left(\int_{\mathbb R^d} \| x - x'\|^p \mu(dx)\mu(dx')\right)^{1/p} \| \theta_t - \theta\| \to 0 \quad \text{as } t\to\infty,
\end{align*}
where $\left(\int_{\mathbb R^d} \| x \|^p \mu(dx)\right)^{1/p} < \infty$ since $\mu \in \mathcal{P}_p(\mathbb{R}^d)$, and the second inequality is due to the Cauchy-Schwartz inequality. 

Using the triangle inequality, we have:
\begin{align*}
    &\left|\text{GW}_p (\theta_t \sharp \mu,\theta_t \sharp \nu) - \text{GW}_p (\theta \sharp \mu,\theta \sharp \nu)\right| \\ \leq& \left|\text{GW}_p (\theta_t \sharp \mu,\theta_t \sharp \nu) - \text{GW}_p (\theta \sharp \mu,\theta_t \sharp \nu)\right| + \left|\text{GW}_p (\theta \sharp \mu,\theta_t \sharp \nu) - \text{GW}_p (\theta \sharp \mu,\theta \sharp \nu) \right| \\
    \leq& \:\text{GW}_p (\theta \sharp \mu,\theta_t \sharp \mu) + \text{GW}_p (\theta \sharp \nu,\theta_t \sharp \nu) \to 0 \: \text{as}\: t\to \infty,
\end{align*}
hence, $\text{GW}_p (\theta_t \sharp \mu,\theta_t \sharp \nu) \to \text{GW}_p (\theta \sharp \mu,\theta \sharp \nu)$ as $t\to\infty$, which complete the proof of continuity.

From the boundedness, continuity, and  $\sigma_{\textrm{RA}}(\theta;\mu,\nu,\sigma_\kappa) \to \mathcal{U}(\mathbb{S}^{d-1})$, we have
\begin{align*}
    \text{RASGW}_{p}^p(\mu,\nu;\sigma_\kappa) & = \mathbb{E}_{\theta \sim \sigma_{\textrm{RA}}(\theta;\mu,\nu,\sigma_\kappa)}[\text{GW}_p(\theta \sharp \mu, \theta \sharp \nu)] \\ &\to \mathbb{E}_{\theta \sim \mathcal{U}(\mathbb{S}^{d-1})}[\text{GW}_p(\theta \sharp \mu, \theta \sharp \nu)] = \textrm{SGW}_{p}^p(\mu,\nu).
\end{align*}
Applying the continuous mapping theorem for $x\to x^{1/p}$, we obtain $\lim_{\kappa \to 0} \text{RASGW}_{p}(\mu,\nu;\sigma_\kappa) \to \textrm{SGW}_p(\mu,\nu)$.

\textit{(iii)} Since we have proved that $\text{GW}_p^p(\theta\sharp \mu,\theta\sharp \nu)$ is bounded and continuous with respective to $\theta$, we can show that  $\sum_{i=1}^L\textrm{GW}_p^p(\theta_i\sharp \mu,\theta_i \sharp \nu) \frac{f(\text{GW}_p^p(\theta_i\sharp \mu,\theta_i \sharp \nu)) }{\sum_{j=1}^L f(\text{GW}_p^p(\theta_j\sharp \mu,\theta_j \sharp \nu))}$ are bounded and continuous with respect to $\theta_1,\ldots,\theta_L$. As $L \to \infty$, we have 
\begin{align*}
    \sum_{i=1}^L\textrm{GW}_p^p(\theta_i\sharp \mu,\theta_i \sharp \nu) \frac{f(\text{GW}_p^p(\theta_i\sharp \mu,\theta_i \sharp \nu)) }{\sum_{j=1}^L f(\text{GW}_p^p(\theta_j\sharp \mu,\theta_j \sharp \nu))} \to \mathbb{E}_{\gamma \sim \sigma_{\mu,\nu,f}(\gamma)}[\gamma \sharp \mu,\gamma \sharp \nu] = \text{EBSGW}_p^p(\mu,\nu;f).
\end{align*}
Applying the continuous mapping theorem for $x\to x^{1/p}$, we obtain $\lim_{L \to \infty} \text{IWRASGW}_{p}(\mu,\nu;\sigma_\kappa,L,f) \to \text{EBSGW}_p(\mu,\nu;f)$. 

\subsection{Proof of Proposition~\ref{prop:sample_complexity}}
\label{subsec:proof:prop:sample_complexity}
The proof of this result follows from the proof of Proposition 3 in~\citet{nguyen2023energy}. We assume that $\mu$ has a compact support $\mathcal{X} \subset \mathbb{R}^d$. 

From Proposition~\ref{prop:connection}, we have
\begin{align*}
    \mathbb{E} [\text{RASGW}_{p} (\mu_{n},\mu;\sigma_\kappa)] &\leq  \mathbb{E} [\text{IWRASGW}_{p} (\mu_{n},\mu;\sigma_\kappa,L)] \\ &\leq \mathbb{E} \left[\text{Max-SGW}_p (\mu_{n},\mu) \right] \leq c\mathbb{E} \left[\text{Max-SW}_p (\mu_{n},\mu) \right],
\end{align*}
for any $\sigma_\kappa \in \mathcal{P}(\mathbb{S}^{d-1})$ 
 and $L\geq 1$. Further, the last inequality follows from \citet{zhang2024gromov}.
 
 Therefore, the proposition follows as long as we can demonstrate that $$\mathbb{E} [\text{Max-SW}_p (\mu_{n},\mu)] \leq C'\sqrt{(d+1) \log n/n}$$ where $\mu_n = \frac{1}{n}\sum_{i=1}^n \delta_{X_i}$ with $X_1,\ldots,X_n \overset{i.i.d}{\sim} \mu$, and $C' > 0$ is some universal constant which satisfies $C = cC'$ and the outer expectation is taken with respect to $X_1,\ldots,X_n$. 

Using the closed-form of one-dimensional Wasserstein distance, we have:
\begin{align*}
    \text{Max-SW}_p (\mu_{n},\mu) & = \max_{\theta \in \mathbb{S}^{d-1}} \int_{0}^{1} |F_{n, \theta}^{-1}(z) - F_{\theta}^{-1}(z)|^{p} d z,
\end{align*}
where $F_{n,\theta}$ and $F_{\theta}$ as the cumulative distributions of $\theta \sharp \mu_{n}$ and $\theta\sharp \mu$. Since $\text{GW}_p( (t \theta) \sharp \mu,(t \theta) \sharp \nu) = t\text{GW}_p( \theta \sharp \mu, \theta \sharp \nu)$ for $t > 0$. We can rewrite Max-SW as:
\begin{align*}
    \text{Max-SW}_p^{p} (\mu_{n},\mu) &=
      \max_{\theta \in \mathbb{R}^{d}: \|\theta\| = 1} \int_{0}^{1} |F_{n, \theta}^{-1}(z) - F_{\theta}^{-1}(z)|^{p} d z \\
    & \leq \text{diam}(\mathcal{X}) \max_{x \in \mathbb{R}, \theta \in \mathbb{R}^{d}: \|\theta\| \leq 1} |F_{n, \theta}(x) - F_{\theta}(x)|^{p}\\
    &=\text{diam}(\mathcal{X}) \sup_{A \in \mathcal{A}} |\mu_{n}(A) - \mu(A)|,
\end{align*}
where $\mathcal{A}$ is the set of half-spaces $\{z \in \mathbb{R}^{d}: \theta^{\top} z \leq x\}$ for all $\theta \in \mathbb{R}^{d}$ such that $\|\theta\| \leq 1$. 
From VC inequality (Theorem 12.5 in~\citet{devroye2013probabilistic}), we have
$$
    \mathbb{P}\left(\sup_{A \in \mathcal{A}} |\mu_{n}(B) - \mu(A)| >t \right) \leq 8 S(\mathcal{A},n) e^{-nt^2 /32}.
$$ with $S(\mathcal{A},n)$ is the growth function. From the Sauer Lemma (Proposition 4.18 in \citet{wainwrighthigh}), the  growth function is upper bounded by $(n+1)^{VC(\mathcal{A})}$. Moreover, we can get $VC(\mathcal{A})= d+1$ from Example 4.21 in~\citet{wainwrighthigh}.

\vspace{0.5em}
\noindent
Let $8 S(\mathcal{A},n) e^{-nt^2 /32} \leq \delta$, we have $t^2 \geq \frac{32}{n} \log \left( \frac{8S(\mathcal{A},n)}{\delta}\right)$. Therefore, we obtain
\begin{align*}
    \mathbb{P}\left(\sup_{A \in \mathcal{B}} |\mu_{n}(A) - \mu(A)| \leq \sqrt{\frac{32}{n} \log \left( \frac{8S(\mathcal{A},n)}{\delta}\right)} \right) \geq 1-\delta,
\end{align*}
Using the Jensen inequality and the tail sum expectation for non-negative random variable, we have:
\begin{align*}
    &\mathbb{E}\left[\sup_{A \in \mathcal{A}} |\mu_{n}(A) - \mu(A)|\right] \\&\leq \sqrt{\mathbb{E}\left[\sup_{A \in \mathcal{A}} |\mu_{n}(A) - \mu(A)|\right]^2} = \sqrt{\int_{0}^\infty \mathbb{P}\left(\left(\sup_{A \in \mathcal{A}} |\mu_{n}(A) - \mu(A)| \right)^2>t \right)dt} \\
    &=\sqrt{\int_{0}^u \mathbb{P}\left(\left(\sup_{A \in \mathcal{A}} |\mu_{n}(A) - \mu(A)| \right)^2>t \right)dt + \int_{u}^\infty \mathbb{P}\left(\left(\sup_{A \in \mathcal{A}} |\mu_{n}(A) - \mu(A)| \right)^2>t \right)dt} \\
    &\leq \sqrt{\int_{0}^u 1dt + \int_{u}^\infty8 S(\mathcal{A},n) e^{-nt /32} dt} = \sqrt{u + 256 S(\mathcal{A},n)  \frac{e^{-nu/32}}{n}}.
\end{align*}
Since the inequality holds for any $u$, we search for the best $u$ that makes the inequality tight. Let $f(u) = u + 256 S(\mathcal{A},n)  \frac{e^{-nu/32}}{n}$, we have $f'(u) = 1+ 8S(\mathcal{A},n) e^{-nu/32}$. Setting $f'(u)=0$, we obtain the minima $u^\star = \frac{32 \log (8S(\mathcal{A},n))}{n}$. Plugging $u^\star$ in the inequality, we obtain:
\begin{align*}
    \mathbb{E}\left[\sup_{A \in \mathcal{A}} |\mu_{n}(A) - \mu(A)|\right] & \leq \sqrt{\frac{32 \log (8S(\mathcal{A},n))}{n} + 32 }\leq C' \sqrt{\frac{(d+1)\log (n+1)}{n}},
\end{align*}
by using Sauer Lemma i.e., $S(\mathcal{A},n) \leq (n+1)^{VC(\mathcal{A})} \leq (n+1)^{d+1}$. Putting the above results together leads to
\begin{align*}
    \mathbb{E} [\text{Max-SW}_p (\mu_{n},\mu)] \leq C'\sqrt{(d+1) \log n/n},
\end{align*}
where $C' > 0$ is some universal constant. As a consequence, we obtain the conclusion of the proof.

\subsection{Proof of Proposition~\ref{proposition:MCerror}}
\label{subsec:proof:proposition:MCerror}

For any $p\geq 1$, $d \geq 1$, $\sigma_\kappa \in \mathcal{P}(\mathbb{S}^{d-1})$, and $\mu,\nu \in \mathcal{P}_p(\mathbb{R}^d)$, using Jensen's inequality, we have:
\begin{align*}
    &| \mathbb{E}[\widehat{\text{RASGW}_{p}^p}(\mu,\nu;\sigma_\kappa)] - \text{RASGW}_{p}^p (\mu,\nu;\sigma_\kappa)| \\
    &\leq  \mathbb{E} | \widehat{\text{RASGW}_{p}^p}(\mu,\nu;\sigma_\kappa) - \text{RASGW}_{p}^p (\mu,\nu;\sigma_\kappa)| \\
    &\leq \left(\mathbb{E} | \widehat{\text{RASGW}_{p}^p}(\mu,\nu;\sigma_\kappa ) - \text{RASGW}_{p}^p (\mu,\nu;\sigma_\kappa)|^2 \right)^{\frac{1}{2}} \\
    &= \left(\mathbb{E} \left( \frac{1}{M}\sum_{l=1}^M  \text{GW}_p^p (\theta_{l} \sharp \mu,\theta_{l} \sharp \nu)  - \mathbb{E}_{\theta\sim \sigma_{\textrm{RA}}(\theta;\mu,\nu,\sigma_\kappa)}\left[  \text{GW}_p^p \left(\theta \sharp \mu, \theta \sharp \nu \right)\right]\right)^2 \right)^{\frac{1}{2}}.
\end{align*}

Since 
\begin{align*}
    \mathbb{E}[\frac{1}{M}\sum_{l=1}^M  \text{GW}_p^p (\theta_{l} \sharp \mu,\theta_{l} \sharp \nu)] &= \frac{1}{M}\sum_{l=1}^M \mathbb{E}[\text{GW}_p^p (\theta_{l} \sharp \mu,\theta_{l} \sharp \nu)] \\ &= \frac{1}{M}\sum_{l=1}^M \widehat{\text{RASGW}_{p}^p}(\mu,\nu;\sigma_\kappa)  = \widehat{\text{RASGW}_{p}^p}(\mu,\nu;\sigma_\kappa),
\end{align*}
we have
\begin{align*}
    \mathbb{E} | \widehat{\text{RASGW}}_{p}^p(\mu,\nu;\sigma_\kappa) - \text{RASGW}_{p}^p (\mu,\nu;\sigma_\kappa)|&\leq  \left( Var_{\theta \sim \sigma_{\textrm{RA}}(\theta;\mu,\nu,\sigma_\kappa )}\left[ \frac{1}{M}  \sum_{l=1}^M  \text{GW}_p^p \left(\theta_l \sharp \mu, \theta_l \sharp \nu \right)\right]\right)^{\frac{1}{2}} \\
    &= \frac{1}{\sqrt{M}} Var\left[ \text{GW}_p^p \left(\theta \sharp \mu, \theta \sharp \nu \right)\right]^{\frac{1}{2}}.
\end{align*}
Now, since $\theta_1,\ldots,\theta_M \overset{i.i.d}{\sim} \sigma_{\textrm{RA}}(\theta;\mu,\nu,\sigma_\kappa )$, this completes the proof.
\section{Additional Baselines}
\label{sec:additional_baselines}
In this section, we present two baselines constructed based on similar concepts in the Sliced-Wasserstein literature. \iffalse, along with the approximately linear algorithm by \citet{scetbon2022linear}\fi  We adapt them to the Sliced Gromov-Wasserstein setup to provide a clearer exposition of the advantages of our distance. These baselines act in addition to SGW and Max-SGW, baselines that already exist. The first baseline is based on a similar premise of finding a good slicing distribution, but uses a computationally expensive optimization procedure to find the distribution. On the other hand, the second baseline is optimization-free, but its main computational bottleneck is sampling from the proposed distribution. 

\textbf{Distributional Sliced Gromov-Wasserstein.} To choose an improved slicing distribution, \citet{nguyen2021distributional} introduced the distributional sliced Wasserstein (DSW) distance. Adapting this to our context, we define the distributional sliced Gromov-Wasserstein (DSGW) distance, which is defined between two probability measures $\mu \in \mathcal{P}_p(\mathbb{R}^d)$ and $\nu \in \mathcal{P}_p(\mathbb{R}^d)$ as:
\begin{align}
\label{eq:DSGW}
    \text{DSGW}_p^p(\mu,\nu) = \max_{\psi \in \Psi} \mathbb{E}_{\theta \sim \sigma_\psi(\theta)} [\text{GW}_p^p(\theta \sharp \mu, \theta \sharp \nu)],
\end{align}
where $\sigma_\psi(\theta) \in \mathcal{P}(\mathbb{S}^{d-1})$, for example, an implicit distribution~\citep{nguyen2021distributional}, the von Mises-Fisher~\citep{jupp1979maximum} (vMF) distribution, and the Power Spherical (PS)~\citep{cao2018partial} distribution with unknown location parameter $\sigma_\psi(\theta):=(\text{PS}) \text{vMF}(\theta|\epsilon,\kappa)$, with $\psi = \epsilon$. By performing $T \geq 1$ iterations of (projected) stochastic (sub)-gradient ascent to estimate the parameter $\hat{\psi}_T$, Monte Carlo samples $\theta_1, \ldots, \theta_L \overset{i.i.d}{\sim} \sigma_{\hat{\psi}_T}(\theta)$ are used to approximate DSGW. The time and space complexities of DSGW are at best $\mathcal{O}(LTn \log n + LTdn)$ and $\mathcal{O}(Ld + Ln)$, respectively, excluding the complexities of sampling from $\sigma_{\hat{\psi}_T}(\theta)$. As the concentration parameter $\kappa \to \infty$, both the vMF and PS distributions converge to the Dirac distribution, giving the max sliced Gromov-Wasserstein distance~\citep{deshpande2019max}.
The Max-SGW involves performing $T \geq 1$ iterations of (projected) (sub)-gradient ascent to find the approximate ``max" projection direction $\hat{\theta}_T$. The estimated value of Max-SGW is then set as $\text{GW}_p(\hat{\theta}_T \sharp \mu, \hat{\theta}_T \sharp \nu)$. Notably, the optimization problem is non-convex~\citep{nietert2022statistical}, implying that the global optimum $\theta^\star$ cannot be achieved. Therefore, Max-SGW approximations do not constitute a metric, even as $T \to \infty$. The time and space complexities of Max-SGW are again at best $\mathcal{O}(Tn \log n + Tdn)$ and $\mathcal{O}(d + n)$, respectively.

\textbf{Energy-driven sliced Gromov-Wasserstein.} In practical applications, the computational effort of optimization often surpasses that of sampling from a fixed slicing distribution. To circumvent costly optimization processes, recent work~\citep{nguyen2023energy} introduces a novel approach that leverages an energy-driven slicing distribution, bypassing optimization altogether. While this was originally developed within the sliced Wasserstein framework, we have adapted it for our context. The energy-driven sliced Gromov-Wasserstein (EBSGW) distance between two probability measures $\mu \in \mathcal{P}_p(\mathbb{R}^d)$ and $\nu\in \mathcal{P}_p(\mathbb{R}^d)$ is defined as follows:
    \begin{align}
    \label{eq:ebsgw}
        &\text{EBSGW}_p^p(\mu,\nu;f) = \mathbb{E}_{\theta \sim \sigma_{\mu,\nu}(\theta;f,p)}\left[ \text{GW}_p^p (\theta\sharp \mu,\theta \sharp \nu)\right],
    \end{align}
where $f:[0,\infty) \to (0,\infty)$ is an increasing energy function such as $f(x)=e^x$, and $\sigma_{\mu,\nu}(\theta;f,p) \propto f(\text{GW}^p_p(\theta \sharp \mu,\theta \sharp \nu))$. EBSGW can be estimated using importance sampling with the uniform proposal distribution $\sigma_0 = \mathcal{U}(\mathbb{S}^{d-1})$. For $\theta_1,\ldots,\theta_L \overset{i.i.d}{\sim} \sigma_0(\theta)$, the estimate is given by: %$\widehat{\text{EBSGW}}_p^p(\mu,\nu;f,L)= $
\begin{align}
\label{eq:emp_ISEBSGW}
   \widehat{\text{EBSGW}_p^p}(\mu,\nu;f,L)= \sum_{l=1}^L   \text{GW}_p^p (\theta_l\sharp \mu,\theta_l \sharp \nu)\hat{w}_{\mu,\nu,\sigma_0,f,p} (\theta_l) ,
\end{align}
where $\hat{w}_{\mu,\nu,\sigma_0,f,p} (\theta_l) = \frac{w_{l}(f;\mu,\nu) / \sigma_0(\theta)}{\sum_{l'=1}^L w_{l'}(f;\mu,\nu) / \sigma_0(\theta)}$ are the normalized importance weights, and $\frac{w_{l}(f;\mu,\nu)}{\sigma_0(\theta)}$ represents the importance weight function. The computational complexity in time and space for calculating EBSGW is at best $\mathcal{O}(Ln\log n+Ldn)$ and $\mathcal{O}(Ld+Ln)$, respectively. For a detailed discussion regarding the complexity and its relation to optimality, see Section \ref{sec:algorithm}. However, this estimation introduces bias, meaning $\mathbb{E}[\widehat{\text{EBSGW}_p^p}(\mu,\nu;f,L)]\neq \text{EBSGW}_p^p(\mu,\nu;f)$. Beyond importance sampling, Markov Chain Monte Carlo (MCMC) methods can also approximate EBSGW, although they carry a significant computational burden~\citep{nguyen2023energy}.

\textbf{Random Path Sliced Gromov-Wasserstein.} In \citet{nguyen2024slicedwassersteinrandompathprojecting}, the authors proposed a variant of Sliced-Wasserstein known as Random Path Sliced Gromov-Wasserstein (RPSW). It introduces the random-path projecting direction (RPD), constructed as a normalized difference between two random vectors sampled from the input measures, with a random perturbation to ensure continuity. This creates a new slicing distribution, the random-path slicing distribution (RPSD), which is optimization-free and efficient to sample. It is not immediately clear how to adapt RPSW to the SGW setting, as RPSW constructs its projecting directions based on just two points (one taken from each distribution), whereas in GW, we sample a pair of points belonging to each distribution. We can extend it to the GW setting by taking the quartet $(X,X',Y,Y')$ followed by constructing a direction by taking a cross pair. However, this is quite arbitrary, and it is not exactly clear how this serves any purpose in our alignment problem. Experimental studies show that our metric outperforms this by quite a margin.

In the context of competing methods, we mention the approximately linear algorithm to estimate GW (between measures in their ambient spaces) proposed by \citet{scetbon2022linear}. Given that it employs low rank approximations to replace complexity-heavy computations underlying GW, e.g., Sinkhorn scaling, it is inherently hostile to gradient computations (non-smooth, not differentiable), and hence unsuitable to be used as an objective function in generative models (e.g., GWGAN). In the absence of a dedicated nested optimization scheme, it also becomes numerically unstable (similar to the vanilla GW \citep{bunne2019learning})during training when used as a loss.

\iffalse
In the next section, we propose a new distance that is optimization-free and also avoids an expensive sampling procedure. It improves upon SGW by using a better slicing distribution. Further, in the experimental section after that, we show the advantages of our distance being both optimization-free and also not having an expensive sampling procedure.
\fi

\section{More Preliminaries}
\label{more_preliminaries}
We briefly revisit the concept of the von Mises-Fisher distribution.\begin{definition} \label{def:vMF} The von Mises–Fisher distribution (vMF) represents a probability distribution that spans the unit sphere $\mathbb{S}^{d-1}$, characterized by the following density function \citep{jupp1979maximum}: \begin{align} f(x| \epsilon, \kappa ) : = C_d (\kappa) \exp(\kappa \epsilon^\top x), \end{align} where $\kappa \geq 0$ serves as the concentration parameter, $\epsilon\in \mathbb{S}^{d-1}$ is designated as the location vector, and $C_d(\kappa) : = \frac{\kappa^{d/2 -1}}{(2 \pi)^{d/2} I_{d/2 -1 }(\kappa) }$ functions as the normalization constant. The term $I_v$ refers to the modified Bessel function of the first kind with order $v$ \citep{temme2011special}. \end{definition} The vMF distribution centers around the mode $\epsilon$, and its density diminishes as $x$ moves further from $\epsilon$. As $\kappa \to 0$, vMF transitions to a uniform distribution, while for $\kappa \to \infty$, it converges to a Dirac distribution located at $\epsilon$~\citep{Suvrit_directional}. Details of the sampling procedure with a mixture of vMF distributions are placed in Algorithm~\ref{Alg:MovMF_sampling}.

\begin{algorithm}[t]
  \caption{Sampling from a mixture of vMF distributions}
  \label{Alg:MovMF_sampling}
\begin{algorithmic}[1]
  \REQUIRE The number of vMF components $k$, location $\{\epsilon_i\}_{i=1}^k$, concentration $\{\kappa_i\}_{i=1}^k$, mixture weights $\{\alpha_i\}_{i=1}^k$ , dimension $d$, unit vector $e_1= (1,0,..,0)$.
  \STATE Sample index $i \sim \text{Categorical}(\alpha_1,...,\alpha_k)$
  \STATE Sample $v \sim \mathcal{U}(\mathbb{S}^{d-2})$ 
%   \STATE Sample $\omega \sim g(\omega \mid \kappa, d) \propto \exp (\omega \kappa)\left(1-\omega^{2}\right)^{\frac{1}{2}(d-3)}$ \\
%  \{{acceptance-rejection sampling} \}
\STATE $b \leftarrow \frac{-2 \kappa_i+\sqrt{4 \kappa_i^{2}+(d-1)^{2}}}{d-1}$, $a \leftarrow \frac{(d-1)+2 \kappa_i+\sqrt{4 \kappa_i^{2}+(d-1)^{2}}}{4}$, $m \leftarrow \frac{4 a b}{(1+b)}-(d-1) \ln (d-1)$
  \REPEAT 
    \STATE Sample $\psi \sim \operatorname{Beta}\left(\frac{1}{2}(d-1), \frac{1}{2}(d-1)\right)$
    \STATE $\omega \leftarrow h(\psi, \kappa_i)=\frac{1-(1+b) \psi}{1-(1-b) \psi}$
    \STATE $t \leftarrow \frac{2 a b}{1-(1-b) \psi}$
    \STATE Sample $u \sim \mathcal{U}(0,1)$
  \UNTIL{{$(d-1) \ln (t)-t+m \geq \ln (u)$}}
 \STATE $h_1 \leftarrow (\omega, \sqrt{1-\omega^2} v^\top)^\top$
%  \STATE $U \leftarrow Householder(e_1,\epsilon)$  \{Householder transform\}
\STATE $\epsilon^\prime \leftarrow e_1 - \epsilon_i$
\STATE $u = \frac{\epsilon^\prime}{\norm{\epsilon^\prime}}$
\STATE $U = \mathbb{I} - 2uu^\top$
  \STATE {\bfseries return} $Uh_1$
\end{algorithmic}
\end{algorithm}

We proceed to introduce the power spherical distribution~\citep{de2020power}:\begin{definition}
\label{def:power_spherical}
The power spherical distribution (PS) describes a probability distribution on the unit sphere $\mathbb{S}^{d-1}$ with the following density function:
\begin{align}
    f(x| \epsilon, \kappa ) = C( \kappa,d) (1 + \epsilon^\top x)^{\kappa},
\end{align}
here, $\kappa \geq 0$ serves as the concentration parameter, $\epsilon\in \mathbb{S}^{d-1}$ is the location vector, and the normalization constant $C_d(\kappa)$ is defined as $ \left\{ 2^{d - 1 + \kappa} \pi^{\frac{d - 1}{2}} \frac{\Gamma(\frac{d - 1}{2} + \kappa)}{\Gamma(d - 1 + \kappa)}\right\}^{-1}$. 
\end{definition}When $\kappa \to 0$, the power spherical distribution tends to the uniform distribution because, for $\kappa \in (0, 1]$, $f(x|\epsilon,\kappa)$ remains uniformly bounded. Moreover, for all $x$ not equal to $- \epsilon$
\begin{align*}
\lim_{\kappa \rightarrow 0} f(x|\epsilon,\kappa) = C(0,d),
\end{align*}
indicating it matches the uniform distribution's density on $\mathbb{S}^{d-1}$. Thus, by invoking the Lebesgue dominated convergence theorem, for any bounded function $g$ on $\mathbb{S}^{d-1}$, it holds that
\begin{align*}
    \lim_{\kappa \rightarrow 0}\int_{\mathbb{S}^{d-1}} g(x) f(x|\epsilon,\kappa) dx = \int_{\mathbb{S}^{d-1}} g(x) \mathcal{U}(\mathbb{S}^{d-1})  dx, 
\end{align*}
where $\mathcal{U}(\mathbb{S}^{d-1})$ denotes the uniform distribution on $\mathbb{S}^{d-1}$. This indicates that the density $f$ of the power spherical distribution converges to that of the uniform distribution as $\kappa \to 0$. Conversely, as $\kappa \rightarrow \infty$, the function $f$ at its mode $\epsilon$ increases unboundedly, and the following conclusion holds for any uniformly bounded function $g$ on $\mathbb{S}^{d-1}$.
\begin{align*}
   \lim_{\kappa \rightarrow \infty} \int_{\mathbb{S}^{d-1}} g(x) f(x|\epsilon,\kappa) dx =  \int_{\mathbb{S}^{d-1}} g(x) \delta_{\epsilon}(x) dx.
   \end{align*}
  It is true since for $x = \epsilon$, and $\kappa \rightarrow \infty$ we have
  \begin{align*}
      f(\epsilon|\epsilon,\kappa) = 2^{\kappa} \frac{\Gamma(d-1 + \kappa)}{\Gamma(\frac{d-1}{2} + \kappa)} \pi^{\frac{1-d}{2}} 2^{1-d-\kappa} = \pi^{\frac{1-d}{2}} 2^{1-d} \frac{\Gamma(d-1 + \kappa)}{\Gamma(\frac{d-1}{2} + \kappa)} \rightarrow \infty.
      \end{align*}
Consequently, as $\kappa \to \infty$, the density of the power spherical distribution with a location vector $\epsilon$ approaches the Dirac delta measure at $\epsilon$. 

\textbf{Power Spherical Distribution Sampling Procedure.} We revisit the sampling technique for the power spherical (PS) distribution, as outlined in Algorithm \ref{Alg:PS_sampling} in \citet{de2020power}. A key distinction between sampling from the PS and the von Mises-Fisher (vMF) distribution is that PS does not require rejection sampling, unlike the vMF's approach as detailed in Algorithm \ref{Alg:MovMF_sampling}. Thus, PS sampling is more efficient than vMF sampling. Moreover, parameter gradient estimation for PS density is simpler compared to vMF since the entire PS sampling algorithm is differentiable, including sampling from the Beta distribution using the implicit reparameterization trick~\citep{figurnov2018implicit}.
\begin{algorithm}[t]
  \caption{Sampling from power spherical distribution}
  \label{Alg:PS_sampling}
\begin{algorithmic}[1]
  \REQUIRE location parameter $\epsilon$, concentration $\kappa$, dimension $d$, unit vector $e_1= (1,0,..,0)$.
  \STATE Sample $z \sim \text{Beta}(\frac{(d-1)}{2}+\kappa,\frac{(d-1)}{2})$
  \STATE Sample $v \sim \mathcal{U}(\mathbb{S}^{d-2})$ 
%   \STATE Sample $\omega \sim g(\omega \mid \kappa, d) \propto \exp (\omega \kappa)\left(1-\omega^{2}\right)^{\frac{1}{2}(d-3)}$ \\
%  \{{acceptance-rejection sampling} \}
\STATE  $w \leftarrow 2z - 1$
 \STATE $h_1\leftarrow (\omega, \sqrt{1-\omega^2} v^\top)^\top$
%  \STATE $U \leftarrow Householder(e_1,\epsilon)$  \{Householder transform\}
\STATE $\epsilon^\prime \leftarrow e_1 - \epsilon$
\STATE $u = \frac{\epsilon^\prime}{\norm{\epsilon^\prime}}$
\STATE $U = \mathbb{I} - 2uu^\top$
  \STATE {\bfseries return} $Uh_1$
\end{algorithmic}
\end{algorithm}
%\subsection{Mixtures of vMF for spherical sliced fused Gromov Wasserstein and its relational regularized autoencoder }
%\label{Sec:MSFG}

\textbf{Radon Transform.}
The classical Radon transform, symbolized by $\mathcal{R}$, converts a function $I \in L^1(\mathbb{R}^d)$— where $$L^1(\mathbb{R}^d) = \{ I:\mathbb{R}^d \rightarrow \mathbb{R}\ |\ \int_{\mathbb{R}^d} |I(x)|dx < \infty\}$$— into an infinite collection of its integrals over the hyperplanes of $\mathbb{R}^d$, as defined by
\begin{eqnarray}
    \mathcal{R} I(t,\theta) = \int_{\mathbb{R}^d} I(x)\delta(t-\langle x, \theta \rangle)dx,
\label{eq:radon}
\end{eqnarray}
for $(t, \theta) \in \mathbb{R} \times \mathbb{S}^{d-1}$. Here, $\mathbb{S}^{d-1} \subset \mathbb{R}^d$ represents the unit sphere in $d$ dimensions, $\delta(\cdot)$ is the Dirac delta function, and $\langle \cdot, \cdot \rangle$ denotes the Euclidean inner-product. The mapping $\mathcal{R}: L^1(\mathbb{R}^d)\rightarrow L^1(\mathbb{R}\times \mathbb{S}^{d-1})$ is established. A typical hyperplane is described as:
\begin{equation}
    H(t, \theta) = \{x\in \mathbb{R}^d \ |\ \langle x, \theta \rangle = t\},
    \label{eq:hyperplanes}
\end{equation}
which can be viewed as a level set for the function $g: \mathbb{R}^d\times\mathbb{S}^{d-1}\rightarrow \mathbb{R}$ defined by $g(x, \theta) = \langle x, \theta \rangle$. For any fixed $\theta$, the integrals over hyperplanes perpendicular to $\theta$ yield a continuous function $\mathcal{R}I(\cdot,\theta) : \mathbb{R} \rightarrow \mathbb{R}$, acting as a projection (or slice) of $I$. The Radon transform is a linear bijective mapping \citep{helgason2011radon}, with its inverse $\mathcal{R}^{-1}$ given by:
\begin{eqnarray}
    I(x) &=& \mathcal{R}^{-1}\big(\mathcal{R}I(t,\theta)\big) \nonumber \\
    &=& \int_{\mathbb{S}^{d-1}} (\mathcal{R}I(\langle x, \theta \rangle,\theta)*\eta(\langle x, \theta \rangle)) d\theta
\end{eqnarray}
where $\eta(\cdot)$ stands for a one-dimensional high-pass filter whose Fourier transform is $\mathcal{F}\eta(\omega) =  c|\omega|^{d-1}$, as justified by the Fourier slice theorem \citep{helgason2011radon} (see supplementary materials), and `$*$' signifies the convolution operation. This formulation of the inverse Radon transform is known as the filtered back-projection technique, extensively used in the reconstruction of images in the field of biomedical imaging. Conceptually, each one-dimensional projection (or slice) $\mathcal{R}I(\cdot,\theta)$ is first filtered using a high-pass filter and then back-projected into $\mathbb{R}^{d}$ along $H(\cdot,\theta)$ to reconstruct an approximation of $I$. The final image $I$ is rebuilt by combining all such back-projected approximations. In practical scenarios, obtaining an infinite number of projections is impossible, thus, the filtered back-projection method replaces integral calculations with a finite sum of projections, akin to a Monte Carlo estimation.

\paragraph{Sliced Wasserstein Distance.} The sliced $p$-Wasserstein distance approach computes a series of one-dimensional perspectives of a multidimensional probability distribution using linear projections, facilitated by the Radon transform. The distance between two distributions is then determined by evaluating the $p$-Wasserstein distance of these one-dimensional representations, i.e., their marginal distributions. Formally, the sliced $p$-Wasserstein distance between $I_\mu$ and $I_\nu$ is expressed as:
\begin{equation}
SW_p(I_\mu,I_\nu)=\left( \int_{\mathbb{S}^{d-1}} W^p_p\big( \mathcal{R} I_\mu(.,\theta), \mathcal{R} I_\nu(.,\theta) \big) \, d\theta \right)^{\frac{1}{p}}.
\label{eq:SW}
\end{equation}
This serves as a true distance function, complying with positive-definiteness, symmetry, and the triangle inequality conditions \citep{bonnotte2013unidimensional,kolouri2016sliced}. The SW distance is computed through integration over the unit sphere in $\mathbb{R}^{d}$. In practice, this is approximated using a Monte Carlo method that samples $\{ \theta_l \}$ uniformly from $\mathbb{S}^{d-1}$ and then replaces the integral with an average over these samples:
\begin{equation}
SW_p(I_\mu,I_\nu) \approx \left(\frac{1}{L}\sum_{l=1}^L W^p_p\big(\mathcal{R}I_\mu(\cdot, \theta_l), \mathcal{R}I_\nu(\cdot, \theta_l)\big)\right)^{\frac{1}{p}}
\label{eq:empiricalSWD}
\end{equation}
The sliced $p$-Wasserstein distance is beneficial in practice: if $\mathcal{R}I_\mu(\cdot, \theta_l)$ and $\mathcal{R}I_\nu(\cdot, \theta_l)$ can be determined for each sampled $\theta_l$, then the SW distance is derived from solving several one-dimensional optimal transport problems with closed-form solutions. This is particularly advantageous when estimating a high-dimensional PDF $I$, since one-dimensional kernel density estimation, necessary for PDF slices, simplifies the process. However, as dimensionality increases, more projections are needed to accurately estimate $I$ from $\mathcal{R}I(\cdot, \theta)$. Essentially, a two-dimensional distribution requiring $L$ projections demands approximately $\mathcal{O}(L^{d-1})$ projections for a similarly smooth $d$-dimensional distribution for $d \geq 2$. For clarification, consider $I_\mu=\mathcal{N}(0, I_d)$ and $I_\nu=\mathcal{N}(x_0, I_d)$, $x_0 \in \mathbb{R}^d$, as two multivariate Gaussian densities with identity covariance matrices. Their projections result in one-dimensional Gaussian distributions $\mathcal{R}I_\mu(\cdot,\theta)=\mathcal{N}(0,1)$ and $\mathcal{R}I_\mu(\cdot,\theta)=\mathcal{N}(\langle \theta, x_0 \rangle, 1)$. Here, $W_2(\mathcal{R}I_\mu(\cdot,\theta),\mathcal{R}I_\nu(\cdot,\theta))$ achieves its peak when $\theta=\frac{x_0}{\|x_0\|_2}$ and measures zero for $\theta$ orthogonal to $x_0$. Randomly chosen vectors from the unit sphere tend to be nearly orthogonal in high dimensions. More precisely, the inequality $Pr(| \langle \theta, \frac{x_0}{\|x_0\|_2} \rangle | < \epsilon) > 1-e^{(-d\epsilon^2)}$ indicates that in high dimensions, most sampled $\theta$s will be nearly orthogonal to $x_0$, leading to $W_2(\mathcal{R}I_\mu(\cdot,\theta), \mathcal{R}I_\nu(\cdot,\theta))\approx 0$ with high probability. To counteract this, one can skip uniform sampling of the unit sphere, selecting $\theta$s that contain significant distinguishing information between $I_\mu$ and $I_\nu$. This approach is used by \citet{deshpande2018generative}, who first computes a linear discriminant subspace and then assesses the empirical SW distance by choosing the $\theta$s as the subspace’s discriminant components. A similar but less empirical method involves the maximum sliced $p$-Wasserstein (max-SW) distance, an alternative OT metric, defined as:
\begin{equation}
\text{max-}SW (I_\mu,I_\nu) = \max_{\theta \in \mathbb{S}^{d-1}} W_p\big(\mathcal{R} I_\mu(\cdot,\theta),\mathcal{R} I_\nu(\cdot,\theta) \big).
\label{eq:msw}
\end{equation}

\section{Algorithms}
\label{sec:algorithm}

\textbf{Computational algorithms of RASGW and IWRASGW.} We present the pseudo-codes for computing RASGW and IWRASGW with Monte Carlo estimation in Algorithm~\ref{alg:RASGW} and Algorithm~\ref{alg:IWRASGW}.

\begin{algorithm}[!th]
\caption{Computational algorithm of RASGW}
\begin{algorithmic}[1]
\label{alg:RASGW}
\REQUIRE Probability measures $\mu$ and $\nu$, $p\geq 1$, the number of projections $L$, $0<\kappa<\infty$
  
  \FOR{$l=1$ to $L$}
  \STATE Sample $X, X' \sim \mu$, $Y, Y' \sim \nu$
  \STATE Sample $\theta_l \sim D_{\kappa}(X,X',Y,Y')$
  \STATE Compute $v_l = \text{GW}_p(\theta_l \sharp \mu,\theta_l \sharp  \nu)$
  % \STATE Compute $w_l = f(\text{GW}_p(\theta_l \sharp \mu,\theta_l \sharp  \nu))$
  \ENDFOR
  \STATE Compute $\widehat{\text{RASGW}}_p(\mu,\nu;L,\sigma_\kappa) = \left(\frac{1}{L}\sum_{l=1}^L v_l \right)^{\frac{1}{p}}$
  
 \STATE \textbf{return} $\widehat{\text{RASGW}}_p(\mu,\nu;L,\sigma_\kappa)$
\end{algorithmic}
\end{algorithm}

\begin{algorithm}[!th]
\caption{Computational algorithm of the IWRASGW}
\begin{algorithmic}[1]
\label{alg:IWRASGW}
\REQUIRE Probability measures $\mu$ and $\nu$, $p\geq 1$, the number of projections $L$,  $0<\kappa<\infty$, and the energy function $f$.
  
  \FOR{$l=1$ to $L$}
  \STATE Sample $X, X' \sim \mu$, $Y, Y' \sim \nu$
  \STATE Sample $\theta_l \sim D_{\kappa}(X,X',Y,Y')$
  \STATE Compute $v_l = \text{GW}_p(\theta_l \sharp \mu,\theta_l \sharp  \nu)$
  \STATE Compute $w_l = f(\text{GW}_p(\theta_l \sharp \mu,\theta_l \sharp  \nu))$
  \ENDFOR
  \STATE Compute $\widehat{\text{IWRASGW}}_p(\mu,\nu;\sigma_\kappa,L,f) = \left(\sum_{l=1}^L v_l \frac{w_l}{\sum_{i=1}^L w_i}\right)^{\frac{1}{p}}$
  
 \STATE \textbf{return} $\widehat{\text{IWRASGW}}_p(\mu,\nu;\sigma_\kappa,L,f) $
\end{algorithmic}
\end{algorithm}

\textbf{Computation of GW in 1D.} Let us consider input distributions $\mu, \nu \in \mathcal{P}_{4}(\mathbb{R}^{d})$ such that the projections $\mu_{\theta} \coloneqq \theta_{\#}\mu$ (similarly define $\nu_{\theta}$), given that $\theta \sim \sigma_{\textrm{RA}}$ preserve finiteness in moments. The equality is understood setwise. Observe that, calculating $\textrm{RASGW}_{2}^{2}(\mu, \nu)$ essentially underlies computing $\textrm{GW}_{2}^{2}(\mu_{\theta}, \nu_{\theta})$, given replicates of $\theta$. Now, if both $\mu$ and $\nu$ are centered (without loss of generality), the following decomposition holds due to \citet{zhang2024gromov}
\begin{align*}
    \textrm{GW}_{2}^{2}(\mu_{\theta}, \nu_{\theta}) = S_{1}(\mu_{\theta}, \nu_{\theta}) + S_{2}(\mu_{\theta}, \nu_{\theta}),
\end{align*}
such that 
\begin{align}
    &S_{1}(\mu_{\theta}, \nu_{\theta}) = \int\big(\theta(x)-\theta(x')\big)^{4} d\mu^{\otimes 2} + \int\big(\theta(y)-\theta(y')\big)^{4} d\nu^{\otimes 2} - 4 \int {\theta(x)}^{2}{\theta(y)}^{2} d\mu \otimes \nu \label{S1}\\ & S_{2}(\mu_{\theta}, \nu_{\theta}) = \inf_{a \in [0.5W_{-}, 0.5W_{+}]} 32a^{2} + \inf_{\pi \in \Pi(\mu_{\theta}, \nu_{\theta})} \int (-4x^{2}y^{2} - 32axy)d\pi, \label{S2}
\end{align}
where $W_{-} = \inf_{\pi \in \Pi(\mu_{\theta}, \nu_{\theta})} \int xy \:d\pi(x,y)$, $W_{+} = \sup_{\pi \in \Pi(\mu_{\theta}, \nu_{\theta})} \int xy \:d\pi(x,y)$, and we write $\theta(\cdot) \coloneqq \langle \theta, \cdot \rangle$. Since $S_{1}$ is independent of the coupling $\pi$, solving the $\textrm{RASGW}$ problem boils down to solving a constrained OT (as given in (\ref{S2})) once we have identified an optimal $a^{*}$. As such, there may indeed be linear solvers for sliced GW variants. However, the question arises whether they exist uniformly corresponding to all optimal $a^{*}$'s in the feasible set, as \citet{beinert2023assignment} construct cases where the \textit{cyclic} permutation attains a lower GW cost compared to \textit{identity} and \textit{anti-identity} permulations, rendering the latter two sub-optimal in 1D. Here, note that the boundaries (involving $W_{-}$ and $W_{+}$) rather depend on sliced OT costs. In this line, \citet{zhang2024gromov} shows that in an empirical setup, solving the $\textrm{GW}_{2}^{2}$ distance, i.e.,
\begin{align*}
    \frac{1}{n^{2}} \min_{\sigma} \sum_{i=1}^{n} \sum_{j=1}^{n} \abs{\:\abs{{\theta(x)}_{i} - {\theta(x)}_{j}}^{2} - \abs{{\theta(y)}_{\sigma(i)} - {\theta(y)}_{\sigma(j)}}^{2}}^{2},
\end{align*}
optimally using $\sigma$ = identity or anti-identity permutation is possible if and only if $a^{*} \in \{0.5W_{-}, 0.5W_{+}\}$. Here, ${\theta(x)}_{i}$ and ${\theta(y)}_{i}$ correspond to the $i$th order statistic in a pool of $n$ projected samples from $\mu$ and $\nu$, respectively. As such, the counterexample \citet{beinert2023assignment} provides belong to the cases other than the boundaries of $a$. While given random samples out of input distributions, it is quite challenging to identify the situation every time one calculates RASGW (or, Sliced GW), the point we try to make in the paper is that the best attainable rate \textit{can} be made $\mathcal{O}(n \log n)$. In this context, we also mention \citet{vayer2019sliced}'s later remark that ``often", optimal plans are sufficiently achieved using the identity or anti-identity permutation in numerical simulations. \iffalse Our experiments (Section \ref{sec:experiments}) corroborate the claim based on generative tasks, as RASGW and IWRASGW clock much smaller convergence time, even compared to linear approximations of GW in 1D \citep{scetbon2022linear}, while consistently producing better generation quality. As such, even if the resultant plans are suboptimal, our methods surpass SOTA competition to date in both efficiency and alignment accuracy. \fi However, we stress that the same may not hold for arbitrary $p$ (instead of $2$) and supports other than Euclidean, since the decomposition and the following result need to be proved in general. 

In light of the discussion, Algorithm \ref{alg:RASGW} can be understood as Algorithm \ref{algo:new}. 

\begin{algorithm}[t]
\caption{Relation-Aware Sliced Gromov-Wasserstein between discrete measures in an Euclidean setup}
\begin{algorithmic}[1]
\label{algo:new}
\REQUIRE $\mu = \frac{1}{n}\sum_{i=1}^{n} \delta_{x_i} \in \mathcal{P}(\mathbb{R}^d)$ and $\nu = \frac{1}{n}\sum_{i=1}^{n} \delta_{y_i} \in \mathcal{P}(\mathbb{R}^{d'})$, given $d' < d$; number of Monte Carlo samples $M$; location-scale distribution $\sigma_\kappa$ on $\mathbb{S}^{d-1}$, $0 < \kappa < \infty$
\STATE $y_i \leftarrow \Delta(y_i)$, $\forall i$
\FOR{$l = 1, \ldots, M$}
    \STATE Sample $\theta_l \sim \sigma_{\textrm{RA}}\left(\theta; \mu, \nu, \sigma_\kappa\right)$ \COMMENT{Definition 4.4}
    \STATE Sort $(\langle x_i, \theta_l \rangle)_i$ and $(\langle y_j, \theta_l \rangle)_j$ in increasing order
    \STATE Solve $\min_{\sigma \in S_n} \sum_{i,j} -\big({\theta_l(x)}_i - {\theta_l(x)}_j\big)^2\big({\theta_l(y)}_{\sigma(i)} - {\theta_l(y)}_{\sigma(j)}\big)^2$, where $S_n$ is the set of all permutations of $\{1,\ldots,n\}$ \COMMENT{solution: $\sigma_{\theta_l} \in$ Id or Anti-Id, due to \citet{vayer2019sliced}}
\ENDFOR
\STATE \textbf{return} $\widehat{\text{RASGW}^{2}_{2}}(\mu,\nu;\sigma_\kappa,M)=\frac{1}{n^2 M} \sum_{l=1}^{M} \sum_{i,j=1}^{n} \left( (\langle x_i - x_j, \theta_l \rangle)^2 - (\langle y_{\sigma_{\theta_l}(i)} - y_{\sigma_{\theta_l}(j)}, \theta_l \rangle)^2 \right)^2$
\end{algorithmic}
\end{algorithm}

%\newpage
\FloatBarrier
\section{Additional Experimental Details}
\label{sec:add_exps}

The gradient calculation of other sliced GW variants follows similarly to RASGW. For example, given samples $\theta_{11,\phi},\ldots,\theta_{HL,\phi} \sim \sigma_{\textrm{RA}}(\theta;\mu_\phi,\nu,\kappa)$, we have
\begin{align*}
    \nabla_\phi\text{IWRASGW}_p^p(\mu_\phi,\nu;\sigma_\kappa,L,H) =\frac{1}{H}\sum_{h=1}^H\bigg[\nabla_\phi\sum_{l=1}^L \text{GW}_p^p(\theta_{hl,\phi} \sharp \mu_\phi, \theta_{hl,\phi} \sharp \nu) \frac{w_{hl,\phi}(f;\mu_\phi,\nu)}{\sum_{j=1}^L w_{hj,\phi}(f;\mu_\phi,\nu)} \bigg].
\end{align*}

\subsection{Ablation Study}
We observe that for a given value of $\kappa$, increasing the number of projection replicates $\theta_{1},\ldots,\theta_{M} \sim \sigma_{\textrm{RA}}$ (in approximating RASGW using MC estimates) indeed results in better alignment (i.e., declining overall $\textrm{GW}_{2}$ values). This is expected due to Proposition \ref{proposition:MCerror}, as the realized bias shrinks gradually. However, that comes at the cost of increasing time required. The following table (Table \ref{tab:sample2}) justifies our choice of $M=500$ as the corresponding alignment is satisfactory without incurring elevated time (e.g., at $M=10000$). The ablation of $L$ in an empirical IWRASGW setup follows a similar protocol. For $\kappa$, heuristically, the value should not be considered too large since it signifies the power of the density $\textrm{PS}(\theta; \epsilon,\kappa) \propto (1+\epsilon^{T}\theta)^{\kappa}$. Our ablations (Table \ref{tab:sample1}) suggest that at $\kappa=50$, the corresponding alignment improves significantly over smaller values and does not improve that much upon further increment. Overall, the simulations reflect that the proposed semi-metric is not too sensitive to the choice of $M$, but needs careful tuning when it comes to $\kappa$. The inference is anticipated as
is instrumental in choosing an appropriate relation-aware projection.

\begin{table}[htbp]
  \centering
  \caption{Ablation study of Hyperparameter $\kappa$ (when $M$ = 500) for RASGW in 3D to 2D experiment for 10000 steps}
  \label{tab:sample1}
  \begin{tabular}{c|c|c}
    \toprule
    $\kappa$ &  $\text{GW}_2$($\downarrow$) & $\text{Time (s)}$($\downarrow$) \\
    \midrule
    1 & 113.43$\pm$2.46 & \textbf{48.156$\pm$2.84} \\
    5 & 100.10$\pm$4.02 & 49.98$\pm$2.45 \\
    10 & 67.08$\pm$3.91 & 50.81$\pm$2.76 \\
    50&\underline{18.80}$\pm$4.64&51.89$\pm$4.09\\
    100&20.41$\pm$4.17&50.09$\pm$2.23\\
    500&19.16$\pm$5.50&49.75$\pm$2.98\\
    1000&20.65$\pm$5.21&49.34$\pm$2.43\\
    5000& 21.04$\pm$4.188  & \underline{48.90$\pm$3.01}       \\
    10000& \textbf{18.10$\pm$2.46}         & 51.21$\pm$3.19          \\
    50000&20.51$\pm$2.13&50.21$\pm$2.89\\
    % Add more rows as needed
    \bottomrule
  \end{tabular}
\end{table}
\begin{table}[htbp]
  \centering
  \caption{Ablation study of Hyperparameter $M$ (when $\kappa$ = 50) for RASGW in 3D to 2D experiment for 10000 steps}
  \label{tab:sample2}
  \begin{tabular}{c|c|c}
    \toprule
    $M$ &  $\text{GW}_2$($\downarrow$) & $\text{Time (s)}$($\downarrow$) \\
    \midrule
    1 & 90.42$\pm$16.11 & \textbf{49.33$\pm$2.84} \\
    5 & 40.23$\pm$5.50 & 49.90$\pm$2.45 \\
    10 & 36.77$\pm$5.30 & \underline{49.81$\pm$3.21} \\
    50&29.05$\pm$5.98&50.12$\pm$3.09\\
    100&23.83$\pm$5.95&50.54$\pm$2.23\\
    500&18.80$\pm$4.64&51.89$\pm$4.09\\
    1000&\underline{18.32$\pm$5.21}&57.49$\pm$2.43\\
    5000& 19.67$\pm$4.188  & 115.66$\pm$3.01      \\
    10000& \textbf{17.38$\pm$2.46}         &185.78$\pm$11.09          \\
   
    % Add more rows as needed
    \bottomrule
  \end{tabular}
\end{table}

\subsection{GWGAN}
\subsubsection{Architecture}
\label{sec:architecture_gwgan}
\citet{bunne2019learning} proposed an architecture consisting of two primary components: the Generator (G) and the Adversary (A), which operate within an adversarial learning framework to generate and evaluate data distributions. The Generator is designed as a four-layer \textit{Multilayer Perceptron (MLP)} that processes a 256-dimensional input and produces a 2-dimensional output. Each hidden layer incorporates ReLU activation functions to introduce non-linearity. The Adversary is implemented as a five-layer \textit{MLP} that processes inputs based on their dimensionality. It accepts either a 3-dimensional or 2-dimensional input, applies ReLU activations at each layer, and outputs either a 3-dimensional or 2-dimensional result, depending on the processing path. To enhance training stability, they utilized Xavier Normal Initialization for the generator to mitigate vanishing gradients, while the adversary employed Orthogonal Initialization to maintain a well-conditioned weight space.
\begin{figure}[ht]
    \centering
    \scalebox{0.65}{ % Reduce the size of the whole figure
    \begin{tikzpicture}[node distance=2.2cm] % Increased node distance for better spacing

        % Define styles
        \tikzstyle{startstop} = [rectangle, rounded corners, minimum width=3cm, minimum height=1cm, text centered, draw=black, fill=red!20]
        \tikzstyle{process} = [rectangle, minimum width=3cm, minimum height=1cm, text centered, draw=black, fill=blue!15]
        \tikzstyle{decision} = [diamond, minimum width=3cm, minimum height=1cm, text centered, draw=black, fill=orange!25]
        \tikzstyle{arrow} = [thick,->,>=stealth]

        % Generator (Left Side)
        \node (gstart) [startstop] {Latent Input (256-D)};
        \node (g1) [process, below of=gstart] {Linear (256 $\to$ 128) + ReLU};
        \node (g2) [process, below of=g1] {Linear (128 $\to$ 128) + ReLU};
        \node (g3) [process, below of=g2] {Linear (128 $\to$ 128) + ReLU};
        \node (gend) [process, below of=g3] {Linear (128 $\to$ 2)};

        \draw [arrow] (gstart) -- (g1);
        \draw [arrow] (g1) -- (g2);
        \draw [arrow] (g2) -- (g3);
        \draw [arrow] (g3) -- (gend);

        % Label for Generator Network
        \node at (gstart) [yshift=2cm] {\textbf{Generator Network}}; % Label above the generator network

        % Adversary (Right Side) - Starts at the same height as Generator
        \node (ainput) [startstop, right of=gstart, xshift=10cm, yshift=0.5cm] {Input (2-D / 3-D)}; % Increased xshift and yshift for better alignment
        \node (adecision) [decision, below of=ainput] {Data Type};
        \node (adata) [process, left of=adecision, xshift=-3cm] {Linear (3 $\to$ 32) + ReLU};
        \node (agen) [process, right of=adecision, xshift=3cm] {Linear (2 $\to$ 32) + ReLU};
        \node (ashared) [process, below of=adecision] {Linear (32 $\to$ 16) + ReLU};
        \node (ashared2) [process, below of=ashared] {Linear (16 $\to$ 8) + ReLU};
        \node (adataout) [process, left of=ashared2, xshift=-3cm] {Linear (8 $\to$ 3)};
        \node (agenout) [process, right of=ashared2, xshift=3cm] {Linear (8 $\to$ 2)};

        \draw [arrow] (ainput) -- (adecision);
        \draw [arrow] (adecision) -- (adata);
        \draw [arrow] (adecision) -- (agen);
        \draw [arrow] (adata) -- (ashared);
        \draw [arrow] (agen) -- (ashared);
        \draw [arrow] (ashared) -- (ashared2);
        \draw [arrow] (ashared2) -- (adataout);
        \draw [arrow] (ashared2) -- (agenout);

        % Label for Adversary Network
        \node at (ainput) [yshift=2cm] {\textbf{Adversary Network}}; % Label above the adversary network

    \end{tikzpicture}
    }
    \caption{Flowchart of the Generator and Adversary networks.}
    \label{fig:architecture_flowchart}
\end{figure}
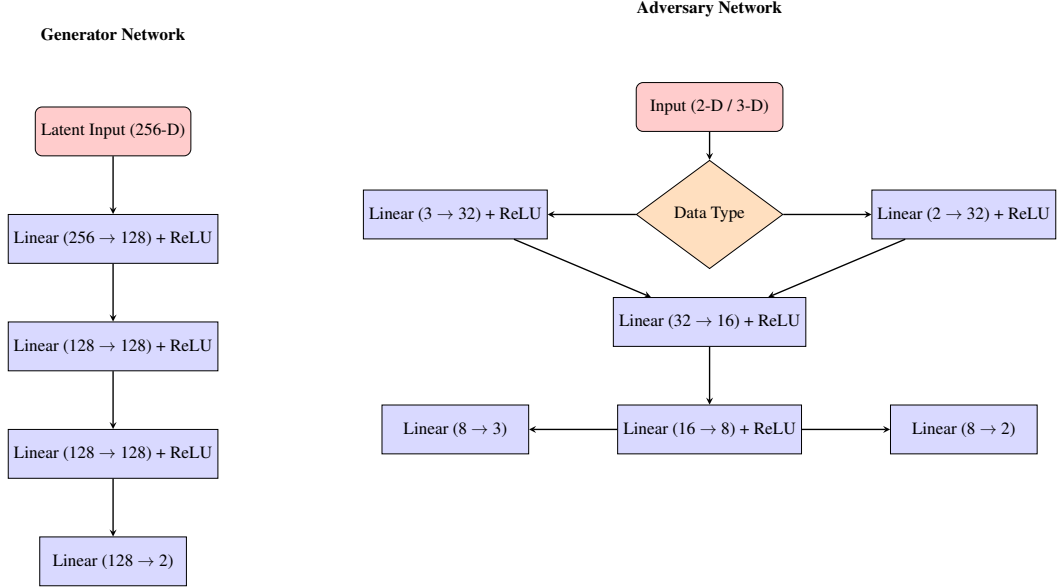

\FloatBarrier

 \subsubsection{For Identical Spaces:}\label{2d_2d} For the experiment of generation of 2D gaussian distribution from a 2D distribution, the Adversary neural network was used along with $l_1$-Regularization to ensure fast convergence. This experiment was performed for 10,000 iterations with RASGW as the loss function during training.
\begin{figure}[ht]
    \centering
    \begin{minipage}[b]{0.6\textwidth}
        \centering
        \includegraphics[width=\linewidth]{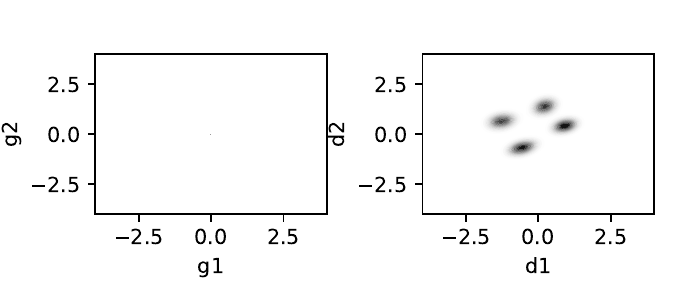}
    \end{minipage}
    \begin{minipage}[b]{0.32\textwidth}
        \centering
        \includegraphics[width=\linewidth]{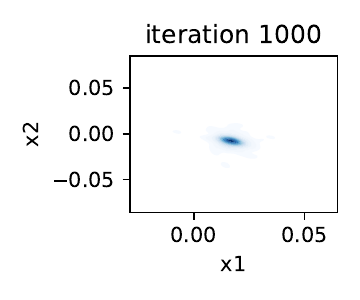}
    \end{minipage}
    \caption{At iteration 1000, using RASGW as the loss function in training: 
    \textbf{Left:} Feature space of the Generator, 
    \textbf{Middle:} Feature space of the Adversary, 
    \textbf{Right:} Generated image. 
    All plots are done using 1000 samples.}
    \label{fig:both_images}
\end{figure}
% \begin{figure}[ht]
%     \centering
%     \begin{minipage}[b]{0.6\textwidth}
%         \centering
%         \includegraphics[width=\linewidth]{feature_004.pdf}
%     \end{minipage}
%     \begin{minipage}[b]{0.32\textwidth}
%         \centering
%         \includegraphics[width=\linewidth]{gen_004.pdf}
%     \end{minipage}
%     \caption{At iteration 5000, using RASGW as the loss function in training: 
%     \textbf{Left:} Feature space of the Generator, 
%     \textbf{Middle:} Feature space of the Adversary, 
%     \textbf{Right:} Generated image. 
%     All plots are done using 1000 samples.}
%     \label{fig:both_images}
% \end{figure}
\begin{figure}[ht]
    \centering
    \begin{minipage}[b]{0.6\textwidth}
        \centering
        \includegraphics[width=\linewidth]{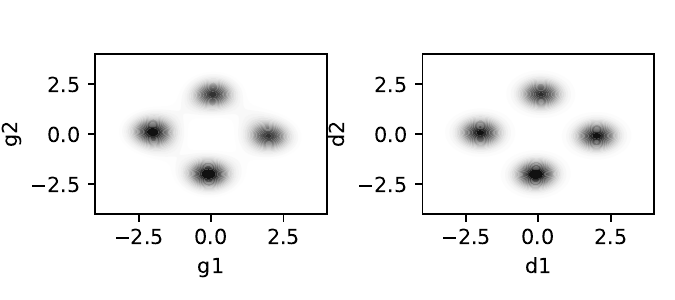}
    \end{minipage}
    \begin{minipage}[b]{0.32\textwidth}
        \centering
        \includegraphics[width=\linewidth]{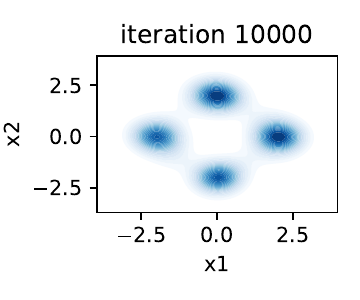}
    \end{minipage}
    \caption{At iteration 10000, using RASGW as the loss function in training: 
    \textbf{Left:} Feature space of the Generator, 
    \textbf{Middle:} Feature space of the Adversary, 
    \textbf{Right:} Generated image. 
    All plots are done using 1000 samples.}
    \label{fig:both_images_}
\end{figure}

\FloatBarrier
\subsubsection{For Incomparable Spaces} 
\label{sec:incomparable}
For the experiment on generating a 2D Gaussian distribution from a 3D Gaussian distribution or the other way around (a 3D Gaussian distribution from a 2D Gaussian distribution), the adversary neural network was not used following \citet{bunne2019learning}. This experiment was also performed for 10,000 iterations with different sliced variants as the loss function during training.

\subsubsection{The 4-point Experiment} \label{4pt}
The Gaussian 4 experiment generates a dataset with 4 clusters. In the 3D case, the clusters are centered at the points (1, 0, 0), (-1, 0, 0), (0, 1, 0), and (0, -1, 0). Each point is sampled from a Gaussian distribution with small noise around these centers, scaled by a factor of 2. This experiment produces 4 distinct clusters in 3D space. It is also possible to use this experiment in 2D by setting the Z-axis coordinate to 0 for each cluster center, resulting in clusters lying in the XY-plane. The points are randomly assigned to one of these clusters, with the corresponding label indicating which center the point was sampled from. Thus, both 3D and 2D versions of the experiment are possible.  It was performed for 10,000 iterations.

\begin{figure}[ht]
\centering
\begin{minipage}{.5\textwidth}
  \centering
  \includegraphics[width=\textwidth]{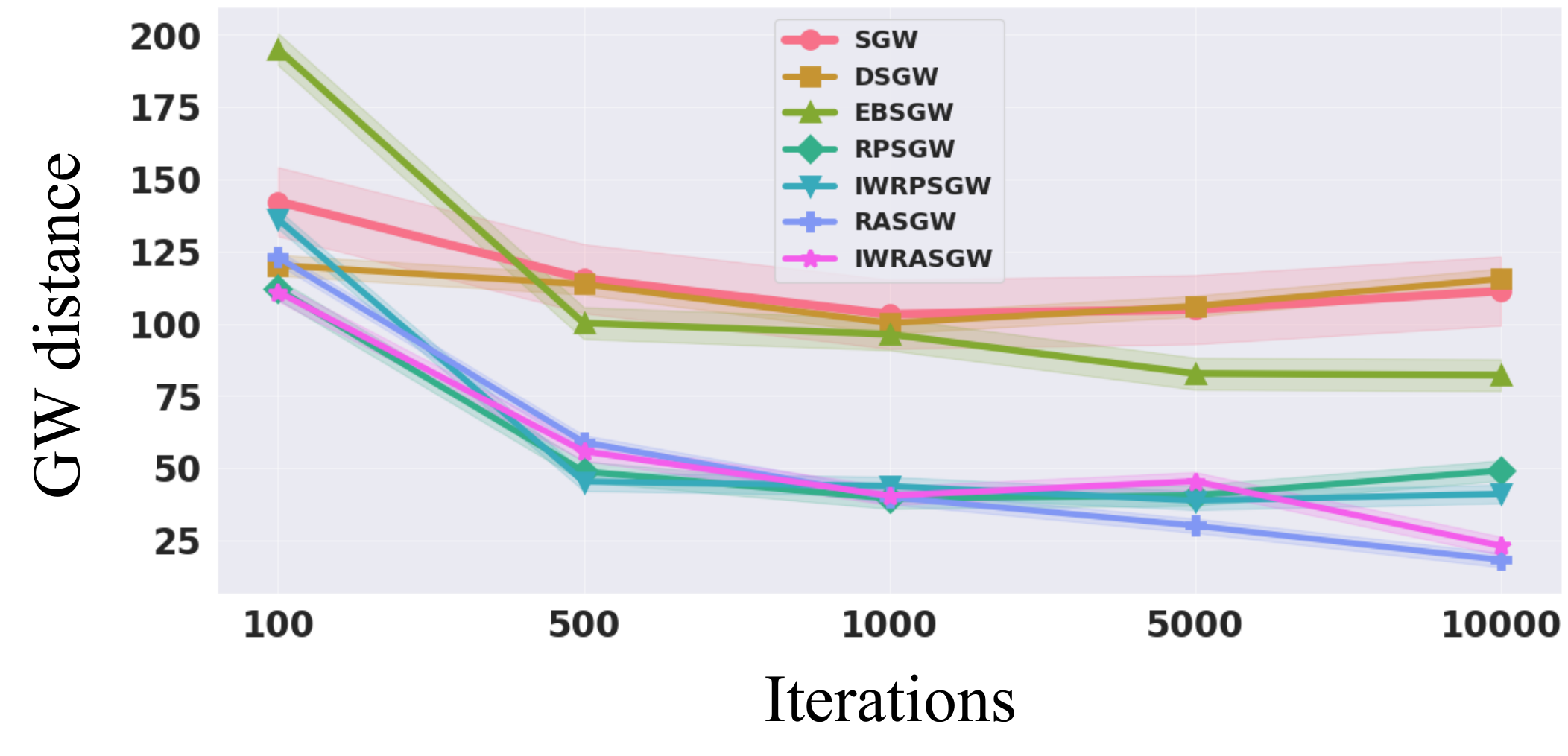}
  \label{fig:gw_iter}
\end{minipage}%
\begin{minipage}{.5\textwidth}
  \centering
  \includegraphics[width=\textwidth]{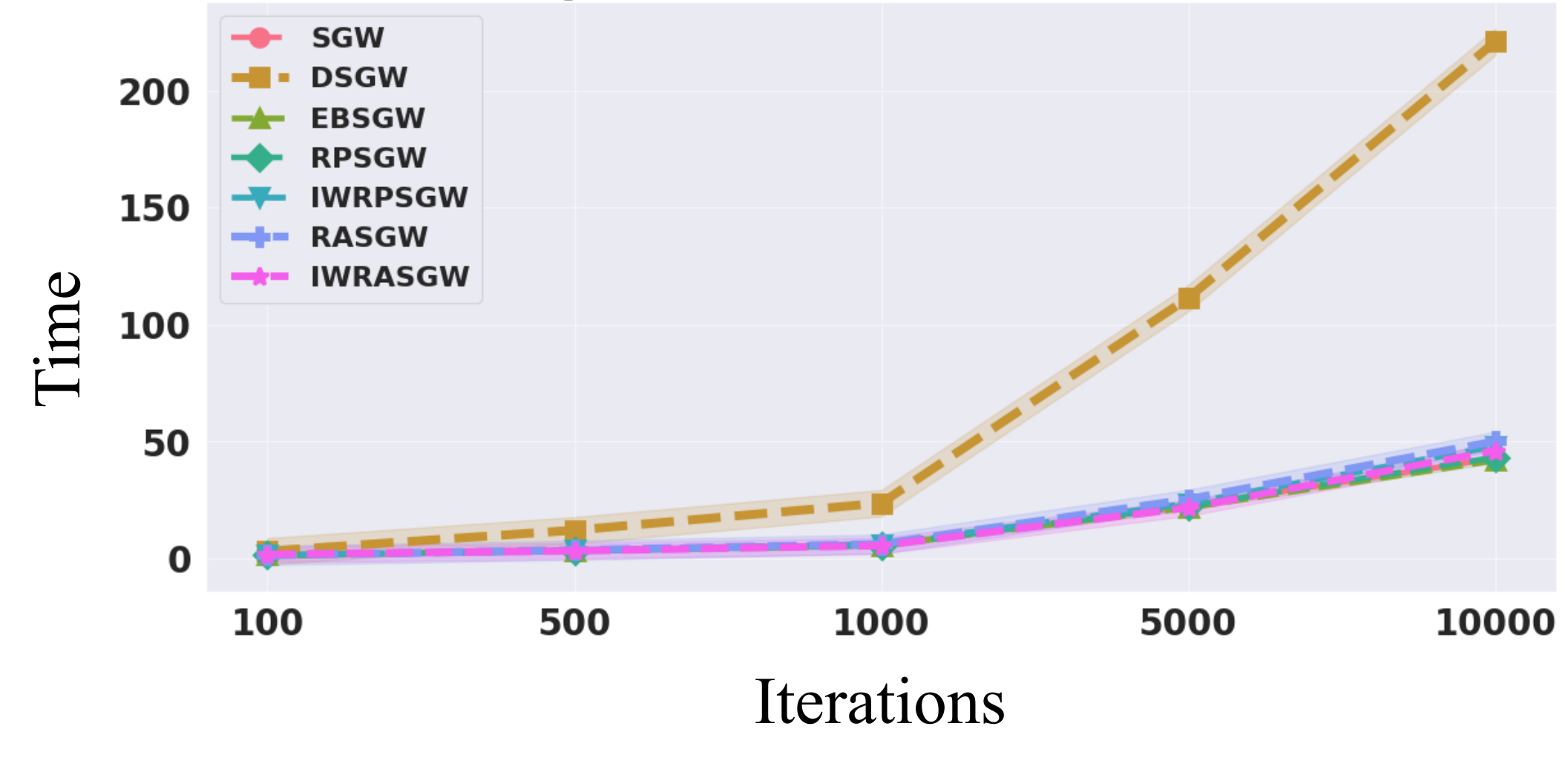}
  \label{fig:time_iter}
\end{minipage}
\vspace{-10pt}
\caption{(\textit{left}) GW distances and (\textit{right}) computation time for 4-point 3D$\rightarrow$2D generations.}
\label{fig:2_3}
\end{figure}

\begin{table*}[!t]
    \centering
    \scriptsize
    \caption{\footnotesize{Gromov Wasserstein-2 distance and computational times across iterations in generation of 2D distribution from 3D distribution}}
    \scalebox{0.85}{
    \begin{tabular}{l|cc|cc|cc|cc|cc|}
    \toprule
     Distance&\multicolumn{2}{c|}{Step 100}&\multicolumn{2}{c|}{Step 500}&\multicolumn{2}{c|}{Step  1000}&\multicolumn{2}{c|}{Step  5000 }&\multicolumn{2}{c|}{Step  10000 }\\
     \cmidrule{2-11}
     & $\text{GW}_2$($\downarrow$) &$\text{Time (s)}$($\downarrow$) & $\text{GW}_2$($\downarrow$) &$\text{Time (s)}$($\downarrow$)& $\text{GW}_2$($\downarrow$) &$\text{Time (s)}$($\downarrow$) & $\text{GW}_2$($\downarrow$) &$\text{Time (s)}$($\downarrow$) &$\text{GW}_2$($\downarrow$) &$\text{Time (s)}$($\downarrow$)\\
     \midrule
    SGW & 142.30 & \textbf{1.26} & 115.64 & \underline{3.22} & 103.29 & \underline{5.53} & 104.91 & \underline{22.34} & 111.32 & \underline{42.82} \\
    Max-SGW & 399.92 & 2.88 & 390.43 & 10.61 & 390.09 & 20.29 & 382.56 & 95.47 & 379.07 & 180.87 \\
    DSGW & 120.25 & 3.07 & 113.71 & 12.04 & 100.23 & 23.62 & 106.12 & 111.33 & 115.54 & 220.90 \\
    EBSGW & 195.09 & 1.30 & 100.23 & \textbf{3.21} & 96.36 & \textbf{5.46} & 82.74 & \textbf{21.65} & 82.20 & \textbf{41.84} \\
    RPSGW & \underline{111.93} & \underline{1.29} & \underline{48.85} & 3.25 & \textbf{39.47} & 5.63 & \underline{40.58} & 22.36 & 49.16 & 43.20 \\
    IWRPSGW & 136.22 & 1.34 & \textbf{45.34} & 3.55 & 43.79 & 6.15 & 38.78 & 23.47 & 41.08 & 47.97 \\
    RASGW & 122.98 & 1.32 & 58.96 & 3.53 & \underline{39.97} & 6.17 & \textbf{30.08} & 25.30 & \textbf{18.21} & 50.08 \\
    IWRASGW & \textbf{111.03} & 1.76 & 55.82 & \textbf{3.21} & 40.46 & \textbf{5.46} & 45.43 & \textbf{21.65} & \underline{23.10} & 46.04 \\
    \bottomrule
    \end{tabular}
    }
    \label{tab:3dto2d}
\end{table*}

\begin{table*}[!t]
    \centering
    \scriptsize
    \caption{\footnotesize{Gromov Wasserstein-2 distance and computational times across iterations in generation of 3D distribution from 2D distribution}}
    \scalebox{0.85}{
    \begin{tabular}{l|cc|cc|cc|cc|cc|}
    \toprule
     Distance&\multicolumn{2}{c|}{Step 100}&\multicolumn{2}{c|}{Step 500}&\multicolumn{2}{c|}{Step  1000}&\multicolumn{2}{c|}{Step  5000 }&\multicolumn{2}{c|}{Step  10000 }\\
     \cmidrule{2-11}
     & $\text{GW}_2$($\downarrow$) &$\text{Time (s)}$($\downarrow$) & $\text{GW}_2$($\downarrow$) &$\text{Time (s)}$($\downarrow$)& $\text{GW}_2$($\downarrow$) &$\text{Time (s)}$($\downarrow$) & $\text{GW}_2$($\downarrow$) &$\text{Time (s)}$($\downarrow$) &$\text{GW}_2$($\downarrow$) &$\text{Time (s)}$($\downarrow$)\\
     \midrule
    SGW & 20.35 & \underline{1.35} & 3.65 & \textbf{3.07} & \underline{2.10} & \textbf{5.23} & 1.67 & \textbf{21.25} & 2.59 & \textbf{41.98} \\
    Max-SGW & 14.80 & 2.77 & 4.39 & 10.38 & 3.60 & 19.71 & 2.09 & 94.08 & 2.50 & 186.66 \\
    DSGW & \underline{13.01} & 3.33 & \underline{3.34} & 12.34 & 3.14 & 23.52 & 3.23 & 111.70 & 1.56 & 222.39 \\
    EBSGW & 23.10 & \textbf{1.33} & 4.66 & \underline{3.27} & 3.46 & \underline{5.38} & 3.72 & \underline{21.86} & 1.82 & \underline{42.33} \\
    RPSGW & 23.04 & 1.39 & \textbf{3.25} & 3.62 & 3.29 & 6.26 & 2.66 & 26.62 & 1.60 & 52.51 \\
    IWRPSGW & 23.42 & 1.37 & 5.47 & 3.50 & \textbf{2.00} & 6.13 & \underline{1.35} & 26.69 & 1.47 & 52.04 \\
    RASGW & \textbf{12.31} & 1.75 & 4.67 & 4.64 & 2.71 & 8.67 & 1.63 & 39.38 & \underline {1.26} & 79.92 \\
    IWRASGW & 29.12 & 1.80 & 5.86 & 4.96 & 2.99 & 9.04 & \textbf{0.73} & 41.00 & \textbf{1.03} & 78.31 \\
    \bottomrule
    \end{tabular}
    }
    \label{tab:2dto3d}
\end{table*}

\begin{figure*}[!ht]
    \centering
    \includegraphics[width=\textwidth]{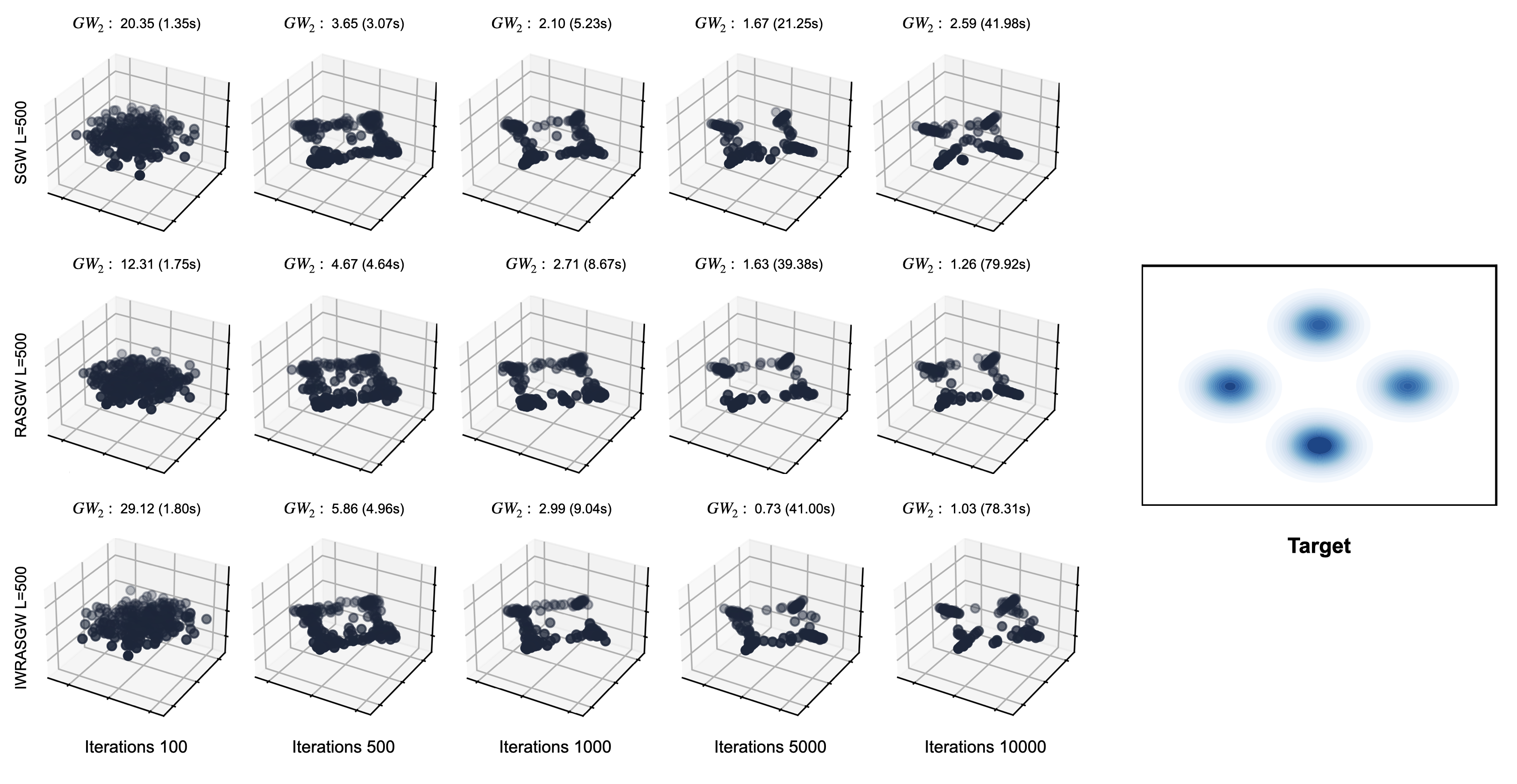}
    \caption{Generation of 3D distribution from 2D distribution for the 4-point experiment}
    \label{fig:2dto3d}
    \end{figure*}
\begin{figure*}[!ht]
    %\vspace{1em} % Add some space between the two figures
    \includegraphics[width=\textwidth]{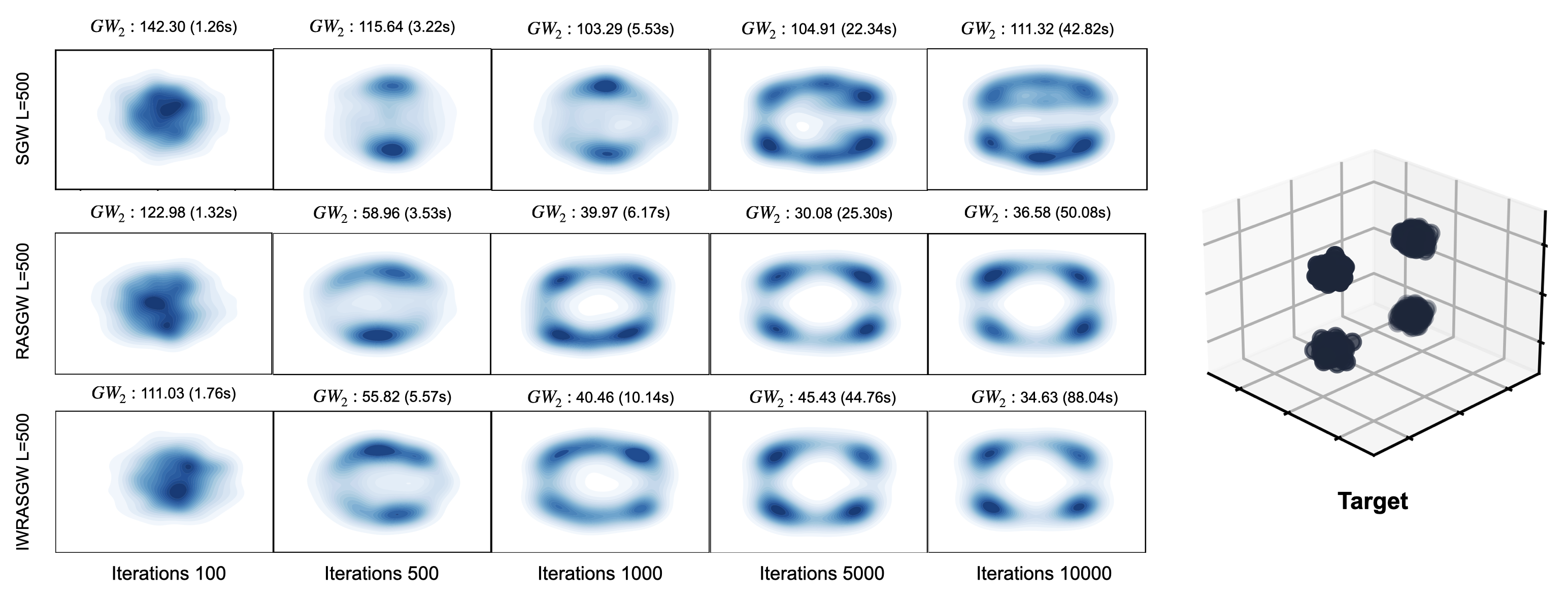}
    \caption{Generation of 2D distribution from 3D distribution for 4-point experiment}
    \label{fig:3dto2d}
\end{figure*}
\FloatBarrier
\subsubsection{The 8-point Experiment}
The Gaussian 8 experiment generates a dataset with 8 clusters arranged symmetrically around the origin. The centers of the clusters are at the points: (1, 0), (-1, 0), (0, 1), (0, -1), and four diagonal points: $( \frac{1}{\sqrt{2}}, \frac{1}{\sqrt{2}} , 
( \frac{1}{\sqrt{2}}, \frac{-1}{\sqrt{2}} ), 
( \frac{-1}{\sqrt{2}}, \frac{1}{\sqrt{2}} ), 
( \frac{-1}{\sqrt{2}}, \frac{-1}{\sqrt{2}} )$
. Each point is sampled from a Gaussian distribution with small noise around these centers, scaled by a factor of 2. This experiment generates 8 distinct clusters, and the points are randomly assigned to these clusters, with each cluster having its own corresponding label. Now, we aim to generate a 3D distribution for this target 2D distribution. It was performed for 10,000 iterations.
\begin{figure}[ht]
    \centering
    % First row of images
    \begin{minipage}[b]{0.3\textwidth}
        \centering
        \includegraphics[width=\linewidth]{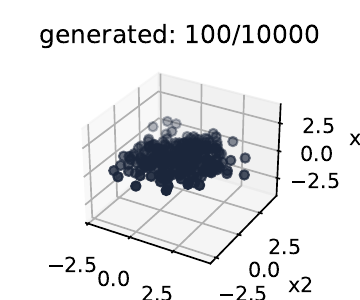}
        $\text{GW}_2:$ 19.75, time: 1.31s
    \end{minipage}
    \hfill
    \begin{minipage}[b]{0.3\textwidth}
        \centering
        \includegraphics[width=\linewidth]{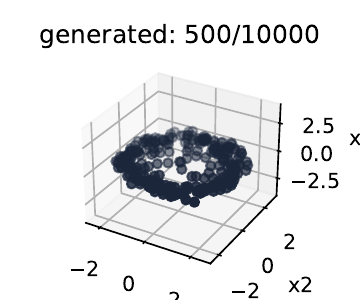}
        $\text{GW}_2:$ 4.27, time: 3.11s
    \end{minipage}
    \hfill
    \begin{minipage}[b]{0.3\textwidth}
        \centering
        \includegraphics[width=\linewidth]{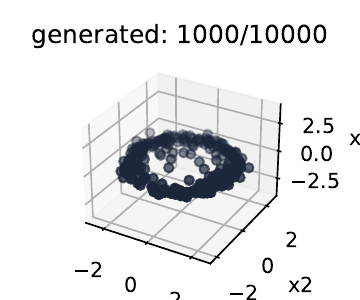}
        $\text{GW}_2:$ 5.96, time: 5.31s
    \end{minipage}
    
    % Second row of images
    \vspace{0.5cm} % Space between the two rows
    \begin{minipage}[b]{0.3\textwidth}
        \centering
        \includegraphics[width=\linewidth]{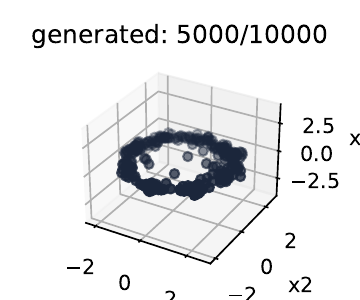}
        $\text{GW}_2:$ 7.55, time: 22.15s
        %\subcaption*{Feature Space 2} 
    \end{minipage}
    \hfill
    \begin{minipage}[b]{0.3\textwidth}
        \centering
        \includegraphics[width=\linewidth]{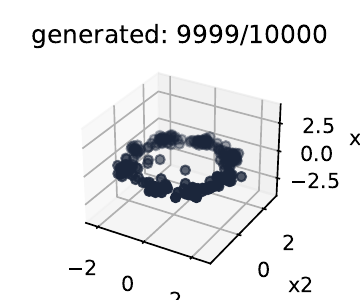}
        $\text{GW}_2:$ 8.70, time: 43.28s
    \end{minipage}
    \hfill
    \begin{minipage}[b]{0.3\textwidth}
        \centering
        \includegraphics[width=\linewidth]{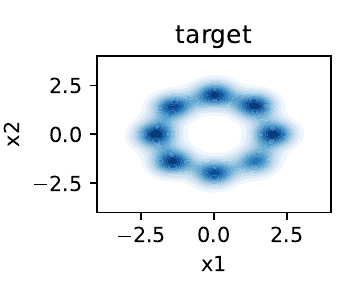}
        2D target distribution 
    \end{minipage}

    % Common caption for the entire figure
    \caption{Generation of 3D distribution from 2D distribution for the 8-point experiment using SGW}
    \label{fig:6pics_SGW}
\end{figure}

\begin{figure}ht
    \centering
    % First row of images
    \begin{minipage}[b]{0.3\textwidth}
        \centering
        \includegraphics[width=\linewidth]{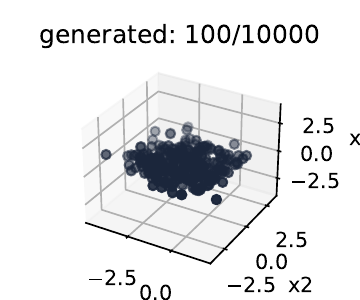}
        $\text{GW}_2:$ 24.22, time: 1.39s
    \end{minipage}
    \hfill
    \begin{minipage}[b]{0.3\textwidth}
        \centering
        \includegraphics[width=\linewidth]{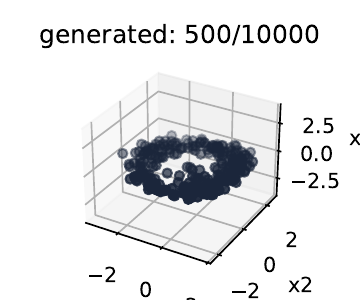}
        $\text{GW}_2:$ 3.84, time: 3.52s
    \end{minipage}
    \hfill
    \begin{minipage}[b]{0.3\textwidth}
        \centering
        \includegraphics[width=\linewidth]{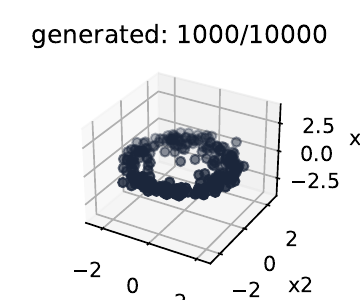}
        $\text{GW}_2:$ 6.02, time: 6.14s
    \end{minipage}
    
    % Second row of images
    \vspace{0.5cm} % Space between the two rows
    \begin{minipage}[b]{0.3\textwidth}
        \centering
        \includegraphics[width=\linewidth]{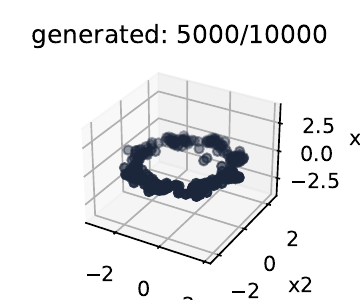}
        $\text{GW}_2:$ 3.26, time: 26.27s
        %\subcaption*{Feature Space 2} 
    \end{minipage}
    \hfill
    \begin{minipage}[b]{0.3\textwidth}
        \centering
        \includegraphics[width=\linewidth]{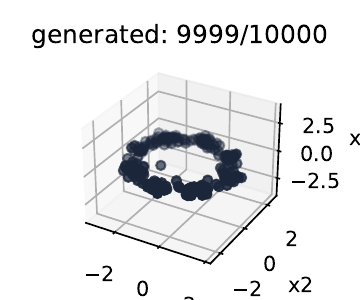}
        $\text{GW}_2:$ 4.98, time: 51.50s
    \end{minipage}
    \hfill
    \begin{minipage}[b]{0.3\textwidth}
        \centering
        \includegraphics[width=\linewidth]{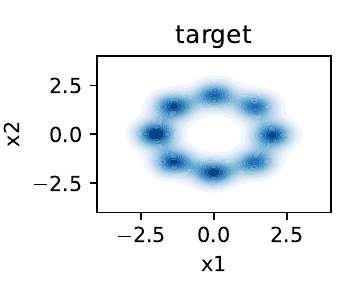}
        2D target distribution
    \end{minipage}

    % Common caption for the entire figure
    \caption{Generation of 3D distribution from 2D distribution for the 8-point experiment using RASGW}
    \label{fig:6pics}
\end{figure}

\begin{figure}[ht]
    \centering
    % First row of images
    \begin{minipage}[b]{0.3\textwidth}
        \centering
        \includegraphics[width=\linewidth]{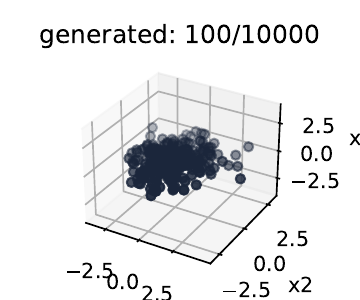}
        $\text{GW}_2:$ 21.12, time: 1.39s
    \end{minipage}
    \hfill
    \begin{minipage}[b]{0.3\textwidth}
        \centering
        \includegraphics[width=\linewidth]{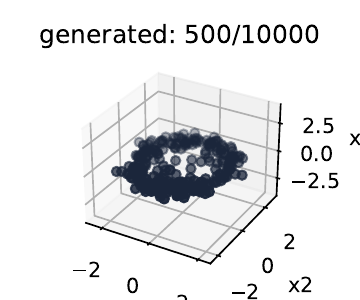}
        $\text{GW}_2:$ 4.15, time: 3.49s
    \end{minipage}
    \hfill
    \begin{minipage}[b]{0.3\textwidth}
        \centering
        \includegraphics[width=\linewidth]{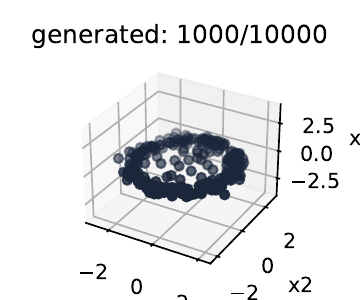}
        $\text{GW}_2:$ 3.90, time: 6.10s
    \end{minipage}
    
    % Second row of images
    \vspace{0.5cm} % Space between the two rows
    \begin{minipage}[b]{0.3\textwidth}
        \centering
        \includegraphics[width=\linewidth]{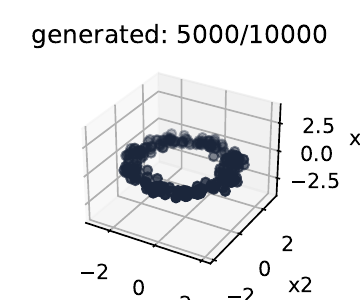}
        $\text{GW}_2:$ 4.62, time: 26.33s
        %\subcaption*{Feature Space 2} 
    \end{minipage}
    \hfill
    \begin{minipage}[b]{0.3\textwidth}
        \centering
        \includegraphics[width=\linewidth]{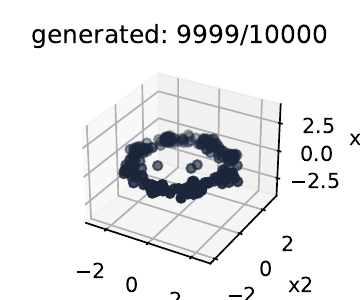}
        $\text{GW}_2:$ 3.78, time: 51.64s
    \end{minipage}
    \hfill
    \begin{minipage}[b]{0.3\textwidth}
        \centering
        \includegraphics[width=\linewidth]{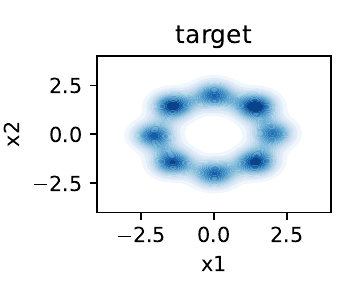}
        2D target distribution
    \end{minipage}

    % Common caption for the entire figure
    \caption{Generation of 3D distribution from 2D distribution for the 8-point experiment using IWRASGW}
    \label{fig:6pics_IWRASGW}
\end{figure}
\FloatBarrier
\subsection{GWAE}
\subsubsection{Architecture}
\label{sec:architecture_gwae}
The encoder and decoder networks in the GWAE models, as proposed by \citet{nakagawa2023gromovwasserstein}, follow a convolutional and fully connected architecture. For the encoder network, the input is a $64 \times 64$ RGB image with 3 channels. The encoder consists of several convolutional layers with increasing feature map sizes and decreasing spatial dimensions, followed by fully connected layers that output the latent variables $\mu$ and $\sigma^2$. The decoder network takes these latent variables as input, first expanding the latent space through fully connected layers and then using deconvolutional layers to progressively upsample the feature maps until the image is reconstructed. The encoder and decoder both use SiLU \cite{hendrycks2023gaussianerrorlinearunits} activations, except for the final layer of the decoder, which uses a sigmoid activation function to ensure that the output is in the range $[0, 1]$.

\begin{figure}[ht]
    \centering
    \scalebox{0.65}
    { % Adjust the size of the entire figure
    \begin{tikzpicture}[node distance=1.8cm] % Reduced node distance for shorter arrows
        % Define styles
        \tikzstyle{startstop} = [rectangle, rounded corners, minimum width=3cm, minimum height=1cm, text centered, draw=black, fill=red!20]
        \tikzstyle{process} = [rectangle, minimum width=3cm, minimum height=1cm, text centered, draw=black, fill=blue!15]
        \tikzstyle{arrow} = [thick,->,>=stealth]
        % \begin{scope}[execute at begin node=$, execute at end node=$]
        % Encoder (Left Side)
        \node (estart) [startstop] {Input Image $(3 \times 64 \times 64)$};
        \node (e1) [process, below of=estart] {Conv (3 $\to$ 32) + SiLU $(3 \times 64 \times 64) \to (32 \times 32 \times 32)$};
        \node (e2) [process, below of=e1] {Conv (32 $\to$ 64) + SiLU $(32 \times 32 \times 32) \to (64 \times 16 \times 16)$};
        \node (e3) [process, below of=e2] {Conv (64 $\to$ 128) + SiLU $(64 \times 16 \times 16) \to (128 \times 8 \times 8)$};
        \node (e4) [process, below of=e3] {Conv (128 $\to$ 256) + SiLU $(128 \times 8 \times 8) \to (256 \times 4 \times 4)$};
        \node (fc1) [process, below of=e4] {FC (256 $\to$ 256) + SiLU $(256 \times 4 \times 4) \to 256$};
        \node (mu) [process, below of=fc1] {FC (256 $\to$ L) for $\mu$ $256 \to L$};
        \node (sigma) [process, below of=mu] {FC (256 $\to$ L) for $\sigma^2$ $256 \to L$};
        \draw [arrow] (estart) -- (e1);
        \draw [arrow] (e1) -- (e2);
        \draw [arrow] (e2) -- (e3);
        \draw [arrow] (e3) -- (e4);
        \draw [arrow] (e4) -- (fc1);
        \draw [arrow] (fc1) -- (mu);
        \draw [arrow] (fc1) -- (sigma);
        % \end{scope}
        % Label for Encoder Network
        \node at (estart) [yshift=2cm] {\textbf{Encoder Network}}; 
        % Decoder (Right Side)
        \node (dstart) [startstop, right of=estart, xshift=9cm] {Latent Input $L$};
        \node (d1) [process, below of=dstart] {FC (L $\to$ 256) + SiLU $L \to 256$};
        \node (d2) [process, below of=d1] {FC (256 $\to$ 256) + SiLU $256 \to 256$};
        \node (d3) [process, below of=d2] {DeConv (256 $\to$ 128) + SiLU $(256 \times 4 \times 4) \to (128 \times 8 \times 8)$};
        \node (d4) [process, below of=d3] {DeConv (128 $\to$ 64) + SiLU $(128 \times 8 \times 8) \to (64 \times 16 \times 16)$};
        \node (d5) [process, below of=d4] {DeConv (64 $\to$ 32) + SiLU $(64 \times 16 \times 16) \to (32 \times 32 \times 32)$};
        \node (dend) [process, below of=d5] {DeConv (32 $\to$ 3) + Sigmoid $(32 \times 32 \times 32) \to (3 \times 64 \times 64)$};
        \draw [arrow] (dstart) -- (d1);
        \draw [arrow] (d1) -- (d2);
        \draw [arrow] (d2) -- (d3);
        \draw [arrow] (d3) -- (d4);
        \draw [arrow] (d4) -- (d5);
        \draw [arrow] (d5) -- (dend);
        % Label for Decoder Network
        \node at (dstart) [yshift=2cm] {\textbf{Decoder Network}}; 
    \end{tikzpicture}
    }
    \caption{The architecture used in the GWAE models has been designed for $64 \times 64$ RGB images. In the experiments, the input size is set to $(\text{Channels}, \text{Height}, \text{Width}) = (3, 64, 64)$. The terms ``FC" and ``Conv" refer to fully-connected (linear) layers and convolutional layers, respectively.}
    \label{fig:encoder_decoder_flowchart_shapes}
\end{figure}
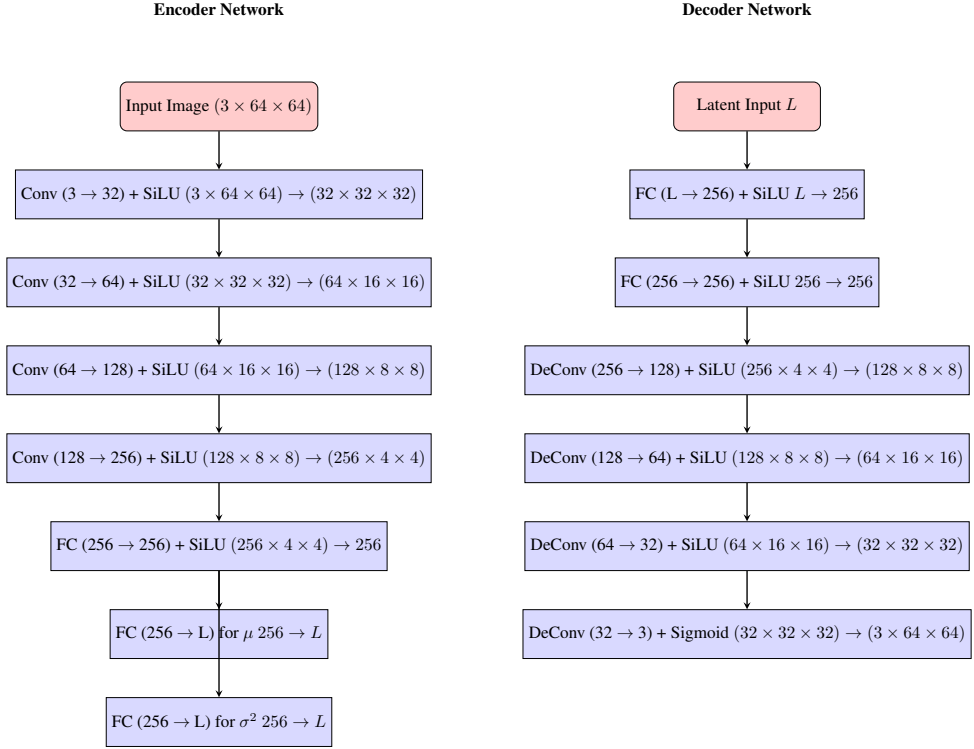

\FloatBarrier
\subsubsection{Datasets}
%\textbf{MNIST:}
\textbf{Omniglot.}
The Omniglot \citep{lake2015human} dataset contains 1,623 handwritten characters from 50 different alphabets, with a total of 80,000 binary-valued images. Each image is a 105×105 pixel representation of a character, with 20 samples per character. The dataset is divided into a training set of 1,100 characters and a test set of 523 characters, making it a challenging benchmark for few-shot learning models.\newline
\textbf{CIFAR-10.}
The CIFAR \citep{krizhevsky2009learning} dataset contains 60,000 32x32 color images across 10 classes in the CIFAR-10 version and 100 classes in the CIFAR-100 version. In CIFAR-10, the images are labeled into categories such as airplane, automobile, bird, cat, dog, and other objects, with each class containing 6,000 images. The dataset is split into 50,000 training images and 10,000 test images. In CIFAR-100, the images are divided into 100 classes, with 600 images per class, grouped into 20 superclasses. It has the same training-test split as CIFAR-10, with 50,000 training and 10,000 test images. The images are 3-channel color images, with each image measuring 32x32 pixels.
\subsubsection{Evaluation metrics}
\label{sec:evaluation_metrics}
\textbf{FID score.}
Fréchet Inception Distance (FID) \citep{heusel2017gans} is commonly used as a metric to evaluate the quality of images generated by generative models. The FID score is defined as the squared 2-Wasserstein distance between the feature distributions of real and generated images. It is assumed that the features of both real and generated images follow multivariate Gaussian distributions, with the real images having mean $\mu_r$ and covariance matrix $\Sigma_r$, and the generated images having mean $\mu_g$ and covariance matrix $\Sigma_g$.

The FID score is expressed as:

\[
\text{FID} = W_2^2\left(N(\mu_r, \Sigma_r), N(\mu_g, \Sigma_g)\right) = \|\mu_r - \mu_g\|_2^2 + \text{tr}\left(\Sigma_r + \Sigma_g - 2(\Sigma_r \Sigma_g)^{\frac{1}{2}}\right),
\]

where $W_2^2$ denotes the squared Wasserstein-2 distance between the Gaussian distributions of real and generated features. The FID score quantifies the discrepancy between these distributions, and lower values are associated with better generation performance, indicating that the generated images are closer to the real images in terms of feature distribution.

As stated in \citet{heusel2017gans}, we follow \citet{nakagawa2023gromovwasserstein} in using the features for the computation of FID, which are extracted from the final pooling layer of the Inception-v3 model that has been pre-trained on the ImageNet dataset \citep{deng2009imagenet}.

\textbf{Peak Signal-to-Noise Ratio.} The Peak Signal-to-Noise Ratio (PSNR) is used for evaluating the quality of image reconstruction. The PSNR is defined as

\[
\text{PSNR} = 20 \log_{10}(\text{MAX}) - 10 \log_{10}(\text{MSE}),
\]

where \text{MAX} denotes the maximum possible pixel value, and \text{MSE} represents the Mean Squared Error (MSE), which is calculated as

\[
\text{MSE} = \frac{1}{N} \sum_{i=1}^{N} (I_i - \hat{I}_i)^2,
\]

where \(I_i\) and \(\hat{I}_i\) represent the pixel values of the original and reconstructed images, respectively, and \(N\) is the total number of pixels. In all the experiments, the value of \text{MAX} is set to 1, as the images in the dataset are scaled within the range [0, 1].

\clearpage
\subsection{Results for CIFAR-10}
\label{sec:results_cifar}
\FloatBarrier
\subsubsection{SGW}
\begin{figure}[htp!]
    \centering
    % First row: Two images
    \begin{minipage}[b]{0.4\textwidth}
        \centering
        \includegraphics[width=\textwidth]{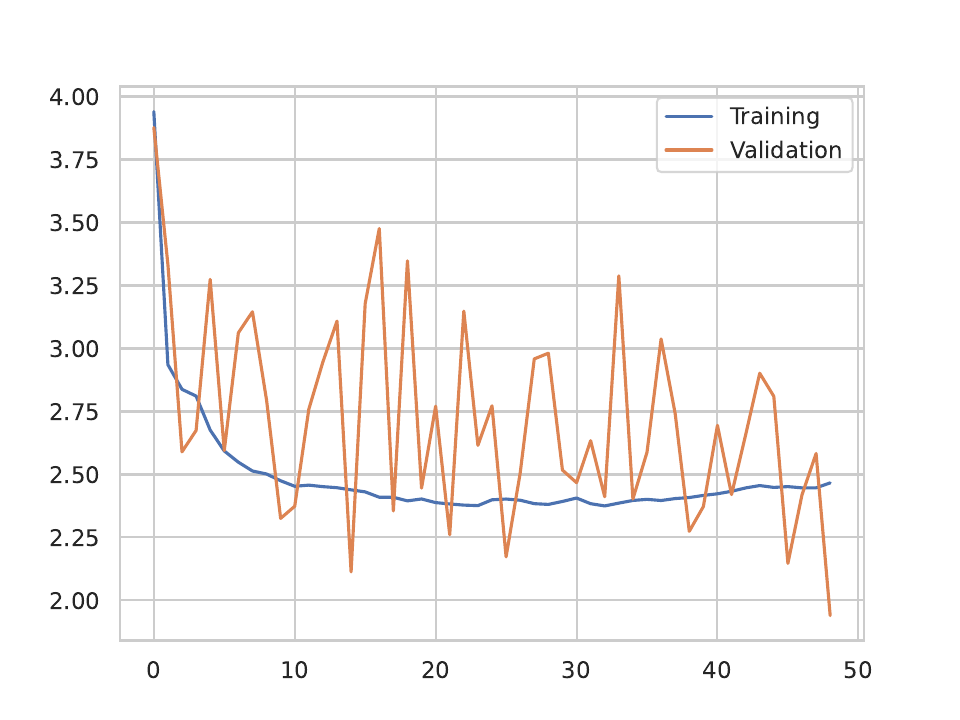} % Replace with your image path
        a. SGW Loss
        %\caption{Image 1}
    \end{minipage}
    \hspace{0.05\textwidth} % Space between the images
    \begin{minipage}[b]{0.4\textwidth}
        \centering
        \includegraphics[width=\textwidth]{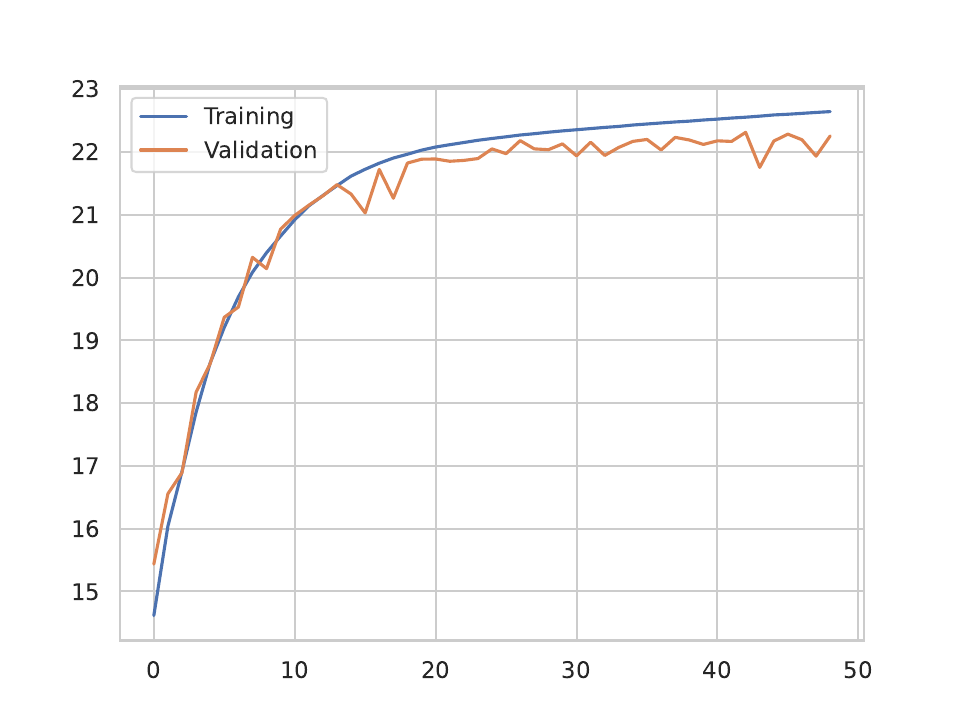} % Replace with your image path
        b. PSNR (in dB)
        %\caption{Image 2}
    \end{minipage}

    \vspace{0.1in} % Vertical space between rows

    % Second row: One image
    \begin{minipage}[b]{0.8\textwidth}
        \centering
        \includegraphics[width=\textwidth]{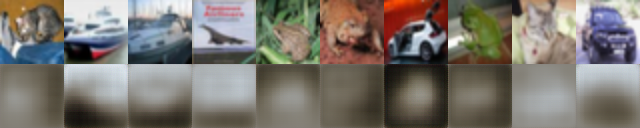} % Replace with your image path
        c. Epoch 1
        %\caption{Image 3}
    \end{minipage}

    \vspace{0.1in} % Vertical space between rows

    % Third row: One image
    \begin{minipage}[b]{0.8\textwidth}
        \centering
        \includegraphics[width=\textwidth]{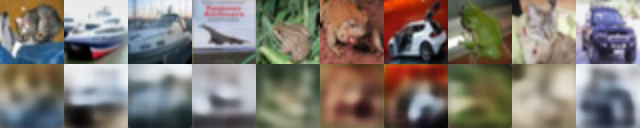} % Replace with your image path
        d. Epoch 10
        %\caption{Image 4}
    \end{minipage}

    \vspace{0.1in} % Vertical space between rows

    % Fourth row: One image
    \begin{minipage}[b]{0.8\textwidth}
        \centering
        \includegraphics[width=\textwidth]{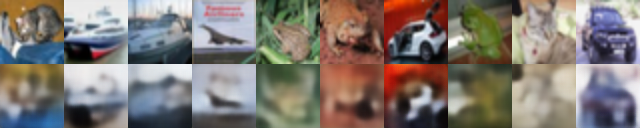} % Replace with your image path
        e. Last(48-th) iteration
        %\caption{Image 5}
    \end{minipage}
    
    \caption{a. Variation of SGW Loss with epochs, b. Variation of PSNR with epochs, c. Reconstruction in Epoch 1, d. Reconstruction in Epoch 10, e. Reconstruction in Last Epoch, \textbf{FID :}72.87, \textbf{PSNR :}22.65, \textbf{Total time (in sec):} 22459}
\end{figure}

\FloatBarrier
\clearpage
\subsubsection{RASGW}
\begin{figure}[htp!]
    \centering
    % First row: Two images
    \begin{minipage}[b]{0.4\textwidth}
        \centering
        \includegraphics[width=\textwidth]{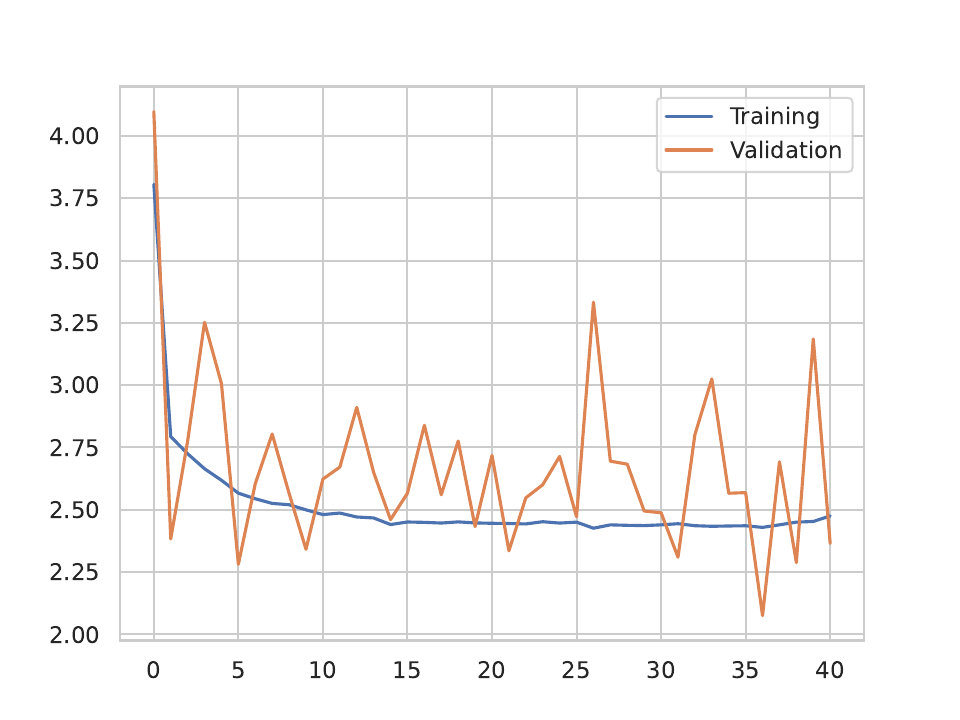} % Replace with your image path
        a. RASGW Loss
        %\caption{Image 1}
    \end{minipage}
    \hspace{0.05\textwidth} % Space between the images
    \begin{minipage}[b]{0.4\textwidth}
        \centering
        \includegraphics[width=\textwidth]{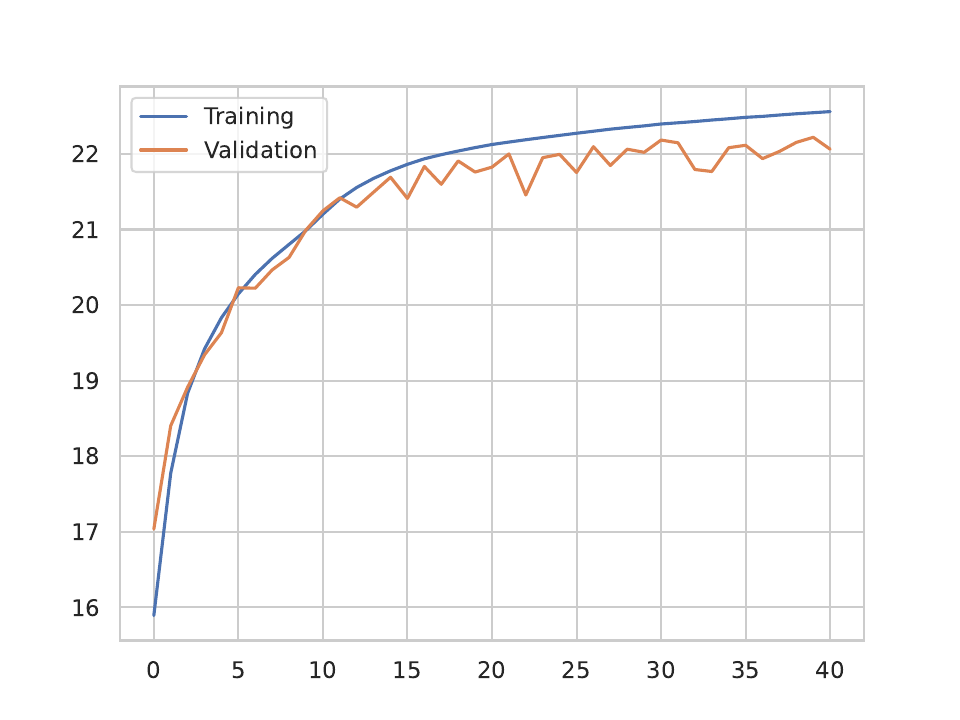} % Replace with your image path
        b. PSNR (in dB)
        %\caption{Image 2}
    \end{minipage}

    \vspace{0.1in} % Vertical space between rows

    % Second row: One image
    \begin{minipage}[b]{0.8\textwidth}
        \centering
        \includegraphics[width=\textwidth]{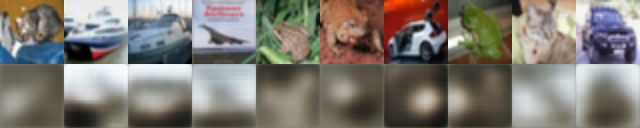} % Replace with your image path
        c. Epoch 1
        %\caption{Image 3}
    \end{minipage}

    \vspace{0.1in} % Vertical space between rows

    % Third row: One image
    \begin{minipage}[b]{0.8\textwidth}
        \centering
        \includegraphics[width=\textwidth]{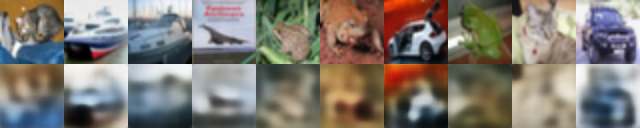} % Replace with your image path
        d. Epoch 10
        %\caption{Image 4}
    \end{minipage}

    \vspace{0.1in} % Vertical space between rows

    % Fourth row: One image
    \begin{minipage}[b]{0.8\textwidth}
        \centering
        \includegraphics[width=\textwidth]{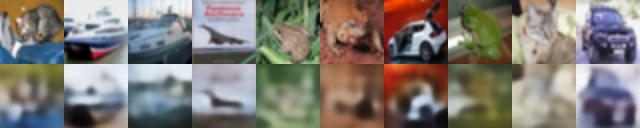} % Replace with your image path
        e. Last(40-th) iteration
        %\caption{Image 5}
    \end{minipage}
    
    \caption{a. Variation of RASGW Loss with epochs, b. Variation of PSNR with epochs, c. Reconstruction in Epoch 1, d. Reconstruction in Epoch 10, e. Reconstruction in Last Epoch, \textbf{FID :}70.46, \textbf{PSNR :}22.70, \textbf{Total time (in sec):} 23613}
\end{figure}
\FloatBarrier
\clearpage
\subsubsection{IWRASGW}
\begin{figure}[htp!]
    \centering
    % First row: Two images
    \begin{minipage}[b]{0.4\textwidth}
        \centering
        \includegraphics[width=\textwidth]{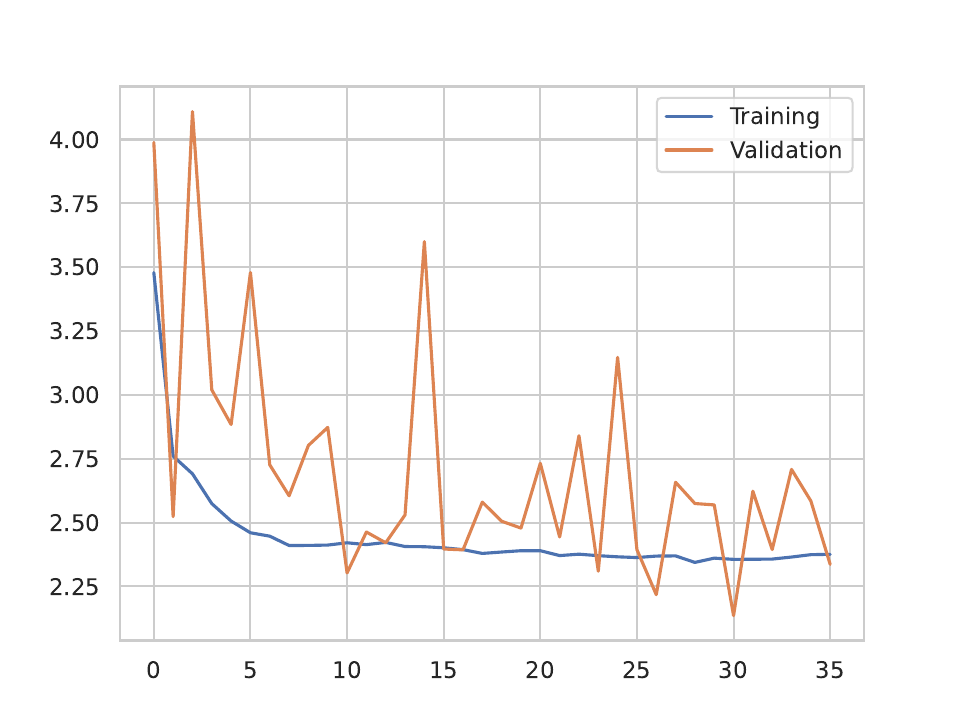} % Replace with your image path
        a. IWRASGW Loss
        %\caption{Image 1}
    \end{minipage}
    \hspace{0.05\textwidth} % Space between the images
    \begin{minipage}[b]{0.4\textwidth}
        \centering
        \includegraphics[width=\textwidth]{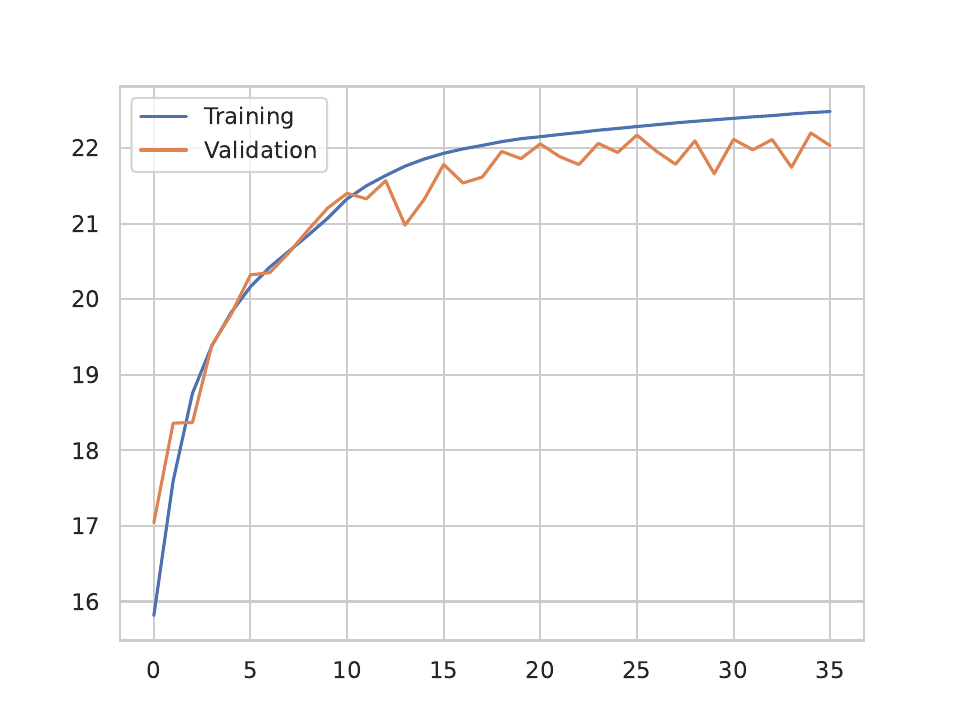} % Replace with your image path
        b. PSNR (in dB)
        %\caption{Image 2}
    \end{minipage}

    \vspace{0.1in} % Vertical space between rows

    % Second row: One image
    \begin{minipage}[b]{0.8\textwidth}
        \centering
        \includegraphics[width=\textwidth]{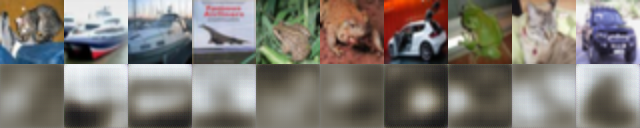} % Replace with your image path
        c. Epoch 1
        %\caption{Image 3}
    \end{minipage}

    \vspace{0.1in} % Vertical space between rows

    % Third row: One image
    \begin{minipage}[b]{0.8\textwidth}
        \centering
        \includegraphics[width=\textwidth]{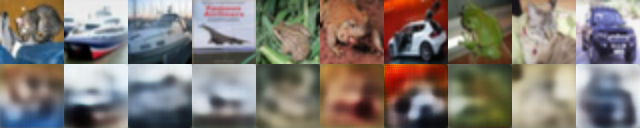} % Replace with your image path
        d. Epoch 10
        %\caption{Image 4}
    \end{minipage}

    \vspace{0.1in} % Vertical space between rows

    % Fourth row: One image
    \begin{minipage}[b]{0.8\textwidth}
        \centering
        \includegraphics[width=\textwidth]{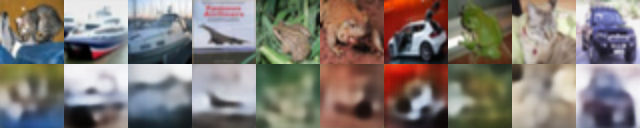} % Replace with your image path
        e. Last(35-th) iteration
        %\caption{Image 5}
    \end{minipage}
    
    \caption{a. Variation of RASGW Loss with epochs, b. Variation of PSNR with epochs, c. Reconstruction in Epoch 1, d. Reconstruction in Epoch 10, e. Reconstruction in Last Epoch, \textbf{FID :}71.99, \textbf{PSNR :}22.73, \textbf{Total time (in sec):} 19879}
\end{figure}
\FloatBarrier
\clearpage
\subsubsection{RPSGW}
\begin{figure}[htp!]
    \centering
    % First row: Two images
    \begin{minipage}[b]{0.4\textwidth}
        \centering
        \includegraphics[width=\textwidth]{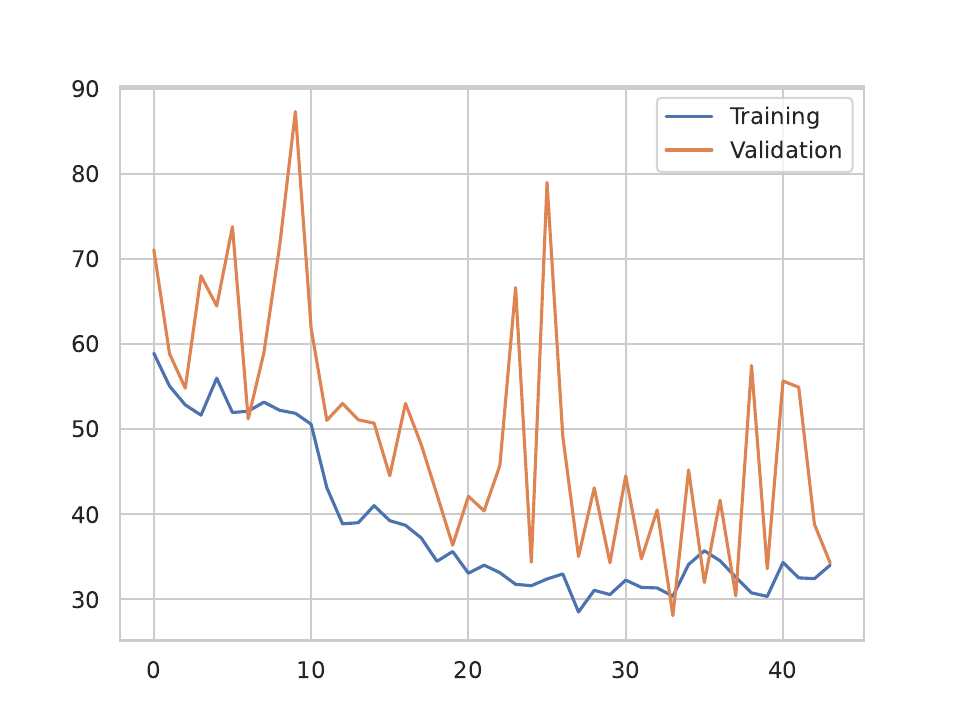} % Replace with your image path
        a. RPSGW Loss
        %\caption{Image 1}
    \end{minipage}
    \hspace{0.05\textwidth} % Space between the images
    \begin{minipage}[b]{0.4\textwidth}
        \centering
        \includegraphics[width=\textwidth]{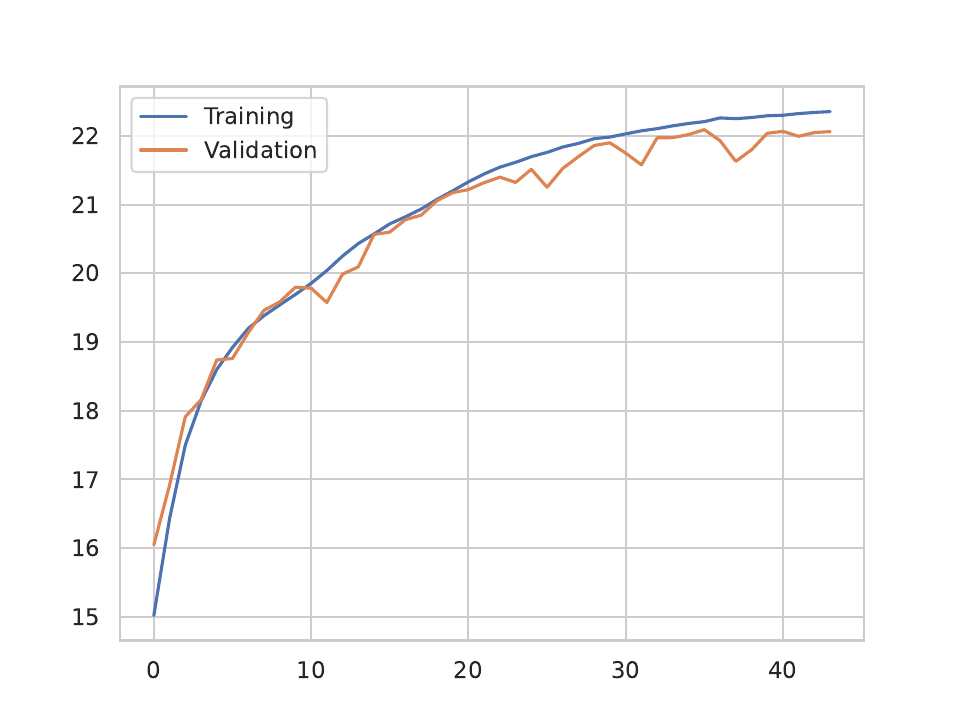} % Replace with your image path
        b. PSNR (in dB)
        %\caption{Image 2}
    \end{minipage}

    \vspace{0.1in} % Vertical space between rows

    % Second row: One image
    \begin{minipage}[b]{0.8\textwidth}
        \centering
        \includegraphics[width=\textwidth]{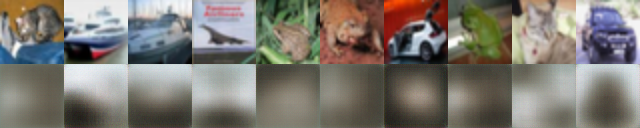} % Replace with your image path
        c. Epoch 1
        %\caption{Image 3}
    \end{minipage}

    \vspace{0.1in} % Vertical space between rows

    % Third row: One image
    \begin{minipage}[b]{0.8\textwidth}
        \centering
        \includegraphics[width=\textwidth]{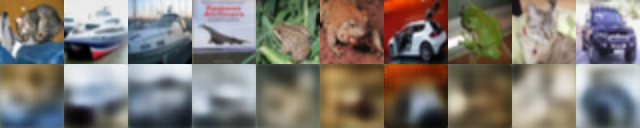} % Replace with your image path
        d. Epoch 10
        %\caption{Image 4}
    \end{minipage}

    \vspace{0.1in} % Vertical space between rows

    % Fourth row: One image
    \begin{minipage}[b]{0.8\textwidth}
        \centering
        \includegraphics[width=\textwidth]{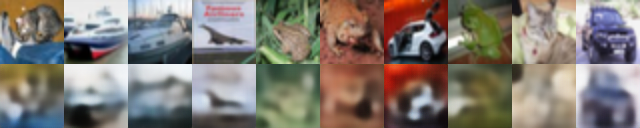} % Replace with your image path
        e. Last(43-th) iteration
        %\caption{Image 5}
    \end{minipage}
    
    \caption{a. Variation of RPSGW Loss with epochs, b. Variation of PSNR with epochs, c. Reconstruction in Epoch 1, d. Reconstruction in Epoch 10, e. Reconstruction in Last Epoch, \textbf{FID :}80.80, \textbf{PSNR :}22.55, \textbf{Total time (in sec):} 24356}
\end{figure}
\FloatBarrier
\clearpage
\subsubsection{EBSGW}
\begin{figure}[htp!]
    \centering
    % First row: Two images
    \begin{minipage}[b]{0.4\textwidth}
        \centering
        \includegraphics[width=\textwidth]{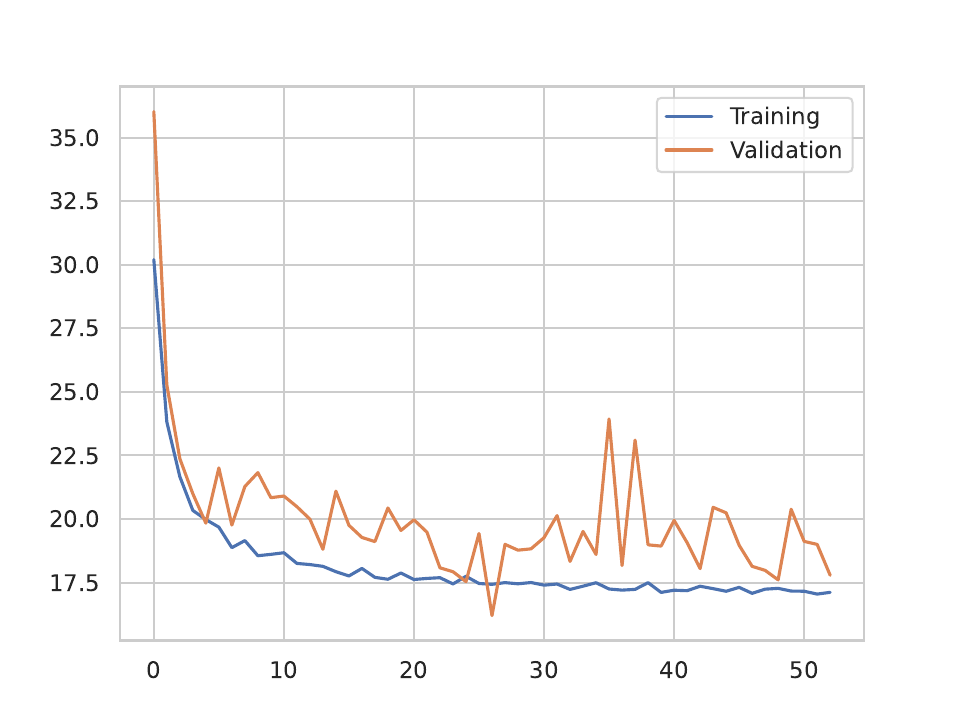} % Replace with your image path
        a. EBSGW Loss
        %\caption{Image 1}
    \end{minipage}
    \hspace{0.05\textwidth} % Space between the images
    \begin{minipage}[b]{0.4\textwidth}
        \centering
        \includegraphics[width=\textwidth]{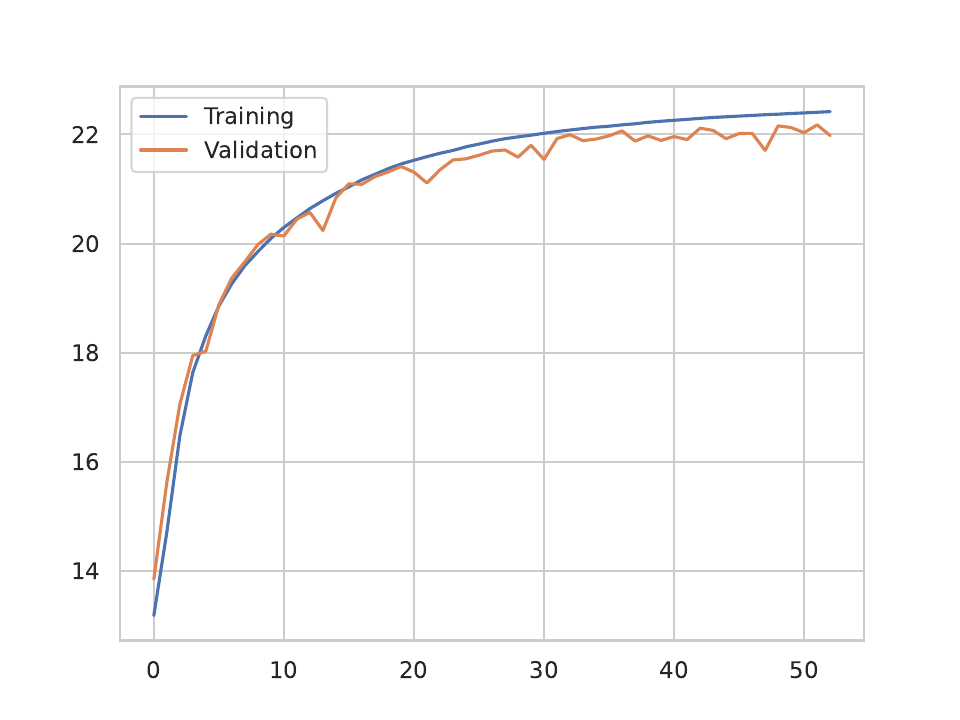} % Replace with your image path
        b. PSNR (in dB)
        %\caption{Image 2}
    \end{minipage}

    \vspace{0.1in} % Vertical space between rows

    % Second row: One image
    \begin{minipage}[b]{0.8\textwidth}
        \centering
        \includegraphics[width=\textwidth]{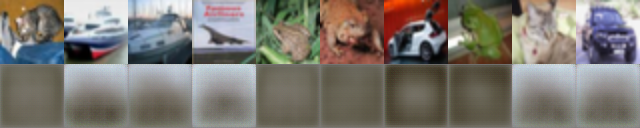} % Replace with your image path
        c. Epoch 1
        %\caption{Image 3}
    \end{minipage}

    \vspace{0.1in} % Vertical space between rows

    % Third row: One image
    \begin{minipage}[b]{0.8\textwidth}
        \centering
        \includegraphics[width=\textwidth]{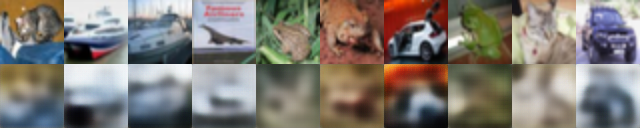} % Replace with your image path
        d. Epoch 10
        %\caption{Image 4}
    \end{minipage}

    \vspace{0.1in} % Vertical space between rows

    % Fourth row: One image
    \begin{minipage}[b]{0.8\textwidth}
        \centering
        \includegraphics[width=\textwidth]{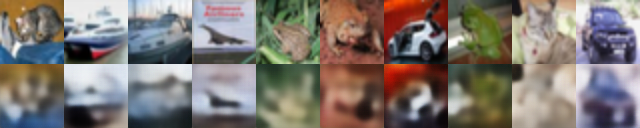} % Replace with your image path
        e. Last(52-th) iteration
        %\caption{Image 5}
    \end{minipage}
    
    \caption{a. Variation of EBSGW Loss with epochs, b. Variation of PSNR with epochs, c. Reconstruction in Epoch 1, d. Reconstruction in Epoch 10, e. Reconstruction in Last Epoch, \textbf{FID :}78.50, \textbf{PSNR :}22.48, \textbf{Total time (in sec):} 28765}
\end{figure}
\FloatBarrier
\clearpage
\subsection{Results on Omniglot}
\label{sec:results_omniglot}
\FloatBarrier
\subsubsection{SGW}
\begin{figure}[htp!]
    \centering
    % First row: Two images
    \begin{minipage}[b]{0.4\textwidth}
        \centering
        \includegraphics[width=\textwidth]{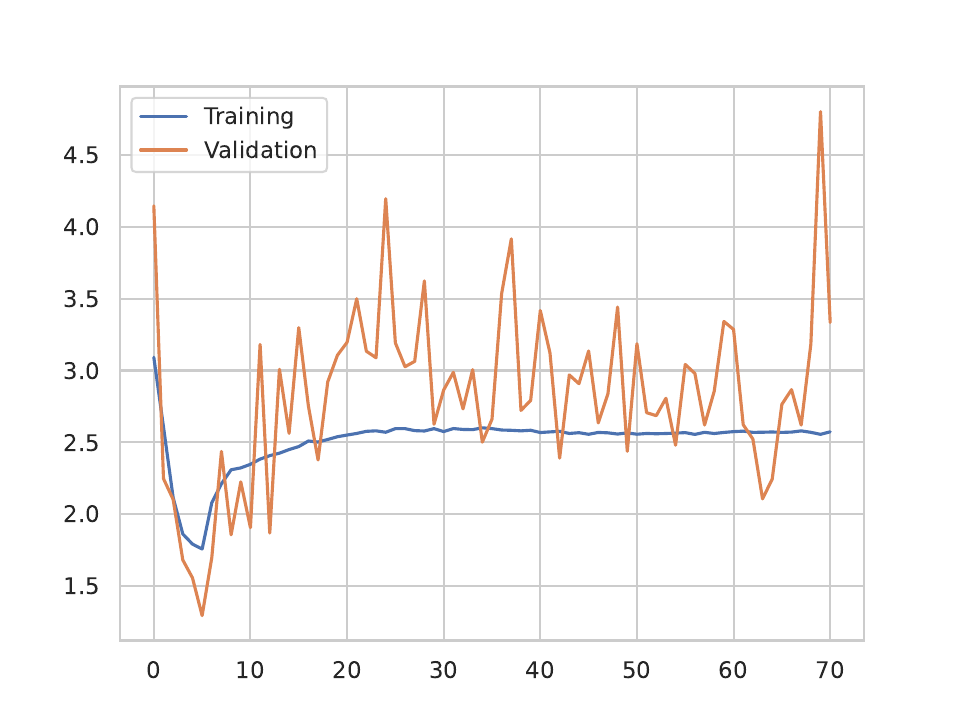} % Replace with your image path
        a. SGW Loss
        %\caption{Image 1}
    \end{minipage}
    \hspace{0.05\textwidth} % Space between the images
    \begin{minipage}[b]{0.4\textwidth}
        \centering
        \includegraphics[width=\textwidth]{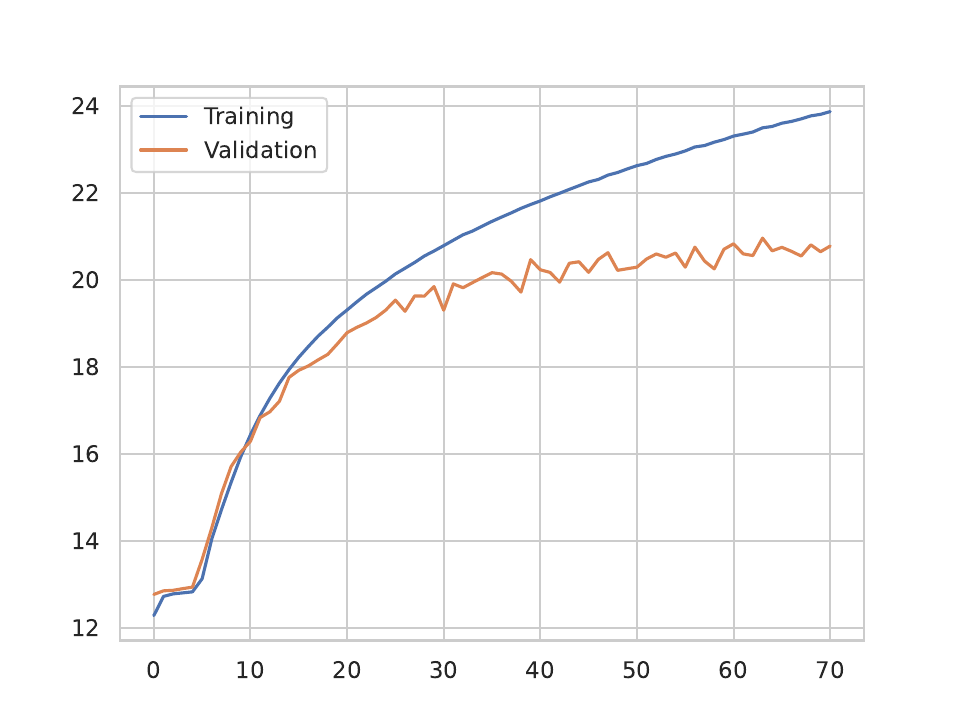} % Replace with your image path
        b. PSNR (in dB)
        %\caption{Image 2}
    \end{minipage}

    \vspace{0.1in} % Vertical space between rows

    % Second row: One image
    \begin{minipage}[b]{0.8\textwidth}
        \centering
        \includegraphics[width=\textwidth]{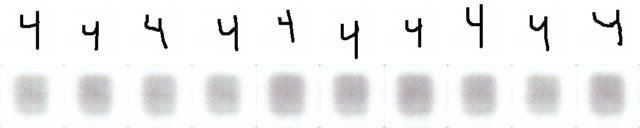} % Replace with your image path
        c. Epoch 1
        %\caption{Image 3}
    \end{minipage}

    \vspace{0.1in} % Vertical space between rows

    % Third row: One image
    \begin{minipage}[b]{0.8\textwidth}
        \centering
        \includegraphics[width=\textwidth]{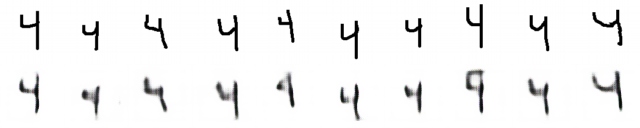} % Replace with your image path
        d. Epoch 10
        %\caption{Image 4}
    \end{minipage}

    \vspace{0.1in} % Vertical space between rows

    % Fourth row: One image
    \begin{minipage}[b]{0.8\textwidth}
        \centering
        \includegraphics[width=\textwidth]{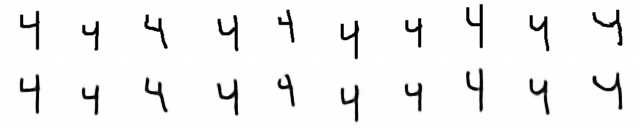} % Replace with your image path
        e. Last(70-th) iteration
        %\caption{Image 5}
    \end{minipage}
    
    \caption{a. Variation of SGW Loss with epochs, b. Variation of PSNR with epochs, c. Reconstruction in Epoch 1, d. Reconstruction in Epoch 10, e. Reconstruction in Last Epoch, \textbf{FID :}20.95, \textbf{PSNR :}22.13, \textbf{Total time (in sec):} 16932}
\end{figure}
\FloatBarrier
\clearpage
\subsubsection{RASGW}
\begin{figure}[htp!]
    \centering
    % First row: Two images
    \begin{minipage}[b]{0.4\textwidth}
        \centering
        \includegraphics[width=\textwidth]{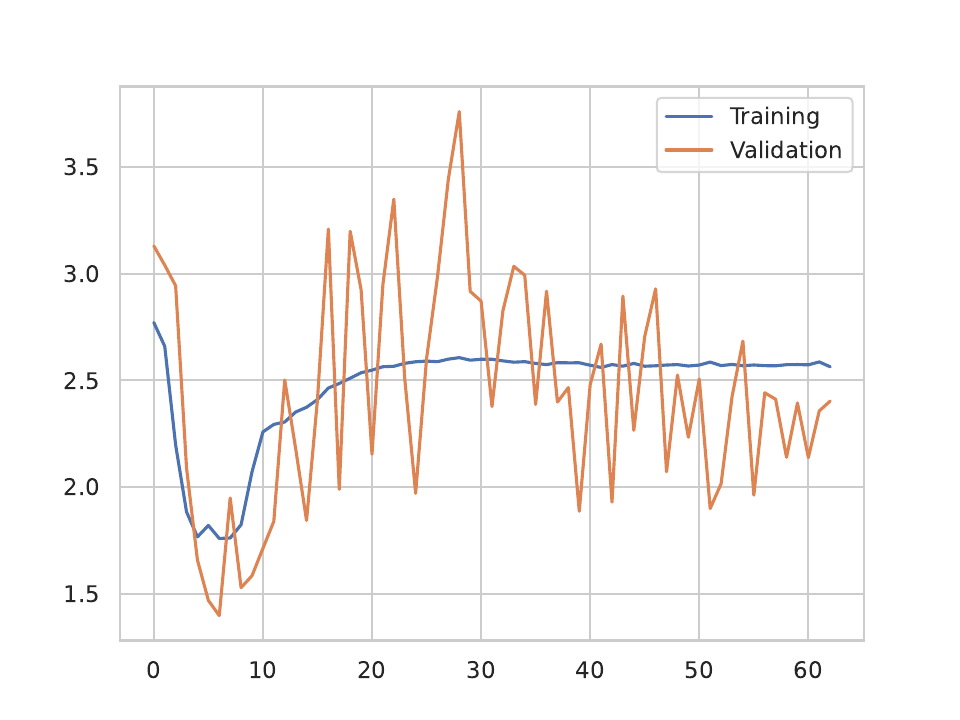} % Replace with your image path
        a. RASGW Loss
        %\caption{Image 1}
    \end{minipage}
    \hspace{0.05\textwidth} % Space between the images
    \begin{minipage}[b]{0.4\textwidth}
        \centering
        \includegraphics[width=\textwidth]{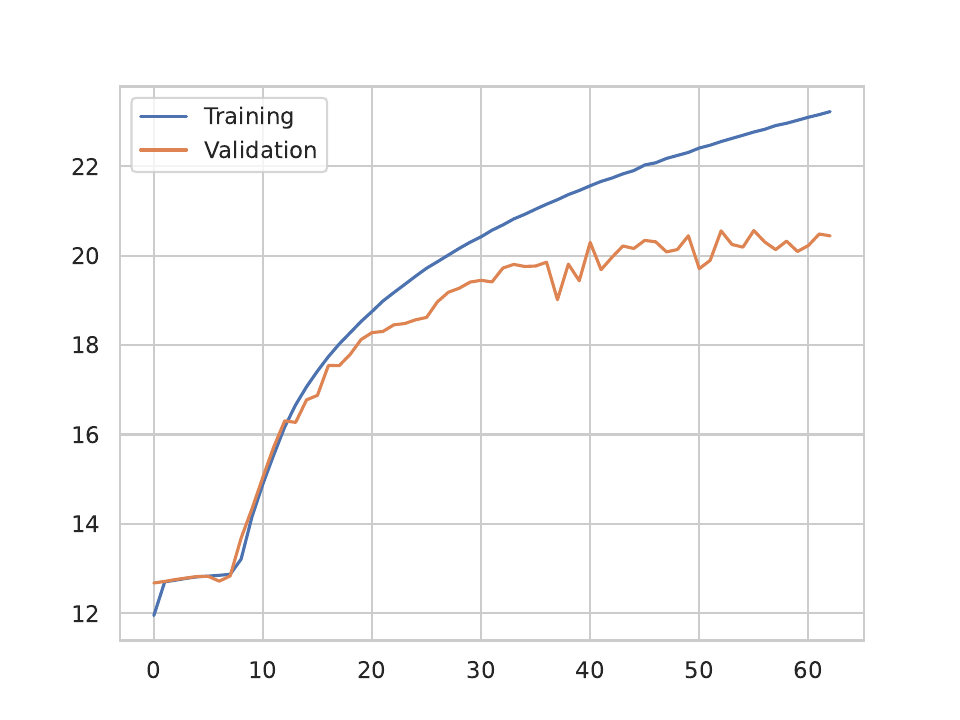} % Replace with your image path
        b. PSNR (in dB)
        %\caption{Image 2}
    \end{minipage}

    \vspace{0.1in} % Vertical space between rows

    % Second row: One image
    \begin{minipage}[b]{0.8\textwidth}
        \centering
        \includegraphics[width=\textwidth]{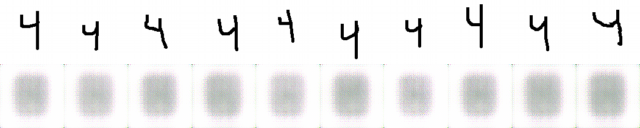} % Replace with your image path
        c. Epoch 1
        %\caption{Image 3}
    \end{minipage}

    \vspace{0.1in} % Vertical space between rows

    % Third row: One image
    \begin{minipage}[b]{0.8\textwidth}
        \centering
        \includegraphics[width=\textwidth]{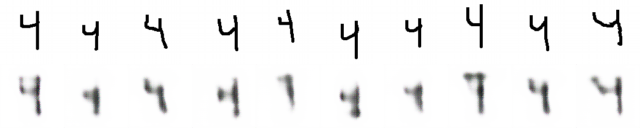} % Replace with your image path
        d. Epoch 10
        %\caption{Image 4}
    \end{minipage}

    \vspace{0.1in} % Vertical space between rows

    % Fourth row: One image
    \begin{minipage}[b]{0.8\textwidth}
        \centering
        \includegraphics[width=\textwidth]{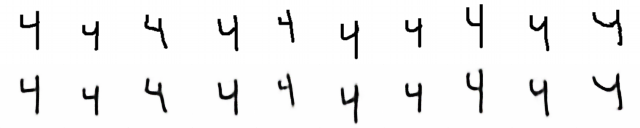} % Replace with your image path
        e. Last(62-nd) iteration
        %\caption{Image 5}
    \end{minipage}
    
    \caption{a. Variation of RASGW Loss with epochs, b. Variation of PSNR with epochs, c. Reconstruction in Epoch 1, d. Reconstruction in Epoch 10, e. Reconstruction in Last Epoch, \textbf{FID :}21.48, \textbf{PSNR :}20.56, \textbf{Total time (in sec):} 16731}
\end{figure}
\FloatBarrier
\clearpage
\subsubsection{IWRASGW}
\begin{figure}[htp!]
    \centering
    % First row: Two images
    \begin{minipage}[b]{0.4\textwidth}
        \centering
        \includegraphics[width=\textwidth]{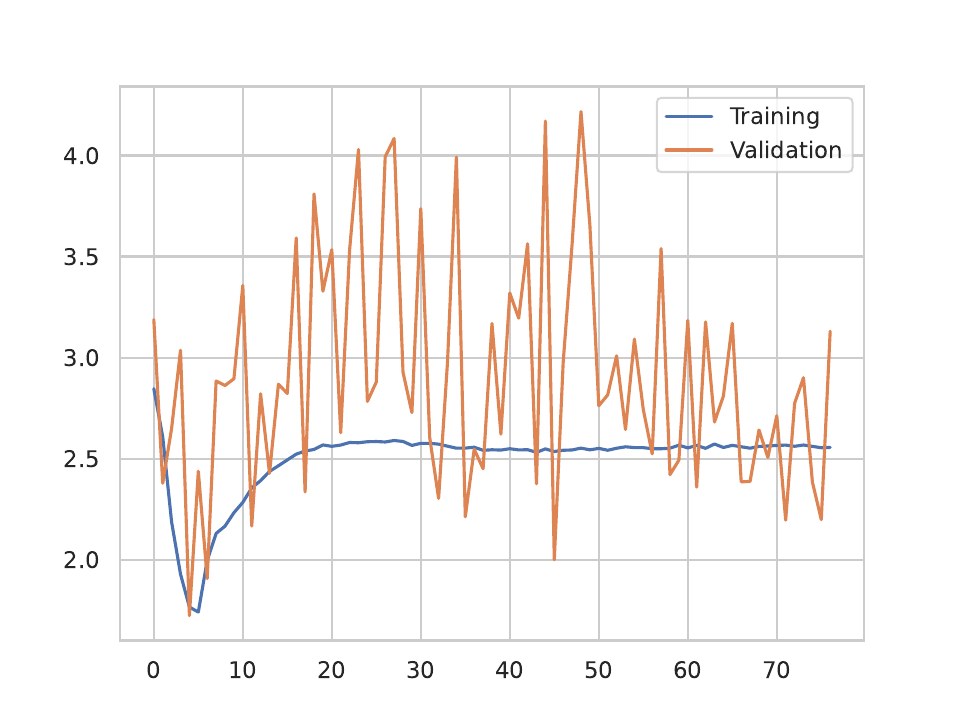} % Replace with your image path
        a. RASGW Loss
        %\caption{Image 1}
    \end{minipage}
    \hspace{0.05\textwidth} % Space between the images
    \begin{minipage}[b]{0.4\textwidth}
        \centering
        \includegraphics[width=\textwidth]{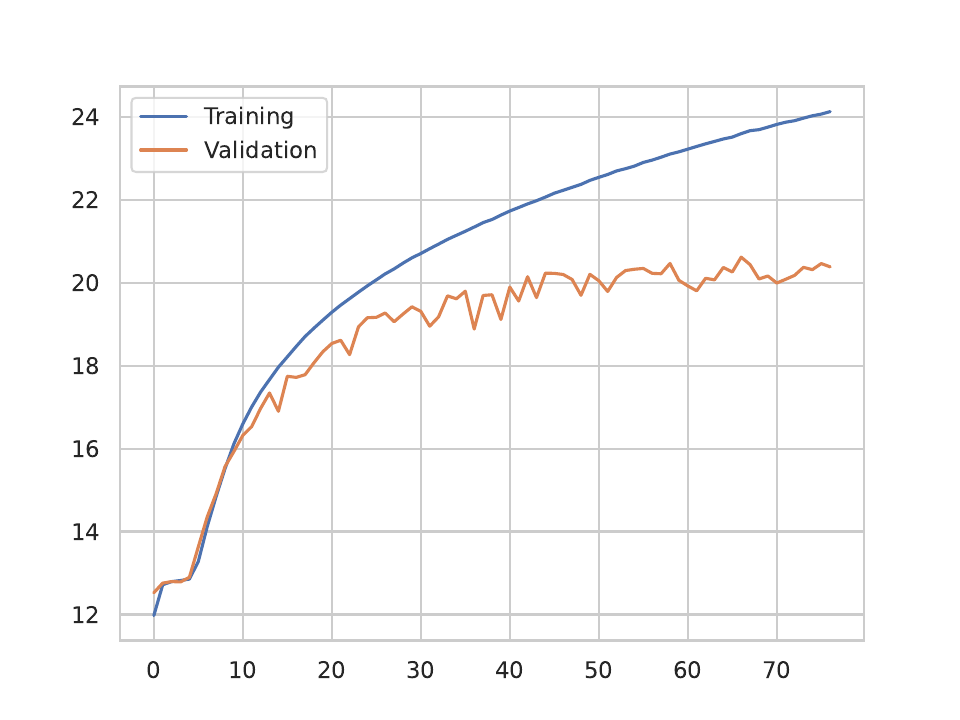} % Replace with your image path
        b. PSNR (in dB)
        %\caption{Image 2}
    \end{minipage}

    \vspace{0.1in} % Vertical space between rows

    % Second row: One image
    \begin{minipage}[b]{0.8\textwidth}
        \centering
        \includegraphics[width=\textwidth]{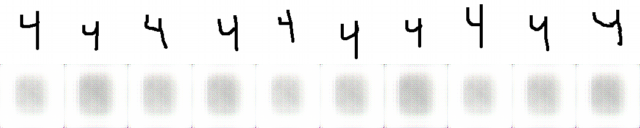} % Replace with your image path
        c. Epoch 1
        %\caption{Image 3}
    \end{minipage}

    \vspace{0.1in} % Vertical space between rows

    % Third row: One image
    \begin{minipage}[b]{0.8\textwidth}
        \centering
        \includegraphics[width=\textwidth]{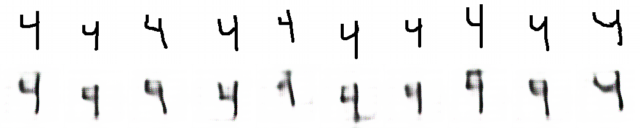} % Replace with your image path
        d. Epoch 10
        %\caption{Image 4}
    \end{minipage}

    \vspace{0.1in} % Vertical space between rows

    % Fourth row: One image
    \begin{minipage}[b]{0.8\textwidth}
        \centering
        \includegraphics[width=\textwidth]{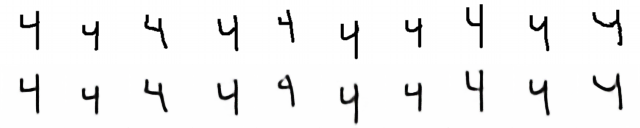} % Replace with your image path
        e. Last(76-th) iteration
        %\caption{Image 5}
    \end{minipage}
    
    \caption{a. Variation of IWRASGW Loss with epochs, b. Variation of PSNR with epochs, c. Reconstruction in Epoch 1, d. Reconstruction in Epoch 10, e. Reconstruction in Last Epoch, \textbf{FID :}20.25, \textbf{PSNR :}21.87, \textbf{Total time (in sec):} 18972}
\end{figure}

\FloatBarrier
\section{Computational Infrastructure}
All the experiments of GWGAN (Gromov-Wasserstein Generative Adversarial Network) were performed on a single NVIDIA RTX 3090 GPU (24 GB), and all the experiments of GWAE (Gromov-Wasserstein Auto-Encoder) were performed on a single NVIDIA GeForce GTX Titan Xp (12 GB).

\label{sec:comp_infra}

\end{document}